%% file: main.tex
\documentclass{article}
\usepackage[square,numbers,compress]{natbib}
\usepackage[final, main]{neurips_2025}

\usepackage[utf8]{inputenc} 
\usepackage[T1]{fontenc}    
\usepackage{url}            
\usepackage{booktabs}       
\usepackage{amsfonts}       
\usepackage{nicefrac}       
\usepackage{microtype}      
\usepackage{makecell}

\input{preamble/preamble}

\title{Fast Non-Log-Concave Sampling under Nonconvex Equality and Inequality Constraints with Landing}

\author{
    Kijung Jeon\\
    Georgia Institute of Technology\\
    \texttt{kjeon@gatech.edu}\\
    \And
    Michael Muehlebach\\
    MPI-IS\\
    \texttt{michaelm@tue.mpg.de}
    \And
    Molei Tao\\
    Georgia Institute of Technology\\
    \texttt{mtao@gatech.edu}
}

\begin{document}

\maketitle

\input{main/abstract}
\input{main/introduction}
\input{main/related_works}

\input{main/preliminaries}

\input{main/main_results}

\input{main/experiments}
\input{main/discussions}

\clearpage
\input{main_aux/ack_impact}
\bibliography{references}
\bibliographystyle{unsrtnat}
\input{main_aux/check_list}

\clearpage
\appendix
\clearpage
\tableofcontents
\clearpage

\input{appendix/scratch}
\clearpage
\input{appendix/Notations}

\clearpage
\input{appendix/Algorithms}
\clearpage

\input{appendix/Proof_of_differential_operator_properties_on_Sigma}
\input{appendix/Construction_of_SDE_decaying_constraints_exponentially_fast}
\input{appendix/Proof_of_theoretical_results_Equality_constraint_LOLD}

\input{appendix/Proof_of_theoretical_results_Inequality_constraint_LOLD}

\input{appendix/Proof_of_theoretical_results_Mixed_constraint_OLLA}
\clearpage
\input{appendix/Experiment_setting}

\end{document}

%% file: preamble/preamble.tex
\usepackage{mathtools}
\usepackage{mdframed}
\usepackage{amsthm, amsfonts, amssymb, mathrsfs}
\usepackage{thmtools}
\usepackage{algorithm}
\usepackage{algpseudocode}
\usepackage{enumitem}
\usepackage{aliascnt}
\usepackage{etoc}
\usepackage{wrapfig}
\usepackage{caption}

\usepackage[dvipsnames]{xcolor}

\usepackage[colorlinks,linktoc=all,pagebackref=true]{hyperref}
\usepackage[all]{hypcap}
\usepackage[nameinlink,capitalise, noabbrev]{cleveref}

\definecolor{darkred}{RGB}{100,0,0}
\definecolor{darkblue}{RGB}{0,0,100}
\definecolor{darkgreen}{RGB}{0,75,0}

\hypersetup{citecolor=darkgreen}
\hypersetup{linkcolor=darkgreen}
\hypersetup{urlcolor=darkgreen}

\declaretheoremstyle[
    spaceabove    = \parsep,
    spacebelow    = \parsep,
    bodyfont      = \normalfont\itshape,
]{theoremsty}

\declaretheoremstyle[
    spaceabove    = \parsep,
    spacebelow    = \parsep,
    bodyfont      = \normalfont,
]{normalsty}

\mdfdefinestyle{coloredstyle}{%
    leftline          = false,%
    rightline         = false,%
    topline           = false,%
    bottomline        = false,%
    backgroundcolor   = gray!10,%
    innertopmargin    = 1.5ex,%
    innerbottommargin = 0.7ex,%
    innerleftmargin   = 1.ex,%
    innerrightmargin  = 1.ex,%
    skipabove         = \topskip,%
    skipbelow         = \parskip%
}

\declaretheorem[name=Theorem, style=theoremsty, mdframed={style = coloredstyle}]{theorem}
\declaretheorem[name=Proposition, style=theoremsty, mdframed={style = coloredstyle}]{proposition}

\declaretheorem[name=Lemma, style=theoremsty, mdframed={style = coloredstyle}]{lemma}

\declaretheorem[name=Theorem, style=theoremsty, mdframed={style = coloredstyle}]{theorem*}
\declaretheorem[name=Proposition, style=theoremsty, mdframed={style = coloredstyle}]{proposition*}
\declaretheorem[name=Corollary, style=theoremsty, mdframed={style = coloredstyle}]{corollary*}
\declaretheorem[name=Lemma, style=theoremsty, mdframed={style = coloredstyle}]{lemma*}
\declaretheorem[name=Conjecture, style=theoremsty, mdframed={style = coloredstyle}]{conjecture*}

\declaretheorem[name=Theorem, style=theoremsty, mdframed={style = coloredstyle}, numberwithin=section]{theorem_ap}
\declaretheorem[name=Proposition, style=theoremsty, mdframed={style = coloredstyle}, numberwithin=section]{proposition_ap}
\declaretheorem[name=Corollary, style=theoremsty, mdframed={style = coloredstyle}, numberwithin=section]{corollary_ap}
\declaretheorem[name=Lemma, style=theoremsty, mdframed={style = coloredstyle}, numberwithin=section]{lemma_ap}

\declaretheorem[name=Remark, style=normalsty]{remark}
\declaretheorem[name=Definition, style=normalsty, mdframed={style = coloredstyle}]{definition}
\declaretheorem[name=Assumption, style=normalsty, mdframed={style = coloredstyle}]{assumption}

\declaretheorem[name=Remark, style=normalsty, numbered=no]{remark*}
\declaretheorem[name=Definition, style=normalsty, mdframed={style = coloredstyle}, numbered=no]{definition*}
\declaretheorem[name=Assumption, style=normalsty, mdframed={style = coloredstyle}, numbered=no]{assumption*}

\etocframedstyle[1]{\textbf{\textsc{Table of Contents}}}
\etocsettocdepth{3}

\input{preamble/math.tex}

\definecolor{blendedblue}{rgb}{0.2,0.2,0.7}

%% file: preamble/math.tex
\usepackage{bbm}

\newcommand{\calD}{{\mathcal{D}}}

\newcommand{\calF}{{\mathcal{F}}}

\newcommand{\calH}{{\mathcal{H}}}

\newcommand{\calL}{{\mathcal{L}}}

\newcommand{\calN}{{\mathcal{N}}}
\newcommand{\calO}{{\mathcal{O}}}

\newcommand{\calX}{{\mathcal{X}}}

\newcommand{\bbE}{\mathbb{E}}

\newcommand{\bbI}{\mathbb{I}}

\newcommand{\bbN}{\mathbb{N}}

\newcommand{\bbP}{\mathbb{P}}

\newcommand{\bbR}{\mathbb{R}}

\newcommand{\lbra}[1]{\left[#1\right]}
\newcommand{\mbra}[1]{\left\{#1\right\}}
\newcommand{\sbra}[1]{\left(#1\right)}
\newcommand{\norm}[1]{\lVert#1\rVert}
\newcommand{\abs}[1]{\left| #1\right|}
\newcommand{\inner}[1]{\langle#1\rangle}

\newcommand{\eye}{{\mathbbm{1}}}
\newcommand{\mean}[1]{\mathbb{E}\left[#1\right]}

\newcommand{\divergence}{\textsf{div}}

\newcommand{\KL}{\textsf{KL}}

\newcommand{\trace}[1]{\textsf{Tr}\sbra{#1}}

\newcommand{\dist}{\text{dist}}

\DeclareMathOperator*{\argmin}{argmin}

%% file: main/abstract.tex
\begin{abstract}

Sampling from constrained statistical distributions is a fundamental task in various fields including Bayesian statistics, computational chemistry, and statistical physics. This article considers sampling from a constrained distribution that is described by an unconstrained density, as well as additional equality and/or inequality constraints, which often make the constraint set nonconvex. Existing methods struggle in the presence of such nonconvex constraints, as they rely on projections, which are computationally expensive or intractable, are specialized to either inequality or equality constraints, and often lack rigorous quantitative convergence guarantees.
In this paper, we introduce \textit{Overdamped Langevin with LAnding} (OLLA), a new framework that can design overdamped Langevin dynamics accommodating both nonlinear equality and inequality constraints. The proposed dynamics also deterministically corrects trajectories along the normal direction of the constraint surface, thus obviating the need for explicit projections. We show that, under suitable regularity conditions on the target density 
and the feasible set $\Sigma\subset\bbR^d$, OLLA converges exponentially fast in 2-Wasserstein distance to the constrained target density $\rho_\Sigma(x) \propto \exp(-f(x))d\sigma_\Sigma$.
Lastly, through experiments, we demonstrate the efficiency of OLLA compared to known constrained Langevin algorithms and their slack variable variants, highlighting its favorable computational cost and fast empirical mixing.\footnote{All code is provided in the following repository: \url{https://github.com/KraitGit/OLLA}}
\end{abstract}

%% file: main/introduction.tex
\section{Introduction}
\label{main:sec:Introduction}

Sampling from complex, constrained statistical distributions is a fundamental problem in machine learning, with applications ranging from Bayesian inference under structured priors to training generative models with safety or fairness constraints. When there is no constraint, a prominent class of sampling techniques is centered around (overdamped) Langevin dynamics, where the drift is set to the gradient of the log-target density. These have gained significant traction due to their strong theoretical guarantees: for example under log-concave target densities \cite{dalalyan2017theoretical,durmus17,durmus19,erdogdu2022convergence,cheng2018convergence,dalalyan2017further,durmus2019analysis,li2022sqrt} or more general densities that satisfy relaxed conditions such as isoperimetric inequalities \citep{NEURIPS2019_65a99bb7,SinhoChewi} 
one can obtain fast, non-asymptotic convergence rates. Langevin dynamics can even be generalized, via a Riemannian Langevin approach, to sample under convex constraints \cite[e.g.,][]{zhang2020wasserstein, li2022mirror}. However, extending Langevin-based approaches to sample from distributions supported on nonconvex sets $\Sigma\subset \bbR^d$ remains a major challenge. Existing techniques typically rely on projection steps, which are computationally expensive and require convexity to ensure convergence. Moreover, most methods offer limited or no quantitative convergence guarantees in this case; in fact, even if the density is log-concave, a nonconvex constraint can easily make the target distribution much harder to sample from, also rendering the analysis more difficult. This is a critical bottleneck for many emerging machine learning applications -- such as imitation learning \citep{zare2024survey} or constrained generative modeling -- where the feasible set is implicitly defined by complex equality and inequality constraints. 

In this article, we introduce OLLA (Overdamped Langevin with LAnding), a suite of projection-free stochastic dynamics that could serve as the foundation for sampling from constrained distributions. OLLA avoids projections by combining two key ideas: (1) relying on local approximations of the feasible set to guide the sampling and (2) introducing a restitution mechanism called ``landing" that guarantees convergence to the feasible set. These techniques build on ideas previously developed in the context of nonconvex optimization, where they have been shown to provide scalable and effective algorithms \citep{JMLR:v23:21-0798,muehlebach2025accelerated,schechtman2023orthogonal}, and share close connection to several powerful constrained sampling approaches in the literature \citep{zhang2022sampling, zhang2024functional, rousset2010free, lelievre2012langevin, lelievre2019hybrid}. 
We adapt and extend the ideas to the sampling setting, which makes OLLA both computationally efficient and theoretically grounded. Our main contributions are summarized as follows:
\begin{itemize}[leftmargin=3mm]
    \item \textbf{(Unified treatment of constraints)} OLLA is described by a stochastic differential equation (SDE) that, in contrast to related prior works, enforces both equality and inequality constraints. This is achieved by constructing the tangent space to the constraint manifold and projecting the overdamped Langevin drift and diffusion terms onto the tangent space resulting in a simple least-squares problem. OLLA recovers the classical equality-only constrained Langevin dynamics as a special case, yet seamlessly accommodates arbitrary smooth inequality constraints without resorting to slack variables or projections.
    \item\textbf{(Exponential convergence)} We prove that the continuous version of OLLA converges to the constrained target distribution $\rho_\Sigma$ at an exponential rate under appropriate regularity assumptions. Our convergence results are non-asymptotic and characterized by the 2-Wasserstein distance in all scenarios (equality constraints only, inequality constraints only, and mixed).
    \item\textbf{(Efficient SDE discretization with trace-estimation)} We introduce OLLA-H, a computationally efficient Euler-Maruyama (EM) discretization of the aforementioned SDE that features a Hutchinson trace estimator \citep{hutchinson1989stochastic} for approximating the Itô-Stratonovich correction term arising from the diffusion. As a result, OLLA-H has low computational cost per iteration, even in high dimensions, and it achieves relatively accurate sampling and empirical mixing, an aspect that we demonstrate in various numerical experiments.
\end{itemize}

%% file: main/related_works.tex
\vspace{-10pt}
\section{Related Works}
\vspace{-5pt}
\label{main:sec:Related Works}
We first focus our literature review on recent works that consider Langevin sampling under nonlinear constraints. One of the closest touching points is \cite{zhang2022sampling}, which describes a gradient descent approach on the KL divergence. The algorithm shares some similarities to ours in that projections are avoided, but the work focused on equality constraints only, whereas OLLA covers both equality and inequality constraints. The work \cite{power2024constrained} proposes the use of slack variables to incorporate inequality constraints to change a mixed problem into an equality-only problem, which comes at the cost of additional spurious dimensions. In a similar vein, \cite{zhang2024functional} designs a particle-based variational inference method to incorporate inequality constraints. The method is effective at sampling under inequality constraints only, suffers, however, from high computational cost in high dimensions due to the estimation of associated boundary integrals.

OLLA is inspired by recent methods in nonlinear optimization 
\cite{muehlebach2025accelerated, JMLR:v23:21-0798, schechtman2023orthogonal} that use a similar landing mechanism and avoid projections onto the feasible set. There has also been important work by  \cite{rousset2010free, lelievre2012langevin, lelievre2019hybrid} who introduced a constrained Langevin dynamics based on numerical schemes such as SHAKE, RATTLE \citep{ciccotti1986molecular, ryckaert1977numerical, andersen1983rattle}, and including Metropolis-Hastings corrections \citep{lelievre2019hybrid}. These works mainly focus on equality-only constraints, although inequality constraints can be incorporated via including slack variables or applying reflection at the boundary. We use these algorithms as baselines and refer them to Constrained Langevin (CLangevin) \citep{lelievre2012langevin}, Constrained Hamiltonian Monte Carlo (CHMC) \citep{lelievre2012langevin}, and Constrained generalized Hybrid Monte Carlo (CGHMC) \citep{lelievre2019hybrid}.

In the present work, constraints are handled through a careful decomposition of the stochastic dynamics on the boundary into normal and tangential parts, thereby avoiding projections and even enabling infeasible initialization. There have also been alternative algorithm designs that, however, do not share these features, for example, \citep{bubeck2018sampling, bubeck2015finite} (based on projection), \citep{kook2022sampling, kook2024gaussian} (barrier functions),
\citep{sato2025convergence} (reflections), or  \citep{zhang2020wasserstein,ahn2021efficient, li2022mirror} (mirror maps). Other works by \cite{karagulyan2020penalized, gurbuzbalaban2024penalized} introduce penalties for constraint violations or relax the notion of constraint satisfaction \cite{chamon2024constrained}. 
In addition, we note that, although closely related, constrained sampling is not the same as sampling on manifolds \cite{girolami2011riemann,cheng2022efficient,gatmiry2022convergence,kong2024convergence}.

%% file: main/preliminaries.tex
\section{Preliminaries \& Notations}
\label{main:sec:preliminaries}

We consider sampling from a target density supported on the compact and connected Riemannian submanifold of $\bbR^d$ defined by $\Sigma:= \mbra{x \in \bbR^d \mid h(x) = 0 ,g(x) \leq 0}$ where $h: \bbR^d \rightarrow \bbR^m$ and $g: \bbR^d \rightarrow \bbR^l$ are smooth functions.  To guarantee that all constraint-related constructions are well-posed, we impose the Linear Independence Constraint Qualification (LICQ) condition \citep{rockafellar2009variational}.
\begin{definition} \ 
    The functions $h,g$ satisfy LICQ if 
        $\mbra{\nabla h_1(x), ... \nabla h_m(x)} \cup \mbra{\nabla g_i(x)}_{i \in I_x}$ 
    are linearly independent for every $x \in \Sigma$, where $I_x$ denotes the set of active inequality constraints, i.e., $I_x:=\mbra{i\in [l]~|~g_i(x)\geq 0}$.
\end{definition}

As a result of LICQ, the tangent space of $\Sigma$ at $x \in \Sigma$ can be defined as 
\begin{equation*}
    T_x \Sigma := \mbra{v \in \bbR^d \mid \nabla h(x) v = 0, \nabla g_{i}(x)v =  0, \ \forall i \in I_x}
\end{equation*} and its orthogonal projector onto $T_x \Sigma$ is $\Pi(x) = I - \nabla J(x)^T G(x)^{-1} \nabla J(x)$, where $J(x): = \lbra{h_1(x), ... , h_m(x), g_{i_1}(x)+\epsilon,..., g_{i_{\abs{I_x}}}(x) +\epsilon}^T, \mbra{i_1,...,i_{\abs{I_x}}} = I_x$ denotes the stacking of constraint functions for some $\epsilon > 0$ and $G(x):= \nabla J(x) \nabla J(x)^T$ is its associated Gram matrix.

For any smooth $f:\Sigma \rightarrow\bbR$ and a smooth vector field $X: \Sigma \rightarrow \bbR^d$, the intrinsic gradient and divergence on $\Sigma$ are given by $\nabla_\Sigma f(x) = \Pi(x) \nabla f(x)$, $\divergence_\Sigma X(x) = \trace{\Pi(x) \nabla X(x)}$, where $\nabla$ denotes the usual Euclidean gradient or Jacobian in $\bbR^d$. Our target density is 
set to be $\rho_\Sigma(x) \propto \exp(-f(x)) d\sigma_\Sigma$ on $\Sigma$ with $d \sigma_\Sigma$ being the induced Hausdorff measure on $\Sigma$. We write $\rho_t$, $\tilde{\rho}_t$ be the law of OLLA and the projected process of OLLA onto $\Sigma$ at time $t$. On $\Sigma$, the natural extension of KL divergence and Fisher information take the form:
\begin{equation*}
    \KL^\Sigma (\rho ||\rho_\Sigma) := \int_\Sigma \rho \ln \sbra{\frac{\rho}{\rho_\Sigma}} d\sigma_\Sigma\quad \text{and} \quad I^\Sigma (\rho ||\rho_\Sigma) := \int_\Sigma \rho \norm{\nabla_\Sigma \ln \sbra{\frac{\rho}{\rho_\Sigma}}}_2^2 d\sigma_\Sigma
\end{equation*}

We now streamline notations appearing in \cref{main:sec:main_results}. A complete list of symbols and precise technical definitions are included in \cref{appendix:sec:Table of Key Notations and assumptions}. In particular, let $\pi(x)$ denote the nearest-point (Euclidean) projection onto $\Sigma$, and let $\lambda_{\text{LSI}}$ be the log-Sobolev constant of $(\Sigma, \rho_\Sigma)$. We then assume the existence of the following constants.
\begin{equation*}
    M_h := \sup_{x_0 \in \text{supp}(\rho_0)} \norm{h(x_0)}_2, \quad M_g := \sup_{x_0 \in \text{supp}(\rho_0)} \norm{g(x_0)}_2,
\end{equation*}
over the support of the initial law $\rho_0$. The constant $\kappa$ (\cref{lem:regularity lemma}, \cref{lem:regularity lemma with boundary}) and $\delta$  captures the regularity of $\Sigma$ and $\hat{U}_\delta$ denotes the tubular neighborhood of width $\delta$ with a special ``recovery'' property (see \cref{thm:recoverable tubular neighborhood}, \cref{thm:recoverable tubular neighborhood with boundary} for the precise definitions of $\delta$ and $\hat{U}_\delta$).

%% file: main/main_results.tex
\vspace{-5pt}
\section{Main results}
\vspace{-5pt}
\label{main:sec:main_results}

\subsection{Construction of OLLA via Least Squares}
We now derive the continuous-time dynamics of OLLA by choosing the drift vector $q$ and the symmetric diffusion matrix $Q$ to be the closest—in a least-squares sense—to the unconstrained usual Langevin coefficients, subject to enforcing both the equality constraints $\mbra{h_i(x)}_{i=1}^m$ and the active inequality constraints $\mbra{g_j(x)}_{j \in I_x}$. This is achieved by applying It\^o's lemma to each constraint function $h_i$ and $g_j$ and splitting the change in, for example, $h_i$, into a martingale term $\nabla h_i(X_t)^T Q(X_t)dW_t$ and a drift term $\lbra{\nabla h_i(X_t)^T q(X_t) + \frac{1}{2} \trace{\nabla^2 h_i(X_t) Q(X_t) Q(X_t)^T}}dt$. By choosing $Q$ so that $Q(x) \nabla h_i(x) = Q(x) \nabla g_j(x) = 0$, the martingale piece vanishes exactly in the normal directions. Simultaneously, we pick the drift vector $q$ to satisfy the linear equation 
\begin{equation}
    \nabla h_i^T q + \frac{1}{2} \trace{\nabla^2 h_i QQ^T} + \alpha h_i = 0, \quad \nabla g_j^T q + \frac{1}{2} \trace{\nabla^2 g_j QQ^T} + \alpha (g_j +\epsilon) = 0
\end{equation}
so that $h(X_t) = h_i(X_0)e^{-\alpha t}$ and  $g_j(X_t)+\epsilon = (g _j(X_0)+\epsilon) e^{-\alpha t}$, where hyperparameters $\alpha, \epsilon>0$ denote the landing or boundary repulsion rate, respectively. This enforces $g_j(X_t)$ to hit $0$ in finite time $t = \frac{1}{\alpha} \ln ((g_j(X_0)+\epsilon)/\epsilon)$, after which $g_j(X_t) \leq 0$ remains forever.
As a result, this approach removes any noise and drift direction in the normal of constraints and implants a pure drift normal to constraints, guaranteeing exponential decay of both equality and active-inequality constraints at a rate $\alpha$. This yields the closed-form SDE in \cref{prop: construction of OLLA} and \cref{lem:Exponential decay of constraint functions}.

\begin{proposition}[Construction of OLLA and its closed form SDE] \ 
    \label{prop: construction of OLLA}
    Consider the following SDE:
    \begin{equation}
        \label{eqn:mixed_OLLA_premitive_form}
        dX_t = q(X_t)dt + Q(X_t) dW_t
    \end{equation}
    where
    \begin{align}
    Q&:= \argmin\limits_{\bar{Q}\in\bbR^{d\times d}} \norm{\sqrt{2}I-\bar{Q}}_F^2 \quad \text{s.t} \quad     \begin{cases}
        \bar{Q}  \nabla h_i = 0, \ \forall i \in [m],\\
          \bar{Q} \nabla g_j = 0, \ \forall j \in I_x,
    \end{cases} \notag \\
    q&:= \argmin\limits_{\bar{q}\in\bbR^{d}} \norm{\bar{q} + \nabla f}_2^2 \quad \text{s.t} \quad 
    \begin{cases}
        \nabla h_i^T \bar{q} + \frac{1}{2} \trace{\nabla^2h_i QQ^T} + \alpha h_i = 0, \qquad \ \ \  \forall i \in [m], \notag \\
        \nabla g_j^T \bar{q} + \frac{1}{2} \trace{\nabla^2g_j QQ^T} + \alpha (g_j + \epsilon) = 0, \ \ \ \forall j \in I_x.
    \end{cases}
    \end{align}
    Then, there exists a closed form SDE (OLLA) of $\eqref{eqn:mixed_OLLA_premitive_form}$ given by:
    \begin{equation}
        \label{eqn:mixed_OLLA_closed_form}
        dX_t = -[\Pi(X_t) \nabla f(X_t) +\alpha \nabla J(X_t)^T G^{-1}(X_t) J(X_t)]dt + \calH(X_t)dt + \sqrt{2}\Pi(X_t) dW_t
    \end{equation}
    where
    \begin{equation}
        \label{eqn:mean_curvature_term}
        \calH := -\nabla J^T G^{-1} \lbra{\trace{\nabla^2 h_1 \Pi}, ..., \trace{\nabla^2 h_{m} \Pi}, \trace{\nabla^2 g_{i_1} \Pi} ,..., \trace{\nabla^2 g_{i_{\abs{I_x}}} \Pi}}^T
\end{equation} is the associated mean curvature correction term of $\Sigma_{I_x} := \mbra{ x \in \bbR^d \mid h(x) = 0, g_{I_x} (x) = 0}$.
\end{proposition}

\begin{remark}[Mean curvature = It\^o-Stratonovich correction] \quad By technical stochastic-calculus identities on manifolds (see, e.g., \citet{rousset2010free}, Lemma 3.19), the It\^o-Stratonovich correction arising from the Stratonovich SDE
    \begin{equation}
        \label{eqn:Stratonovich_form_OLLA}
        dX_t = -[\Pi(X_t) \nabla f(X_t) +\alpha \nabla J(X_t)^T G^{-1}(X_t) J(X_t)]dt + \sqrt{2}\Pi(X_t) \circ dW_t 
    \end{equation}
coincides exactly with the mean-curvature term $\calH(x)$ of $\Sigma_{I_x} := \mbra{x \in \bbR^d \mid h(x)=0, g_{I_x} (x) = 0}$.
\end{remark}

\begin{remark}[Relation to orthogonal direction samplers from variational $\KL$] \quad 
    \citet{zhang2022sampling} introduced an overdamped-Langevin sampler for equality constraints by minimizing the constrained KL divergence via an orthogonal-space variational formulation, and \citet{zhang2024functional} also used a similar approach to handle single inequality constraint. In the absence of inequality constraints, our OLLA dynamics coincide with the equality constrained sampler of \cite{zhang2022sampling} up to a modified potential $\hat{f}(x) = f(x) + \frac{1}{2} \ln \det(G)$, since the mean curvature correction satisfies 
    \begin{equation}
        \calH(x) = \divergence \Pi(x) + \Pi(x) \nabla \ln \sbra{ \sbra{\det G(x)}^{\frac{1}{2}}}
    \end{equation} (see \citet{rousset2010free}, Lemma 3.21). Correspondingly, whereas their framework yields a stationary measure proportional to $e^{-f(x)} \delta_\Sigma (dx)$ under the coarea formula, OLLA converges to the Riemannian volume-weighted density $e^{-f(x)} d\sigma_\Sigma(x)$ on $\Sigma$ (see \citet{rousset2010free}, Lemma 3.2). In addition to this, the work \cite{zhang2024functional} enforces a single inequality constraint via a purely deterministic normal drift $-\alpha \nabla g/\norm{\nabla g}_2$, without any stochastic component in the tangential direction. OLLA instead projects noise and drift vectors tangentially even during the landing phase, thereby preserving exploration on the evolving manifold $\Sigma^t := \mbra{x \in \bbR^d \mid h(x) = h(X_0)e^{-\alpha t}}$ and improving mixing. 
\end{remark}

\begin{lemma}[Exponential decay of constraint functions] \ 
    \label{lem:Exponential decay of constraint functions}  
    The dynamics induced by \eqref{eqn:mixed_OLLA_closed_form} satisfy the following properties almost surely for $\forall i \in [m], \forall j \in I_{X_0}$:
    \begin{equation}
        h_i(X_t) = h_i(X_0)e^{-\alpha t}, \quad t\geq0
    \end{equation}
    and
    \begin{equation*}
        \begin{cases}
            \begin{aligned}[t]
            g_j(X_t) &= -\epsilon + (g_j(X_0) + \epsilon) e^{-\alpha t}, \quad && t\leq \frac{1}{\alpha} \ln \sbra{ \frac{g_j(X_0)+\epsilon}{\epsilon}}\\
            g_j(X_t) &\leq 0,\quad && t\geq \frac{1}{\alpha} \ln \sbra{ \frac{g_j(X_0)+\epsilon}{\epsilon}}
            \end{aligned}
        \end{cases}
    \end{equation*}
    with $g(X_t) \leq 0, \forall t \geq 0$ for $j \notin I_{X_0}$, where $I_x := \mbra{k \in [l] \mid g_k(x) \geq 0}$ is the index set of active inequality constraints.
\end{lemma}

\subsection{Non-asymptotic Convergence Analysis of OLLA}
\label{main:sec:Non-asymptotic Convergence Analysis of OLLA}
In this subsection, we establish non-asymptotic convergence guarantees for OLLA in three scenarios -- equality only case, inequality only case, mixed case -- by recognizing that OLLA has a rapid landing property driven by landing rate $\alpha$ and natural mixing on the landed manifold induced by the LSI constant. For the detailed notations and proofs, we refer to the \cref{appendix:sec:Table of Key Notations and assumptions} and related Appendices.

\textbf{Equality-Only Scenario.} \quad When the constraints consist of only smooth equalities $h(x) =0$, the continuous-time OLLA dynamics \eqref{eqn:mixed_OLLA_closed_form} can be written as
\begin{equation}
    \label{eqn:equality_constrained_OLLA}
    dX_t = -\lbra{\Pi(X_t) \nabla f(X_t) +\alpha \nabla h(X_t)^T G^{-1}(X_t) h(X_t)}dt + \sqrt{2}\Pi(X_t) \circ dW_t
\end{equation}
The drift term naturally decomposes into a tangential term, which moves along the constraint surface, and a normal landing term, which forces each coordinate $h_i(X_t)$ to decay exponentially fast, as summarized in \cref{lem:Exponential decay of constraint functions}. 

Once $X_t \sim \rho_t$ lies within a tubular neighborhood of $\Sigma$, the nearest projection map $\pi$ onto $\Sigma$ becomes available, and the projected process $Y_t= \pi(X_t) \sim \tilde{\rho}_t$ can be defined (\cref{thm:recoverable tubular neighborhood}). Then, the regularity of $\Sigma$ naturally implies that $\norm{X_t - Y_t}_2 \lesssim \norm{h(X_t)}_2   = \calO(e^{-\alpha t})$ holds (\cref{lem:regularity lemma}), leading to the contraction of $W_2(\rho_t, \tilde{\rho}_t) = \calO (e^{-\alpha t})$ (\cref{lem:bound of W_2_eq_only}), where $\rho_t, \tilde{\rho}_t$ are the laws of $X_t, Y_t$, respectively. Furthermore, the combination of Lipschitzness of $\pi$ and $\norm{X_t-Y_t}_2 = \calO(e^{-\alpha t})$ enable us to write the projected process $Y_t$ as overdamped Langevin dynamics on $\Sigma$ with noisy drift vector and diffusion matrix whose norm is bounded by $\calO(e^{-\alpha t})$ (\cref{cor:SDE representation of projected process from equality-constrained OLLA}). Therefore, the effect of noisy terms can be dominated by the effect of the LSI constant, which leads to the exponentially fast convergence of $\KL^\Sigma(\tilde{\rho}_t || \rho_\Sigma)$ (\cref{lem:Upper bound of KL_Sigma}). A rigorous combination of these insights gives the following theorem.

\begin{theorem}[Convergence result for equality-constrained OLLA] \ 
    \label{thm:Convergence result for equality-constrained OLLA}
    Suppose assumptions \ref{asm:C1} to \ref{asm:C4} hold. Let $X_t$ be the stochastic process following the equality-constrained OLLA and let $\rho_t, \tilde{\rho}_t$ be the law of $X_t$ and its projection $Y_t = \pi(X_t)$ on $\Sigma$ for $t \geq t_{\text{cut}}$, $t_{\text{cut}} := \max \mbra{\frac{1}{\alpha} \ln \delta, \frac{1}{\alpha} \ln C_5}$, respectively.
    Then, for $\alpha > 2 \lambda_{\text{LSI}}$ for all $t \geq t_{\text{cut}}$, it holds that
    \begin{align*}
        W_2(\rho_t, \rho_\Sigma) \leq \frac{M_h}{\kappa}e^{-\alpha t} + \sqrt{\frac{2}{\lambda_{\text{LSI}}} \KL^\Sigma (\tilde{\rho}_t || \rho_\Sigma)} 
    \end{align*}
    where 
    \begin{equation*}
        \KL^\Sigma (\tilde{\rho}_t || \rho_\Sigma) \leq \exp \sbra{-2 \lambda_{\text{LSI}} (t-t_{\text{cut}}) - \frac{2\lambda_{\text{LSI}}C_5}{\alpha}(e^{-\alpha t} -e^{-\alpha t_{\text{cut}}}) }[\KL^\Sigma (\tilde{\rho}_{t_{\text{cut}}} || \rho_\Sigma) + C_7]
    \end{equation*}
    for some constants $C_5=\calO \sbra{1+ \frac{C_{L_A}M_h}{\kappa} +\sbra{\frac{C_{L_A}M_h}{\kappa}}^2},  C_7 := \frac{C_6 e^{-\alpha t_{\text{cut}}}}{\alpha - 2\lambda_{\text{LSI}}} > 0$ with $C_{L_A}$ being the Lipschitz constant of $\nabla \pi(x) \Pi(x) $ on $\hat{U}_\delta(\Sigma)$.
\end{theorem}

\textbf{Inequality-Only Scenario.} \quad With smooth inequalities $g(x) \leq 0$, we introduced a small boundary repulsion parameter $\epsilon >0$ so that each initially violated constraint $g_j >0$ is driven to the boundary within a finite landing time $t_{\text{cut}} = \frac{1}{\alpha} \ln((g_j(X_0)+\epsilon)/\epsilon)$. 

From the Fokker-Planck perspective, the normal probability flux at the boundary satisfies $\inner{n, J_t} = -\alpha \epsilon \rho_t$ where $J_t$ is the probability current density of $\rho_t$ and $n$ is the outward unit normal vector of $\Sigma$. The boundary repulsion enforces the outward probability flow to become zero after the landing time $t_{\text{cut}}$, and the normal probability flux vanishes to zero after time $t_{\text{cut}}$ (\cref{lem:Boundary behaviors of rho_t of inequality-constrained OLLA}). This enables us to ignore the boundary integral appearing on the time derivative of $\KL(\rho_t ||\rho_\Sigma)$ guaranteeing exponential decay driven by the effect of the LSI constant. Therefore, the following theorem comes by directly analyzing the time derivative of $\KL(\rho_t ||\rho_\Sigma)$ after time $t_{\text{cut}}$.

\begin{theorem}[Convergence result for inequality-constrained OLLA] \ 
    \label{thm:Convergence result for inequality-constrained OLLA}
    Suppose assumptions \ref{asm:C1} to \ref{asm:C4} hold. Let $X_t$ be the stochastic process following the inequality-constrained OLLA 
    and let $\rho_t$ be the law of $X_t$. Then, for $t\geq t_{\text{cut}}, t_{\text{cut}}: = \frac{1}{\alpha} \ln \sbra{\frac{M_g+\epsilon}{\epsilon}}$,
    \begin{align*}
        W_2(\rho_t, \rho_\Sigma) \leq  \sqrt{\frac{2}{\lambda_{\text{LSI}}} \KL^\Sigma (\rho_t || \rho_\Sigma)} 
    \end{align*}
    where 
    \begin{equation*}
        \KL^\Sigma(\rho_t ||\rho_\Sigma) \leq e^{-2 \lambda_{\text{LSI}} (t-t_{\text{cut}})}
        \KL^\Sigma(\rho_{t_{\text{cut}}} || \rho_\Sigma).
    \end{equation*}
\end{theorem}

\begin{wrapfigure}{r}{0.43\textwidth}
    \centering
    \includegraphics[width=1.0\linewidth, trim= 0 50 0 50, clip]{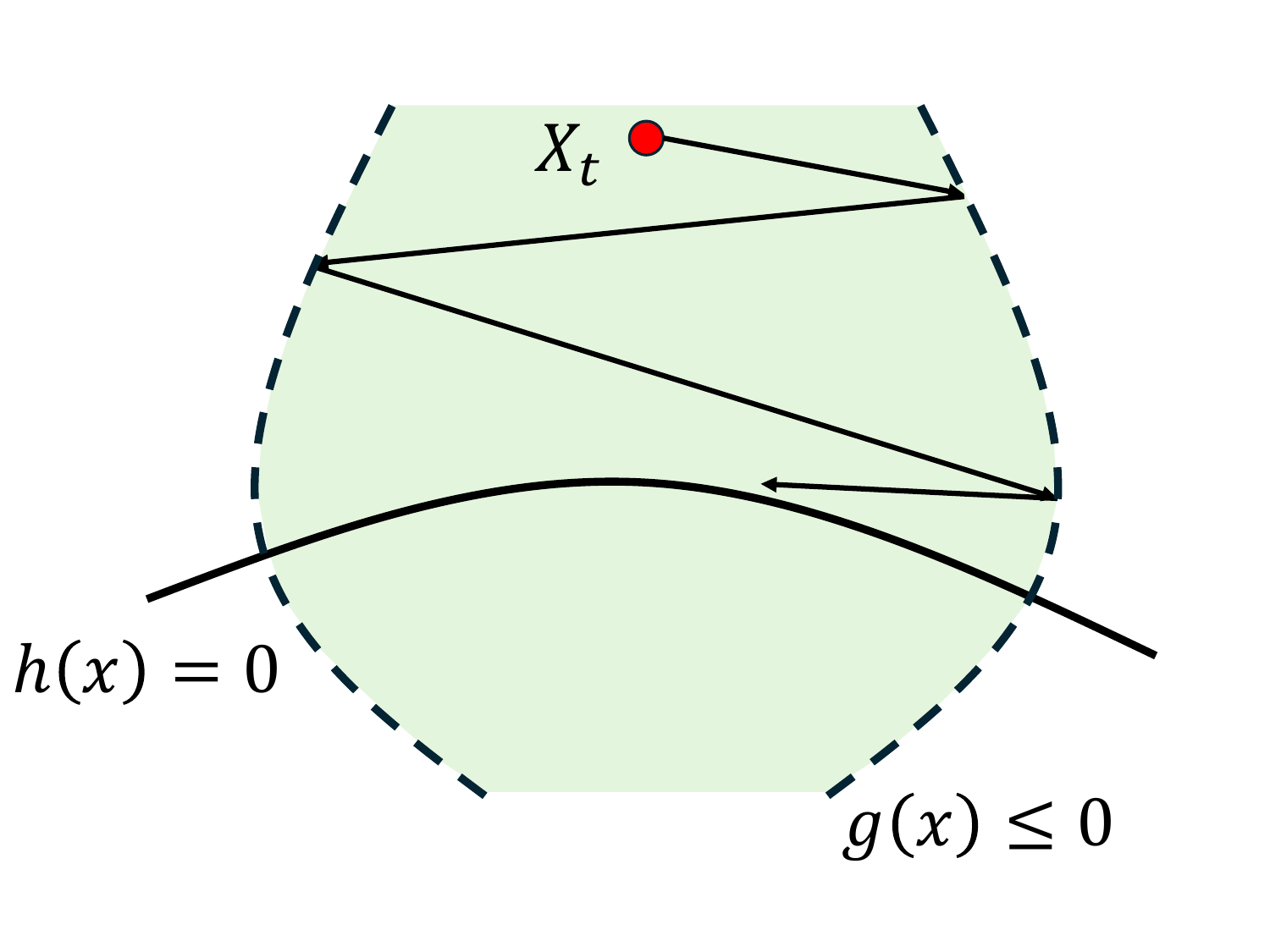}
    \vspace{-12pt}
    \caption{OLLA trajectory in mixed case}
    \label{fig:OLLA trajectory in mixed case}
    \vspace{-20pt}
\end{wrapfigure}

\textbf{Mixed Scenario.} \quad In the mixed setting, OLLA's dynamics are sensitive to the boundary repulsion from $g$. It is noteworthy that $\Sigma$ remains unchanged for different inequality functions $g$, as long as the boundary of the feasible set, where $g(x) =0$ intersects $h(x)=0$, is identical. Nevertheless, the choice of function $g$ affects the landing dynamics, which in turn can alter the convergence rate of OLLA.

To quantify this, we define the projected manifold $\Sigma_p := \pi(\mbra{ x \in \bbR^d \mid h(x) = p, g(x) \leq 0})$ and assume the norm of boundary velocity $v_p^b$ of $\partial \Sigma_p$ is regulated by $V \norm{p}_2^\beta$ for some $V, \beta>0$ (Assumption \ref{asm:M2}). Additionally, we assume $\Sigma_p$ lies inside $\text{int}(\Sigma)$ for $0 < \norm{p}_{2} < \delta$ to avoid stopping behavior of $Y_t$ on $\partial \Sigma$ (Assumption \ref{asm:M1}). Under these assumptions, the trajectory of $X_t$ can be illustrated as in \cref{fig:OLLA trajectory in mixed case} and the proof ideas of equality-only and inequality-only can be seamlessly combined. The following theorem arises from a rigorous integration of previous high-level ideas. We also provide a theorem with a relaxed version of Assumption \ref{asm:M1} in the appendix; see  \cref{remark:Comments on the assumptions} (when $X_0 \sim \delta(x_0)$), \cref{remark:Relaxed assumption of M1}, and \cref{cor:Convergence result for mixed-constrained OLLA mild assumption} for the details.

\begin{theorem}[Convergence result for mixed-constrained OLLA] \ 
    \label{thm:Convergence result for mixed-constrained OLLA}
     Suppose assumptions \ref{asm:C1} to \ref{asm:C4} and \ref{asm:M1} to \ref{asm:M3} hold. Define $X_t$ to be the stochastic process following OLLA dynamics (\ref{eqn:mixed_OLLA_closed_form}) and $\tilde{\rho}_t$ be the law of $Y_t := \pi(X_t)$ after $t \geq t_{\text{cut}}, t_{\text{cut}}:= \max \mbra{\frac{1}{\alpha} \ln \sbra{\frac{M_g +\epsilon}{\epsilon}}, \frac{1}{\alpha} \ln \sbra{\frac{M_h}{\delta}}, \frac{1}{\alpha} \ln (\tilde{C_5})}$.
    Then, for $\alpha >  2\lambda_{\text{LSI}}$ and $\beta \geq 1$, the following non-asymptotic convergence rate of $W_2(\rho_t, \rho_\Sigma)$ can be obtained as follows
    \begin{align*}
        W_2(\rho_t, \rho_\Sigma) \leq \frac{M_h}{\kappa}e^{-\alpha t} + \sqrt{\frac{2}{\lambda_{\text{LSI}}} \KL^\Sigma (\tilde{\rho}_t || \rho_\Sigma)} 
    \end{align*}
    where 
    \begin{equation*}
        \KL^\Sigma (\tilde{\rho}_t ||\rho_\Sigma) \leq \exp \sbra{-2\lambda_{\text{LSI}} (t-t_{\text{cut}}) - \frac{2\lambda_{\text{LSI}} \tilde{C}_5}{\alpha}(e^{-\alpha t} - e^{-\alpha t_{\text{cut}}})} [\KL^\Sigma (\tilde{\rho}_{t_{\text{cut}}} || \rho_\Sigma) + \tilde{C}_7 + \tilde{C}_8] 
    \end{equation*}
    for some constants $G_4, G_5, G_6, \tilde{C}_6 > 0$, $\tilde{C_7}:= (\tilde{C}_6 + \alpha G_4 G_5 M_h) \frac{e^{-\alpha t_{\text{cut}}}}{\alpha - 2\lambda_{\text{LSI}}}$,
    \begin{equation*}
        \tilde{C}_5=\calO \sbra{1+ \frac{\tilde{C}_{L_A}M_h}{\kappa} +\sbra{\frac{\tilde{C}_{L_A}M_h}{\kappa}}^2}, \quad \tilde{C_8}:= (G_6 V M_h^\beta) \frac{e^{-\alpha \beta t_{\text{cut}}}}{\alpha \beta - 2\lambda_{\text{LSI}}},
\end{equation*}
and with $\tilde{C}_{L_A}$ being the Lipschitz constant of $\nabla \pi(x) \Pi(x) $ on $\hat{U}_\delta(\Sigma)$.
\end{theorem}

\vspace{-5pt}
\subsection{Euler-Maruyama Discretization \& Hutchinson Trace Estimation}

    To implement OLLA in practice, we discretize the continuous-time SDE by the Euler-Maruyama (EM) update. At each iteration, we compute three components for the drift vector: the projected gradient drift, the landing drift, and the mean curvature correction $\calH$. In particular, $\calH$ \eqref{eqn:mean_curvature_term} requires forming the full Hessian of each constraint and computing traces of the form $\trace{\Pi \nabla^2 h_i}$ and $\trace{\Pi \nabla^2 g_j}$, an $\calO(d \cdot \text{grad-cost})$ operation that quickly becomes infeasible in high dimensions. We therefore employ the Hutchinson trace estimator \citep{hutchinson1989stochastic}, which gives the following approximation:
    \begin{equation}
        \label{eqn:hutchinson_estimation}
        \trace{\Pi \nabla^2 h_i} \approx \frac{1}{N} \sum_{k=1}^N  (\Pi v_k)^T (\nabla^2 h_i v_k), \quad \trace{\Pi \nabla^2 g_j} \approx \frac{1}{N} \sum_{k=1}^N  (\Pi v_k)^T (\nabla^2 g_j v_k)
    \end{equation}
    for each $i \in [m], j \in I_x$ where $v_k \sim \calN(0, I_d)$ are independent standard normal samples. Each Hessian-Vector Product (HVP) $\nabla^2 h_i v_k$ (or $\nabla^2 g_j v_k$) can be computed at a cost similar to one gradient evaluation, so $N$ probes incur only $\calO (N \cdot \text{grad-cost})$ computational cost rather than $\calO(d \cdot \text{grad-cost})$, saving significant computational cost in high-dimension circumstances. In our experiments, $N=5$ suffices to achieve low variance estimates that match the performance of the full Hessian computation.
    
    By combining these numerical schemes, we arrive at the full algorithm of OLLA in \cref{alg:em_olla}.

%% file: main/experiments.tex
\section{Experiments}
\vspace{-5pt}
\label{main:sec:Experiments}
  To demonstrate the sampling accuracy and efficiency, we compare OLLA and its Hutchinson-accelerated variant (OLLA-H) against three standard constrained samplers: CLangevin \citep{lelievre2012langevin}, CHMC \citep{lelievre2012langevin},  and CGHMC \citep{lelievre2019hybrid}. While the three baselines were originally designed for equality constraints, we introduce one slack variable per inequality constraint via $g_j(x) + \frac{1}{2}\delta_j^2 =0$ for each $j \in [l]$ so that $g_j(x) \leq 0$ is enforced. CGHMC, meanwhile, does not rely on slack variables, but instead uses a Metropolis-Hasting correction step not only to reject samples based on the energy of the proposed samples, but also based on the inequality constraint violation. Therefore, CGHMC allows more accurate and unbiased sampling from the constrained distribution than CHMC and CLangevin in the long run, and we use it as a ground truth for measuring the accuracy of our methods.

\begin{figure}
    \centering
    \includegraphics[width=1\linewidth]{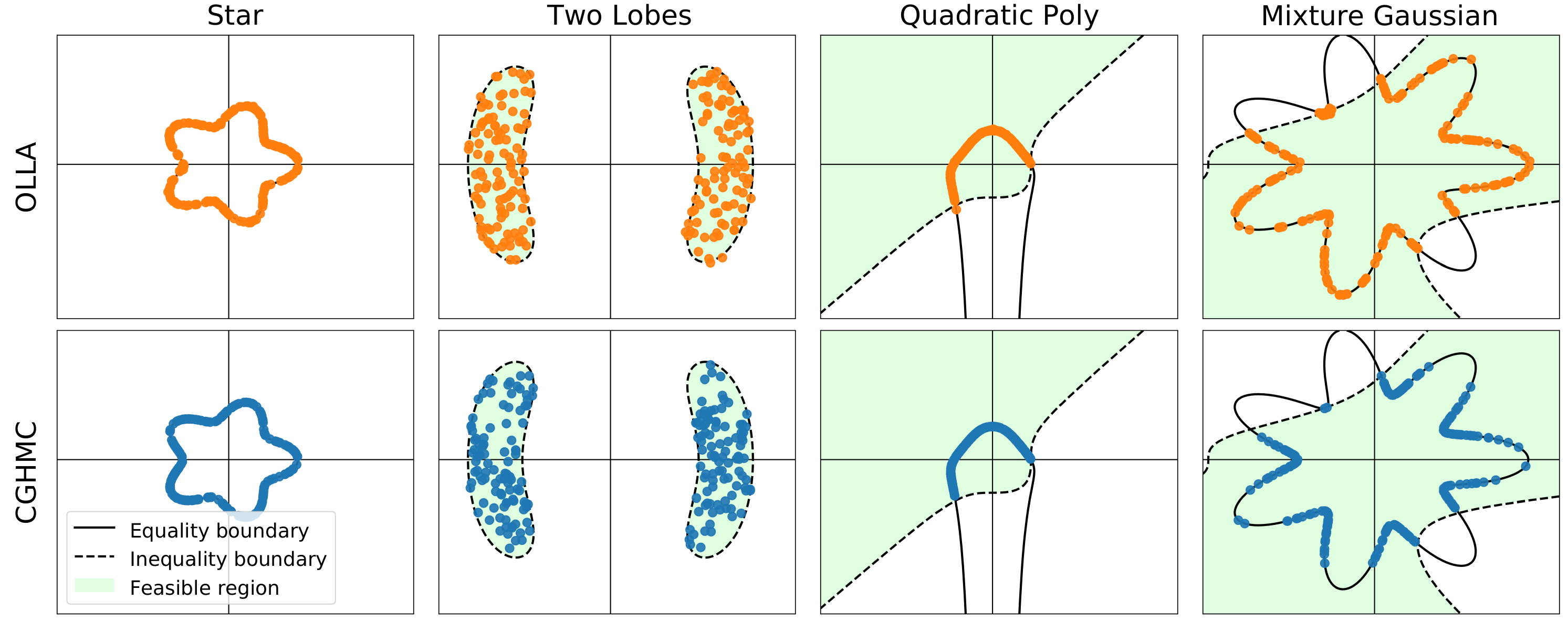}
    \caption{Scatter plots of 200 samples from OLLA (top row) and CGHMC (bottom row) on four 2D synthetic examples. Solid lines show equality constraints, dashed lines show inequality boundaries, and green shaded areas mark feasible region by inequality constraints. OLLA closely matches the CGHMC samples in each scenario.}
    \label{fig:scatter plots of 2D synthetic examples}
    \vspace{-15pt}
\end{figure}
\vspace{-5pt}
\subsection{Synthetic 2D Examples}
\vspace{-5pt}

\label{exp:Synthetic 2D Examples}
We first evaluate sampling on two-dimensional manifolds:
(1) star-shaped equality manifold, (2) two-lobe inequality manifold both with uniform density, (3) manifold defined by quadratic-polynomial equality and inequality constraints with a standard Gaussian target, (4) Gaussian mixture with nine components restricted to a seven lobes manifold by a nonlinear inequality (both for mixed scenario).

For each problem, we ran 200 independent chains in parallel with 5,000 steps, and collected the last 200 samples. We then computed the $W_2$ distance and energy distance between the empirical distribution of the samples and the target distribution, as well as the constraint violation defined by $\bbE[\abs{h(x)}], \bbE[\max{g(x)^+}]$, respectively. The results are shown in \cref{fig:scatter plots of 2D synthetic examples} and \cref{fig:Convergence diagnostics on the 9 Gaussian mixture on 7 Lobe manifold}. Further details of example setups and additional results are included in \cref{app:Experiment setup for synthetic 2D examples}.
\begin{figure} 
    \centering
    \includegraphics[width=1\linewidth]{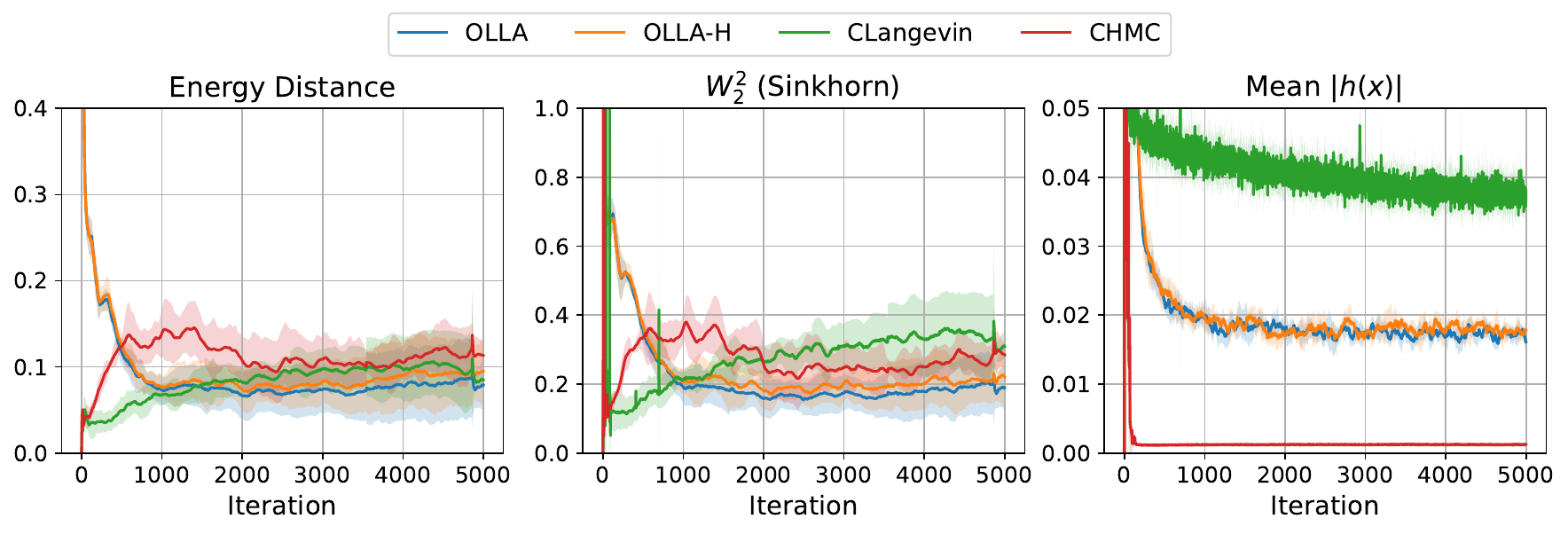}
    \caption{Convergence diagnostics on the Gaussian mixture of 9 components on the 7 lobes manifold with $\alpha = 100, \epsilon = 1$. From left to right: (1) energy distance to CGHMC samples, (2) squared $W_2^2$ distance to CGHMC samples, and (3) mean constraint violation $\bbE[\abs{h}]$. Solid lines and shaded bands show the mean $\pm$1 SD over five independent runs. Both OLLA and OLLA-H rapidly decrease $\bbE[\abs{h}]$ down to small values and maintain it there, which achieving the lowest energy and $W_2^2$ errors.}
    \label{fig:Convergence diagnostics on the 9 Gaussian mixture on 7 Lobe manifold}
    \vspace{-12pt}
\end{figure}

\begin{wrapfigure}{r}{0.38\textwidth}
    \vspace{-15pt}
    \centering
    \captionof{table}{Effect of $\alpha$ on $W_2^2, \bbE[\abs{h}]$}
    \label{tab:Effect of alpha}
    \begin{tabular}{ccc}
        \toprule
        $\alpha$ & $W_2^2$ & $\mathbb{E}[|h(x)|]$ \\
        \midrule
        1   & $0.363\text{\scriptsize$\pm$0.064}$ & $0.682\text{\scriptsize$\pm$0.017}$ \\
        10  & $0.200\text{\scriptsize$\pm$0.035}$ & $0.130\text{\scriptsize$\pm$0.001}$ \\
        100 & $0.159\text{\scriptsize$\pm$0.032}$ & $0.017\text{\scriptsize$\pm$0.001}$ \\
        200 & $0.121\text{\scriptsize$\pm$0.019}$ & $0.008\text{\scriptsize$\pm$0.001}$ \\
        \bottomrule
    \end{tabular}
    \vspace{5pt}
    \captionof{table}{Effect of $\epsilon$ on $W_2^2, \bbE[\max g^+]$}
    \label{tab:Effect of epsilon}
    \begin{tabular}{@{} ccc @{}}
        \toprule
        $\epsilon$ & $W_2^2$ & $\mathbb{E}[\max g^+(x)]$ \\
        \midrule
        0.1 & $0.151\text{\scriptsize$\pm$0.026}$ & $0.082\text{\scriptsize$\pm$0.017}$ \\
        1   & $0.108\text{\scriptsize$\pm$0.011}$ & $0.067\text{\scriptsize$\pm$0.027}$ \\
        5   & $0.123\text{\scriptsize$\pm$0.018}$ & $0.040\text{\scriptsize$\pm$0.015}$ \\
        10  & $0.112\text{\scriptsize$\pm$0.034}$ & $0.019\text{\scriptsize$\pm$0.006}$ \\
        \bottomrule
    \end{tabular}
    \vspace{-15pt}
\end{wrapfigure}

\textbf{Sampling Accuracy \& Constraint Violation of OLLA.} \quad As shown in \cref{fig:Convergence diagnostics on the 9 Gaussian mixture on 7 Lobe manifold}, OLLA and OLLA-H match the performance of CHMC and CLangevin in both Wasserstein and energy-distance metrics, demonstrating that our landing-based approach attains sampling accuracy on par with established methods. Also, constraint violations for OLLA and OLLA-H remained at low levels without computationally expensive projection steps.

\textbf{Effect of Hyperparameters $\alpha$ and $\epsilon$.} \quad
We further examine the influence of the landing rate $\alpha$ and boundary repulsion $\epsilon$ on sampling accuracy and constraint satisfaction. 
As shown in \cref{tab:Effect of alpha,tab:Effect of epsilon}, increasing $\alpha$ accelerates convergence, yielding smaller $W_2^2$ values and reduced equality constraint violations, consistent with our theoretical prediction that larger $\alpha$ enhances the landing and contraction rates toward $\Sigma$. However, excessively large $\alpha$ leads to numerical instability, causing $W_2^2$ to rise and the sampler to collapse (\cref{tab:effect-alpha_app}). 
A similar trend is observed for $\epsilon$. Stronger repulsion lowers inequality violations, but beyond a certain range, $W_2^2$ remains nearly unchanged, aligning with the continuous-time theory that $\epsilon$ primarily affects the finite landing time rather than the asymptotic convergence rate. Overly large $\epsilon$ values can again destabilize the dynamics and lead to numerical breakdown (\cref{tab:effect-epsilon_app}).

\vspace{-5pt}
\subsection{Scaling of OLLA under High-Dimensionality and Large Number of Constraints}
\vspace{-5pt}

\label{exp:Scaling of OLLA under High-dimensionality and large number of constraints}



We assess the robustness and scalability of OLLA and OLLA-H using a synthetic \textit{stress-test} problem that enables explicit control over the ambient dimension $d$ and the numbers of equality and inequality constraints $(m,l)$. 
Samples are drawn from a uniform distribution on a constrained manifold $\Sigma$ defined by linear and quadratic constraints. Each algorithm runs for 1,000 iterations with burn-in and thinning, and we vary one of $d$, $m$, or $l$ while fixing the others to disentangle their individual effects.

We evaluate two key metrics: (1) CPU time per effective sample (CPU/ESS), and (2) the accuracy of representative test function estimates such as $P(x_1>0)$ and $K(x) = \sin(x_1)e^{x_2} + \log(|x_3|+1)\tanh(x_4) + \prod_{i=5}^9\cos(x_i)$. Detailed experimental setups are provided in \cref{app:Experiment setup for high-dimensional and large number of constraint}.

\begin{figure}
    \centering
    \includegraphics[width=1\linewidth]{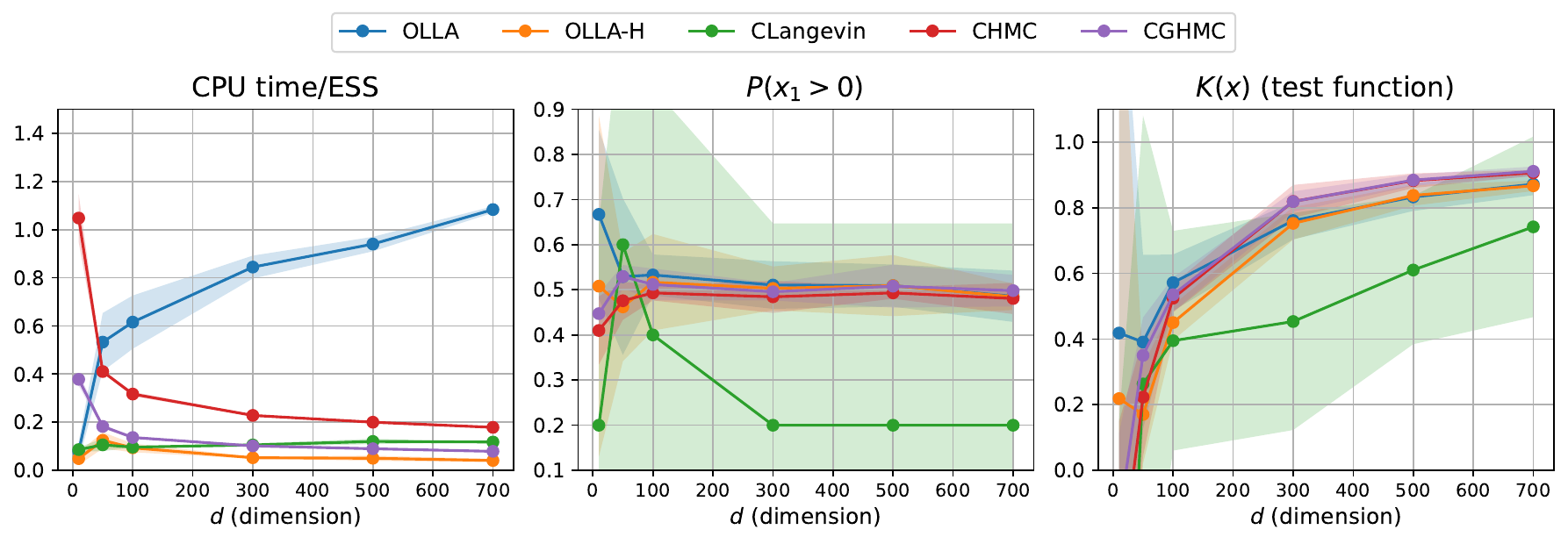}
    \caption{Sampling performance and accuracy as the dimension $d$ increases (with $m=l=5)$. From left to right: (1) CPU time per ESS versus $d$, (2) Estimated probability $P(x_1 >0)$ versus $d$, (3) Estimated value of $K(x)$ versus $d$. Shaded bands shows $\pm$ 1SD over five runs.}
    \label{fig:Sampling performance and accuracy as the dimension d increases}
    \vspace{-16pt}
\end{figure}

\textbf{Scaling under Dimension.} \quad \cref{fig:Sampling performance and accuracy as the dimension d increases} illustrates the sampling performance of algorithms as $d$ increases from $10$ to $700$ (with $m=l=5$). On the left, CPU time/ESS of OLLA-H remains essentially flat around $0.05$s/sample while OLLA grows linearly (reaching $\approx 1.1$s/sample at $d = 700$), and CHMC and CLangevin stay at $0.2-0.3$s/sample and $0.1$s/sample respectively. In the center, both OLLA-H, CHMC, and CGHMC maintain $P(x_1 > 0 ) \approx 0.5$, whereas CLangevin collapses to $\approx 0.2$ in high dimensions, indicating severe bias. On the right, the estimate of nonlinear test function $K(x)$ shows that OLLA, OLLA-H, CHMC, and CGHMC all produce virtually identical estimates even as $d$ grows, while CLangevin lags behind. Overall, the results indicate that OLLA-H scales, maintaining reliable performance as the dimension $d$ increases.

\textbf{Scaling under the Number of Constraints.} \quad With $d=100$, we separately increased the number of equalities $m$ (with $l=5$) and the number of inequalities $l$ (with $m=5$). In the presence of a large number of equalities, OLLA and OLLA-H may lose their edge over equality constrained specialized baselines. In these situations, the maximum $\alpha$ that is stable in the discrete algorithm becomes limited to prevent significant discretization error. However, we observe that even, at near limit value of $\alpha$, $\bbE[\abs{h}]$ may remain relatively high unless the step size $\Delta t$ is further reduced. We add an additional study for this on \cref{app:Experiment setup for high-dimensional and large number of constraint}. By contrast, when adding more inequalities, OLLA-H stays relatively accurate and shows very low CPU time/ESS (\cref{fig:Sampling performance and accuracy as the inequalities l increases}).
\begin{figure}
    \centering
    \vspace{-12pt}
    \includegraphics[width=1\linewidth]{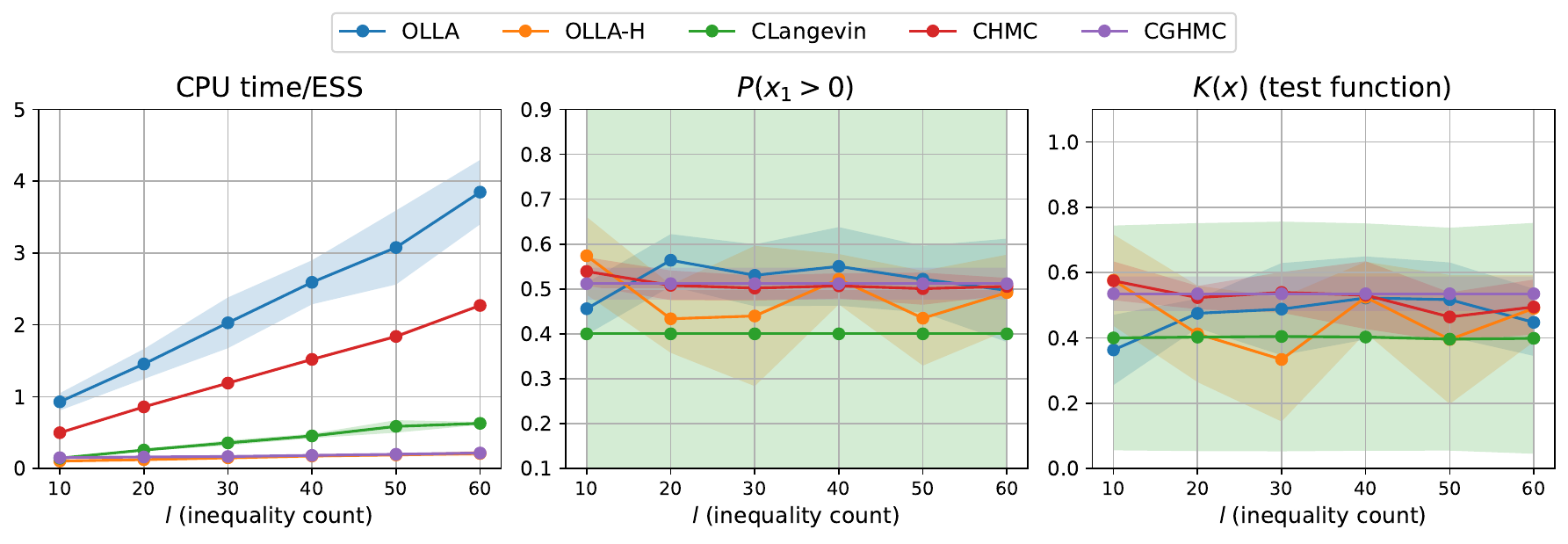}
    \caption{Sampling performance and accuracy as the number of inequality constraints $l$ increases (with $d = 100, m = 5)$. From left to right: (1) CPU time per ESS versus $l$, (2) Estimated probability $P(x_1 >0)$ versus $l$, (3) Estimated value of $K(x)$ versus $l$. Shaded bands shows $\pm$ 1SD over five runs. Note that OLLA and CGHMC results on CPU time/ESS overlap almost perfectly, suggesting their comparable performance on this metric.}
    \label{fig:Sampling performance and accuracy as the inequalities l increases}
    \vspace{-20pt}
\end{figure}

\vspace{-5pt}
\subsection{Molecular System with Realistic Potential}
\vspace{-5pt}
\label{exp:Molecular system with realistic potential}
We further evaluate samplers on a molecular system with realistic potentials and geometric constraints. The system incorporates equality constraints $h$ (fixed bond lengths and angles) and inequality constraints $g$ (steric hindrance), alongside torsion and Weeks-Chandler-Anderson potential terms. Experiments were conducted for increasing dimensions $d=3N_{\text{atom}}$ by varying the number of atoms. Detailed experimental details are provided in \cref{app:Experiment setup for Molecular system and Bayesian logistic regression task} and \cref{tab:polymer_results} summarizes the results. 

\begin{table}[h]
\vspace{-10pt}
\centering
\renewcommand{\cellalign}{cc} 
\setlength{\tabcolsep}{4pt} 
\caption{Estimates of radius of gyration squared ($R_g^2$) and total CPU time for sampling (in brackets) on the molecular system with realistic potential (complex $\Sigma$, small $|I_x|$, large $N_{\text{newton}}$). The average constraint violation for OLLA-H was below $0.007$ (equality), while projection-based samplers maintained violations below $0.0001$. Inequality violations were observed to be $0$ for all algorithms. The $(d,m,l)$ configurations are: $(15, 7, 6)$, $(30, 17, 36)$, $(45, 27, 91)$, $(60, 37, 171)$, $(90, 57, 406)$.}
\label{tab:polymer_results}
\vspace{5pt}
\begin{tabular}{lccccc}
\toprule
method $/$ dim $(d)$ & $15$ & $30$ & $45$ & $60$ & $90$ \\
\midrule
OLLA-H $(N=0)$ & 
\makecell{$1.392$ \scriptsize{$\pm$ $0.026$} \\ {\footnotesize ($\mathbf{78}$s)}} & 
\makecell{$5.414$ \scriptsize{$\pm$ $0.047$} \\ {\footnotesize ($\mathbf{151}$s)}} & 
\makecell{$12.240$ \scriptsize{$\pm$ $0.102$} \\ {\footnotesize ($\mathbf{254}$s)}} & 
\makecell{$21.820$ \scriptsize{$\pm$ $0.098$} \\ {\footnotesize ($\mathbf{334}$s)}} & 
\makecell{$49.080$ \scriptsize{$\pm$ $0.223$} \\ {\footnotesize ($\mathbf{704}$s)}} \\
OLLA-H $(N=5)$ & 
\makecell{$1.370$ \scriptsize{$\pm$ $0.032$} \\ {\footnotesize ($182$s)}} & 
\makecell{$5.424$ \scriptsize{$\pm$ $0.045$} \\ {\footnotesize ($400$s)}} & 
\makecell{$12.200$ \scriptsize{$\pm$ $0.155$} \\ {\footnotesize ($656$s)}} & 
\makecell{$21.780$ \scriptsize{$\pm$ $0.098$} \\ {\footnotesize ($916$s)}} & 
\makecell{$48.940$ \scriptsize{$\pm$ $0.273$} \\ {\footnotesize ($1732$s)}} \\
CLangevin    & 
\makecell{$1.396$ \scriptsize{$\pm$ $0.012$} \\ {\footnotesize ($421$s)}} & 
\makecell{$5.526$ \scriptsize{$\pm$ $0.015$} \\ {\footnotesize ($3620$s)}} & 
\makecell{$12.400$ \scriptsize{$\pm$ $0.000$} \\ {\footnotesize ($10940$s)}} & 
\makecell{$22.140$ \scriptsize{$\pm$ $0.049$} \\ {\footnotesize ($22600$s)}} & 
\makecell{$49.960$ \scriptsize{$\pm$ $0.049$} \\ {\footnotesize ($52220$s)}} \\
CHMC         & 
\makecell{$1.410$ \scriptsize{$\pm$ $0.000$} \\ {\footnotesize ($147$s)}} & 
\makecell{$5.580$ \scriptsize{$\pm$ $0.000$} \\ {\footnotesize ($468$s)}} & 
\makecell{$12.500$ \scriptsize{$\pm$ $0.000$} \\ {\footnotesize ($1012$s)}} & 
\makecell{$22.200$ \scriptsize{$\pm$ $0.000$} \\ {\footnotesize ($1712$s)}} & 
\makecell{$49.960$ \scriptsize{$\pm$ $0.049$} \\ {\footnotesize ($3660$s)}} \\
CGHMC        & 
\makecell{$1.410$ \scriptsize{$\pm$ $0.000$} \\ {\footnotesize ($135$s)}} & 
\makecell{$5.580$ \scriptsize{$\pm$ $0.000$} \\ {\footnotesize ($282$s)}} & 
\makecell{$12.500$ \scriptsize{$\pm$ $0.000$} \\ {\footnotesize ($467$s)}} & 
\makecell{$22.200$ \scriptsize{$\pm$ $0.000$} \\ {\footnotesize ($652$s)}} & 
\makecell{$49.940$ \scriptsize{$\pm$ $0.049$} \\ {\footnotesize ($1116$s)}} \\
\bottomrule
\end{tabular}
\vspace{-5pt}
\end{table}

As dimension $d$ grows, the feasible set $\Sigma$ becomes increasingly complex, while the number of active inequality constraints $|I_x|$ remains small. However, the numerous equality constraints make each projection step demanding for projection-based methods (CHMC, CLangevin, CGHMC), resulting in substantial Newton iteration $N_{\text{newton}}$ and increased computational costs. Particularly, CLangevin is significantly affected from this, whereas CGHMC is faster but still slower than OLLA-H $(N=0)$.

In contrast, increasing the Hutchinson probe $N$ from $N=0$ to $N=5$ improves the mean equality constraint violation across all dimensions $d$, reducing it from a range of $\approx 0.0055$ down to $\approx 0.00035$. Despite this, both OLLA-H configurations $(N=0, N=5)$ yield comparable test function estimation accuracy. This suggests that the constraint violation at $N=0$ was already sufficiently low, rendering the improvement's impact on the final sampling accuracy negligible. Consequently, the $N=0$ variant---where HVP evaluations are completely skipped---achieves significantly lower computational cost. This behavior aligns with findings in \citet{zhang2022sampling}, suggesting that omitting the It\^o-Stratonovich (mean curvature) correction term has negligible practical impact on sampling accuracy, while markedly improving efficiency.

\vspace{-5pt}
\subsection{Bayesian Logistic Regression with Fairness and Monotonicity Constraints}
\vspace{-5pt}

\label{exp:Bayesian NN}
We evaluate the samplers on a high-dimensional Bayesian logistic regression task using a two-layer neural network trained on the German Credit dataset \cite{statlog_(german_credit_data)_144}. 
The setup enforces fairness through equality constraints $h$ ensuring parity in true positive rate and false positive rate between demographic groups, along with monotonicity constraints $g$ on selected input data. Further details of experimental setup are provided in \cref{app:Experiment setup for Molecular system and Bayesian logistic regression task}.

This problem poses significant challenges for projection-based samplers. 
Step sizes effective in unconstrained scenarios led to projection failures for CLangevin and CHMC, and to acceptance rates below $5\%$ for CGHMC. 
To maintain stability, their step sizes were substantially reduced. 
In contrast, the projection-free OLLA-H remained stable across all settings and, for fairness, was also evaluated with the same reduced step size as CLangevin.

As summarized in \cref{tab:blr_fairness}, OLLA-H $(N=0)$ consistently attains the lowest test negative log-likelihood (NLL) while being orders of magnitude faster than projection-based baselines. Although the projection-based methods achieve tighter feasibility, OLLA-H maintains small violations without noticeable degradation in accuracy. These results highlight OLLA-H’s robustness and computational advantage in high-dimensional constrained Bayesian logistic regression task.

\begin{table}[h]
\vspace{-10pt}
\centering
\renewcommand{\cellalign}{tc}
\renewcommand{\arraystretch}{1.15}
\setlength{\tabcolsep}{5pt} 
\caption{
Test NLL and total CPU time for sampling (in brackets) on the Bayesian logistic regression with fairness and monotonicity constraints (high-dimensional $\Sigma$). 
The average constraint violations for OLLA-H were below $0.005$ (equality) and $0.15$ (inequality),
while projection-based samplers (CLangevin, CHMC, CGHMC) maintained feasibility below $0.0008$ (equality) and no inequality violation.
}
\label{tab:blr_fairness}
\vspace{5pt}
\begin{tabular}{lccccc} 
\toprule
method $ /$ dim $(d)$ & $1986$& $4994$ & $9986$ & $49986$ & $100002$ \\ 
\midrule
OLLA-H $(N=0)$ &
\makecell{$0.514$ \scriptsize{$\pm$ $0.013$} \\ {\footnotesize ($\mathbf{63}$s)}} &
\makecell{$0.521$ \scriptsize{$\pm$ $0.008$} \\ {\footnotesize ($\mathbf{70}$s)}} &
\makecell{$0.524$ \scriptsize{$\pm$ $0.014$} \\ {\footnotesize ($\mathbf{70}$s)}} &
\makecell{$0.523$ \scriptsize{$\pm$ $0.011$} \\ {\footnotesize ($\mathbf{81}$s)}} &
\makecell{$0.520$ \scriptsize{$\pm$ $0.015$} \\ {\footnotesize ($\mathbf{82}$s)}} \\

OLLA-H $(N=5)$ &
\makecell{$0.520$ \scriptsize{$\pm$ $0.013$} \\ {\footnotesize ($159$s)}} &
\makecell{$0.524$ \scriptsize{$\pm$ $0.008$} \\ {\footnotesize ($180$s)}} &
\makecell{$0.505$ \scriptsize{$\pm$ $0.004$} \\ {\footnotesize ($205$s)}} &
\makecell{$0.517$ \scriptsize{$\pm$ $0.011$} \\ {\footnotesize ($189$s)}} &
\makecell{$0.516$ \scriptsize{$\pm$ $0.004$} \\ {\footnotesize ($197$s)}} \\

CLangevin &
\makecell{$0.573$ \scriptsize{$\pm$ $0.004$} \\ {\footnotesize ($1162$s)}} &
\makecell{$0.568$ \scriptsize{$\pm$ $0.013$} \\ {\footnotesize ($1176$s)}} &
\makecell{$0.564$ \scriptsize{$\pm$ $0.022$} \\ {\footnotesize ($1194$s)}} &
\makecell{$0.580$ \scriptsize{$\pm$ $0.005$} \\ {\footnotesize ($1428$s)}} &
\makecell{$0.570$ \scriptsize{$\pm$ $0.011$} \\ {\footnotesize ($1370$s)}} \\

CHMC &
\makecell{$0.599$ \scriptsize{$\pm$ $0.015$} \\ {\footnotesize ($526$s)}} &
\makecell{$0.595$ \scriptsize{$\pm$ $0.020$} \\ {\footnotesize ($532$s)}} &
\makecell{$0.599$ \scriptsize{$\pm$ $0.017$} \\ {\footnotesize ($561$s)}} &
\makecell{$0.606$ \scriptsize{$\pm$ $0.004$} \\ {\footnotesize ($586$s)}} &
\makecell{$0.605$ \scriptsize{$\pm$ $0.004$} \\ {\footnotesize ($611$s)}} \\

CGHMC &
\makecell{$0.600$ \scriptsize{$\pm$ $0.007$} \\ {\footnotesize ($76$s)}} &
\makecell{$0.600$ \scriptsize{$\pm$ $0.009$} \\ {\footnotesize ($77$s)}} &
\makecell{$0.606$ \scriptsize{$\pm$ $0.003$} \\ {\footnotesize ($82$s)}} &
\makecell{$0.598$ \scriptsize{$\pm$ $0.020$} \\ {\footnotesize ($83$s)}} &
\makecell{$0.601$ \scriptsize{$\pm$ $0.007$} \\ {\footnotesize ($88$s)}} \\
\bottomrule
\vspace{-25pt}
\end{tabular}
\end{table}

%% file: main/discussions.tex
\section{Conclusion \& Future works}
\vspace{-5pt}
\label{main:sec:Conclusion and Limitations}
We have presented Overdamped Langevin with Landing (OLLA), a projection-free SDE sampler that enforces nonlinear equality and inequality constraints by deterministically “landing” trajectories onto the feasible set while retaining full tangential noise, and proved that its continuous dynamics converge exponentially fast in 2-Wasserstein distance under appropriate regularity.  Building on this, we proposed OLLA-H, an EM discretization that uses a Hutchinson trace estimator for approximating the Itô–Stratonovich correction at only $\mathcal O(N\cdot\text{grad‐cost})$ per step, and showed in both 2D and high-dimensional tests that it matches the accuracy of established constrained samplers while drastically reducing runtime. Future work will include non-asymptotic convergence guarantees for the discrete algorithm—closing the gap between SDE theory and implementation—and developing OLLA variants that remain stable even with many equality constraints in very high dimensions, further extending its scope to large-scale constrained probabilistic inference.

%% file: main_aux/ack_impact.tex
\clearpage
\label{main:sec:ack_impact}
\section*{Acknowledgements}
KJ and MT are grateful for partial supports by NSF Grants DMS-1847802, DMS-2513699, DOE Grants NA0004261, SC0026274, and Richard Duke Fellowship. MM is supported by the German Research Foundation. The authors thank Andre Wibisono for insightful discussion.

%% file: main_aux/check_list.tex
\newpage
\section*{NeurIPS Paper Checklist}

\begin{enumerate}

\item {\bf Claims}
    \item[] Question: Do the main claims made in the abstract and introduction accurately reflect the paper's contributions and scope?
    \item[] Answer: \answerYes{} 
    \item[] Justification: Summarized contributions in Section 1.

\item {\bf Limitations}
    \item[] Question: Does the paper discuss the limitations of the work performed by the authors?
    \item[] Answer: \answerYes{} 
    \item[] Justification: Section 5.2 and Section 6 illustrates a limitation.

\item {\bf Theory assumptions and proofs}
    \item[] Question: For each theoretical result, does the paper provide the full set of assumptions and a complete (and correct) proof?
    \item[] Answer: \answerYes{} 
    \item[] Justification: Appendix A, C to G and Section 2, 4 address theoretical results.

    \item {\bf Experimental result reproducibility}
    \item[] Question: Does the paper fully disclose all the information needed to reproduce the main experimental results of the paper to the extent that it affects the main claims and/or conclusions of the paper (regardless of whether the code and data are provided or not)?
    \item[] Answer: \answerYes{}{} 
    \item[] Justification: Section 5 and Appendix H covers experiment setting to reproduce results. 

\item {\bf Open access to data and code}
    \item[] Question: Does the paper provide open access to the data and code, with sufficient instructions to faithfully reproduce the main experimental results, as described in supplemental material?
    \item[] Answer: \answerYes{} 
    \item[] Justification: Code is provided. Running .ipynb files can reproduce the results. 

\item {\bf Experimental setting/details}
    \item[] Question: Does the paper specify all the training and test details (e.g., data splits, hyperparameters, how they were chosen, type of optimizer, etc.) necessary to understand the results?
    \item[] Answer: \answerYes{} 
    \item[] Justification: Section 5 anb Appendix H illustrates the experimental setting and details.

\item {\bf Experiment statistical significance}
    \item[] Question: Does the paper report error bars suitably and correctly defined or other appropriate information about the statistical significance of the experiments?
    \item[] Answer: \answerYes{} 
    \item[] Justification: All the plots and tables have $\pm$ 1 SD and mean values and trial numbers are specified.

\item {\bf Experiments compute resources}
    \item[] Question: For each experiment, does the paper provide sufficient information on the computer resources (type of compute workers, memory, time of execution) needed to reproduce the experiments?
    \item[] Answer: \answerYes{} 
    \item[] Justification: Appendix H illustrates the computing resources for experiments.
    
\item {\bf Code of ethics}
    \item[] Question: Does the research conducted in the paper conform, in every respect, with the NeurIPS Code of Ethics \url{https://neurips.cc/public/EthicsGuidelines}?
    \item[] Answer: \answerYes{} 
    \item[] Justification: The research conducted in the paper conform, in every respect, with the NeurIPS Code of Ethics \url{https://neurips.cc/public/EthicsGuidelines}

\item {\bf Broader impacts}
    \item[] Question: Does the paper discuss both potential positive societal impacts and negative societal impacts of the work performed?
    \item[] Answer: \answerNA{} 
    \item[] Justification: The paper covers theoretical analysis and introduces an algorithm, which does not relate to societal impacts.
    
\item {\bf Safeguards}
    \item[] Question: Does the paper describe safeguards that have been put in place for responsible release of data or models that have a high risk for misuse (e.g., pretrained language models, image generators, or scraped datasets)?
    \item[] Answer: \answerNA{} 
    \item[] Justification: The paper does not pose such risks.

\item {\bf Licenses for existing assets}
    \item[] Question: Are the creators or original owners of assets (e.g., code, data, models), used in the paper, properly credited and are the license and terms of use explicitly mentioned and properly respected?
    \item[] Answer: \answerNA{} 
    \item[] Justification: The paper does not use existing assets from another parties.

\item {\bf New assets}
    \item[] Question: Are new assets introduced in the paper well documented and is the documentation provided alongside the assets?
    \item[] Answer: \answerNA{} 
    \item[] Justification: The paper does not introduce new assets.

\item {\bf Crowdsourcing and research with human subjects}
    \item[] Question: For crowdsourcing experiments and research with human subjects, does the paper include the full text of instructions given to participants and screenshots, if applicable, as well as details about compensation (if any)? 
    \item[] Answer: \answerNA{} 
    \item[] Justification: The paper does not involve crowdsourcing nor research with human subjects.

\item {\bf Institutional review board (IRB) approvals or equivalent for research with human subjects}
    \item[] Question: Does the paper describe potential risks incurred by study participants, whether such risks were disclosed to the subjects, and whether Institutional Review Board (IRB) approvals (or an equivalent approval/review based on the requirements of your country or institution) were obtained?
    \item[] Answer: \answerNA{}{} 
    \item[] Justification: The paper does not involve crowdsourcing nor research with human subjects.

\item {\bf Declaration of LLM usage}
    \item[] Question: Does the paper describe the usage of LLMs if it is an important, original, or non-standard component of the core methods in this research? Note that if the LLM is used only for writing, editing, or formatting purposes and does not impact the core methodology, scientific rigorousness, or originality of the research, declaration is not required.
    \item[] Answer: \answerNA{} 
    \item[] Justification: LLM was not involved during the development of the research.

\end{enumerate}

%% file: appendix/Notations.tex
\section{Table of Key Notations, Assumptions, and Remarks}
\label{appendix:sec:Table of Key Notations and assumptions}
\renewcommand{\arraystretch}{1.3}

\begin{table}[ht]
\caption{Table of Key Notations}
\label{tab:notation}
\vspace{5pt}
\centering
\renewcommand{\arraystretch}{1.3}
\begin{tabular}{c l l}
\toprule
\textbf{Symbol} & \textbf{Definition} & \textbf{Descriptions} \\
\midrule

$h$
  & $h(x)=[h_1(x),\dots,h_m(x)]^T$
  & Equality constraints \\

$g$
  & $g(x)=[g_1(x),\dots,g_n(x)]^T$
  & Inequality constraints \\

$\Sigma$ 
  & $\{x\in \bbR^d \mid  h(x)=0, g(x)\le0\}$
  & Constraint manifold  \\

$I_{x}$
  & $\{i \in [l] \mid  g_i(x) \geq 0\} = \mbra{i_1,..,i_{\abs{I_x}}}$
  & Active index set of inequalities \\
$g_{I_{x}}$
  & $g_{I_x}(x) = [g_{i_1}(x), ... g_{i_{\abs{I_x}}}(x)]^T$ 
  & Active inequality constraints \\

$J(x)$
  & $\mbra{h(x)^T , g_{i_1}(x)+\epsilon, ..., g_{i_{\abs{I_x}}}(x)+\epsilon}^T$
  & Constraint‐correction vector \\

$\Pi(x)$
  & $I-\nabla J(x)^T G(x)^{-1} \nabla J(x)$
  & Orthogonal projector onto $T_x \Sigma$ \\

$T_x\Sigma$
  & $\{v \in \bbR^d  \mid \nabla h(x)v=0, \nabla g_{I_x}(x)v = 0\}$
  & Tangent space of $\Sigma$ at \(x\) \\

$\nabla_{\Sigma}f$
  & $\Pi(x) \nabla f(x)$
  & Intrinsic gradient on $\Sigma$ (\ref{appendix:sec:Proof of differential operator properties on Sigma}) \\

$\divergence_{\Sigma} X$
  & $\trace{\Pi(x) \nabla X(x)}$
  & Intrinsic divergence on $\Sigma$ (\ref{appendix:sec:Proof of differential operator properties on Sigma}) \\

$d\sigma_{\Sigma}$
  & Induced surface (Hausdorff) measure of $\Sigma$
  & Surface measure on $\Sigma$ \\

$G(x)$
  & $\nabla J(x) \nabla J(x)^T$
  & Gram matrix, full rank assumed\\

$U_\epsilon (\Sigma)$
  & Tubular neighborhood of $\Sigma$ with reach $\epsilon$ 
  & Usual tubular neighborhood\\

$\hat U_\delta (\Sigma)$
  & Recoverable tubular neighborhood with width $\delta$ 
  & See details in (\ref{thm:recoverable tubular neighborhood}, \ref{thm:recoverable tubular neighborhood with boundary})\\

$M_h$
  & $\sup_{x_0\in \text{supp}(\rho_0)} \norm{h(x_0)}_2$
  & Initial bound of $h(x)$ \\

$M_g$
  & $\sup_{x_0\in \text{supp}(\rho_0)} \norm{g(x_0)}_2$
  & Initial bound of $g(x)$  \\

$\pi(x)$
  & $\arg\min_{y\in\Sigma}\norm{x-y}_2$
  & Nearest point projection onto $\Sigma$ \\

$\epsilon$
  & Boundary repulsion rate 
  & Controls effect of repulsion. \\

$\alpha$
  & Landing rate 
  & Controls constraint decay\\

$\lambda_{\rm LSI}$
  & Log–Sobolev constant of $\rho_\Sigma$ on $\Sigma$
  & Enable $\KL^\Sigma \leq \frac{1}{2\lambda_{\text{LSI}}}I^\Sigma$ \\

$\rho_t$
& Density of the process $X_t$ at time $t$
& Law of $X_t$ (following OLLA) \\

$\tilde\rho_t$
& Density of the projected process $Y_t$ at time $t$
& Law of $Y_t := \pi(X_t)$\\

$\rho_\Sigma$
& Target (stationary) density on $\Sigma$
& Proportional to $\exp(-f)d\sigma_\Sigma$ \\

$\KL^\Sigma(\rho\|\pi)$
  & $\int_\Sigma \rho\ln\frac{\rho}{\pi} d\sigma_\Sigma$
  & KL‐divergence on $\Sigma$  \\

$I^\Sigma(\rho\|\pi)$
  & $\int_\Sigma \rho \norm{\nabla_\Sigma\ln\frac{\rho}{\pi}}_2^2 d\sigma_\Sigma$
  & Fisher information on $\Sigma$ \\

$\kappa$
  & Regularity constant of $\Sigma$
  & See details in (\ref{lem:regularity lemma}, \ref{lem:regularity lemma with boundary})\\

$\Sigma_p$
  & $\Sigma_p := \pi(\mbra{x \in \bbR^d \mid h(x) = p , g(x) \leq 0})$
  & Projected manifold for $\norm{p}_2 \leq \delta$\\

$v_p^b$
  & Boundary velocity of $\partial \Sigma_p$
  & See details in \cref{thm:Convergence result for mixed-constrained OLLA}\\
\bottomrule
\end{tabular}
\vspace{1mm}
\end{table}

\clearpage
\begin{assumption}[Assumptions for exponential convergence of OLLA] \ 
\label{asm:assumption for exponential convergence of OLLA}
The followings are assumptions used for proving the convergence results of OLLA:
\begin{enumerate}[label = (C\arabic*)]
    \item
    \label{asm:C1}
    \textbf{LICQ}. \quad The linear independence constraint qualification holds on $\Sigma$ if, for every $x\in \Sigma$, 
    \begin{equation*}
        \mbra{\nabla h_i(x)}_{i=1}^m \cup \mbra{\nabla g_j(x)}_{j \in I_x}
    \end{equation*}
    is a set of linearly independent vectors.

    \item
    \label{asm:C2}
    \textbf{Manifold regularity}.  \quad The Riemannian manifold $\Sigma$ is compact and connected. Also, the equality constraint function $h(x)$ is coercive, i.e., $\norm{h(x)}_2 \rightarrow \infty$ as $\norm{x}_2 \rightarrow \infty$. Assume $\nabla h(x) \neq 0,\  \forall x \in \Sigma$ and $\text{dim}(\Sigma) = d$ when there are only inequality constraints.

    \item
    \label{asm:C3}
    \textbf{Initial constraints bounds}. \quad  The initial distribution $\rho_0$ satisfies
    \begin{equation*}
        M_h := \sup_{x_0 \in \text{supp}(\rho_0)} \norm{h(x_0)}_2 < \infty, \quad M_g := \sup_{x_0 \in \text{supp}(\rho_0)} \norm{g(x_0)}_2 < \infty.
    \end{equation*}

    \item
    \label{asm:C4}
    \textbf{Log-Sobolev constant}. \quad  The stationary distribution $\rho_\Sigma(x)  \propto \exp(-f(x))d\sigma_\Sigma$ satisfies a Log-Sobolev Inequality (LSI) with constant $\lambda_{\text{LSI}}$ so that 
    \begin{equation*}
        \KL^\Sigma (\rho || \rho_\Sigma) \leq \frac{1}{2\lambda_{\text{LSI}}} I^\Sigma (\rho || \rho_\Sigma).
    \end{equation*}
    \textit{Note: Compactness of $\Sigma $ with non-negativity of Ricci curvature guarantees the LSI condition. See \citep{Gross1993LogarithmicSI, otto2000generalization, rothaus1981diffusion, rothaus1986hypercontractivity, wang1997estimation, wang2020fast} or \cref{remark:Comment on LSI} for detailed description.} 
\end{enumerate}

\begin{enumerate}[label = (M\arabic*)]
    \item 
    \label{asm:M1}
    \textbf{Regularity of $\Sigma_p$}. \quad  The projected manifold $\Sigma_p := \pi(\mbra{x \in \bbR^d \mid h(x) = p, g(x) \leq 0})$ lies inside the interior of $\Sigma$ for $0 < \norm{p}_2 <\delta$, where $\delta$ is the width of recoverable tubular neighborhood $\hat{U}_\delta(\Sigma)$.
    \item
    \label{asm:M2}
    \textbf{Regularity of $\partial \Sigma_p$}. \quad The boundary velocity $v_p^b$ of $\partial \Sigma_p$ appearing on Leibniz integral rule satisfies $\sup_{x \in \partial \Sigma_p} \norm{v_p^b}_2 \leq V \norm{p}_2^\beta$ for some $V>0, \beta >0$.  Also, assume $M_\Sigma := \sup_{\norm{p}_2 < \delta} \sigma_{\partial \Sigma_p} (\partial \Sigma_p) < \infty $.
    \item
    \label{asm:M3}
    \textbf{Bound on $\rho_t, \rho_\Sigma$}. \quad  There exists the constants $G_1, G_2, G_3 > 0$ such that $G_1 := \sup_{t \geq 0, x \in \Sigma} \tilde{\rho}_t <\infty$ and $0 < G_2 \leq \rho_\Sigma \leq G_3$ for every $x \in \Sigma$.
\end{enumerate}
\end{assumption}

\begin{remark}[Comments on the assumptions \ref{asm:M1} and \ref{asm:M3}]
    \label{remark:Comments on the assumptions}
    \leavevmode
    \begin{itemize}[leftmargin=*]
        \item Although the assumption \ref{asm:M1} is stated for all small perturbation $p \in \bbR^m, 0<\norm{p}_2 < \delta$, in our anaylsis on \cref{thm:Convergence result for mixed-constrained OLLA}, we only require:
    \begin{equation*}
        \Sigma_t := \pi(\mbra{x \in \bbR^d \mid h(x) = h(X_0)e^{-\alpha t}, g(x) \leq 0}) \subset \text{int}(\Sigma)
    \end{equation*}
    for $t \geq t_{\text{cut}}$, $t_{\text{cut}}:= \max \mbra{\frac{1}{\alpha} \ln \sbra{\frac{M_g +\epsilon}{\epsilon}}, \frac{1}{\alpha} \ln \sbra{\frac{M_h}{\delta}}, \frac{1}{\alpha} \ln (\tilde{C_5})}$. Hence, when $X_0 \sim \delta (x_0)$ for some $x_0 \in \bbR^d$, we can replace the global requirement ``for all $p$ with $0<\norm{p}_2 < \delta$'' by a weaker, one-dimensional condition:
    \begin{enumerate}[label = \textbf{(M\arabic*$'$)}, leftmargin=40pt]
        \item If $X_0 \sim \delta (x_0)$ for some $x_0 \in \bbR^d$, let $u = \frac{h(x_0)}{\norm{h(x_0)}_2}$ and assume 
        \begin{equation*}
            \Sigma_{su} := \pi(\mbra{x \in \bbR^d \mid h(x) = su, g(x) \leq 0}) \subset \text{int}(\Sigma), \quad \text{        for all $s \in (0,\delta)$}.
        \end{equation*}
        \label{asm:M1'}
    \end{enumerate}\vspace{-\baselineskip}
    \item For the uniform boundedness $\sup_{t\geq 0, x \in \Sigma} \tilde{\rho}_t <\infty$ in the assumption \ref{asm:M3}, we recall from the proof of \cref{lem:Upper bound of KL_Sigma_mixed} that $\tilde{\rho}_t$ satisfies the following Fokker-Planck equation:
    \begin{equation*}
        \partial_t \tilde{\rho}_t = -\divergence_\Sigma (\tilde{\rho}_t (\nabla_\Sigma \ln \rho_\Sigma + b_N)) +\sum_{k=1}^d \divergence_\Sigma (\divergence_\Sigma (\tilde{\rho}_t (f_k + \delta_k)) (f_k+\delta_k)).
    \end{equation*}
    Since $\norm{\delta_k}_2 = \calO(e^{-\alpha t})$, the second-order differential operator of the Fokker-Planck equation becomes uniformly elliptic for sufficiently large $t$. In the absense of boundary conditions (equality-only case), the result from \citet{saloff1992uniformly} would then guarantee $\sup_{t\geq 0, x \in \Sigma} \tilde{\rho}_t <\infty$. However, the non-standard boundary conditions imposed on $\Sigma $ in our setting preclude a direct application of those results. We therefore retain the uniform boundedness of $\tilde{\rho}_t$ as an explicit assumption.
    \end{itemize}
\end{remark}

\begin{remark}[Relaxed assumption of \textbf{(M1)}]
    \label{remark:Relaxed assumption of M1}
    We remark that the assumption \textbf{(M1)} can be strong in practice. To mitigate this issue, we propose the following assumption \textbf{(M1$''$)} which is a milder assumption than \textbf{ (M1)}. It assumes 
    \begin{enumerate}[label = (M\arabic*$''$), leftmargin=40pt]
        \item Let $t_1  :=\max\mbra{\frac{1}{\alpha} \ln \sbra{\frac{M_g +\epsilon}{\epsilon}}), \frac{1}{\alpha} \ln \sbra{\frac{M_h}{\delta}}}$ and define $p_t := \bbP(\pi(X_t) \in \partial \Sigma)$ for $t \geq t_1$. Suppose there exist $\gamma_{p}, C_{p} > 0, t_p \geq t_1$ such that $p_t, \partial_t p_t \leq C_p e^{-\gamma_p t} \leq 1/2$ for $t \geq t_p$.
        \label{asm:M1''}
    \end{enumerate}
    We note that this assumption is a necessary condition of the previous assumption (M1), and the landing property guarantees $\lim_{t \rightarrow \infty} p_t = 0$. Also, this assumption ensures that both $p_t, \partial_t p_t$ decays exponentially fast, which can be satisfied depending on how fast the sticking behavior of $\pi(X_t)$ at the bounadry is attenuated as $t$ increases.

\end{remark}

\begin{corollary_ap}[Convergence result for mixed-constrained OLLA with milder assumption] 
\label{cor:Convergence result for mixed-constrained OLLA mild assumption}
    Suppose assumptions \ref{asm:C1} to \ref{asm:C4} and \text{\ref{asm:M1''}} to \ref{asm:M3} hold. Define $X_t$ to be the stochastic process following OLLA dynamics (\ref{eqn:mixed_OLLA_closed_form}) and $\tilde{\rho}_t$ to be the law of $Y_t := \pi(X_t)$ after $t \geq t_{\text{cut}}, t_{\text{cut}} := \max \mbra{\frac{1}{\alpha} \ln \sbra{\frac{M_g +\epsilon}{\epsilon}}, \frac{1}{\alpha} \ln \sbra{\frac{M_h}{\delta}}, \frac{1}{\alpha} \ln (\tilde{C}_5), t_p}$. Then, the following non-asymptotic convergence rate of $W_2(\rho_t, \rho_\Sigma)$ holds for $\alpha >  2\lambda_{\text{LSI}}$, $\beta \geq 1$, and $t > t_{\text{cut}}$:
    \begin{align*}
        W_2(\rho_t, \rho_\Sigma) \leq \frac{M_h}{\kappa}e^{-\alpha t} + \sqrt{\frac{2}{\lambda_{\text{LSI}}} \KL^\Sigma (\tilde{\rho}_t^{\Sigma^{\circ}} || \rho_\Sigma) + C_p e^{-\gamma_p t} \cdot \text{diam}_\Sigma(\Sigma)^2)}
    \end{align*}
    where 
    \begin{align*}
        \mkern-5mu  \KL^\Sigma (\tilde{\rho}_t^{\Sigma^\circ} || \rho_\Sigma) \leq & \exp \sbra{ \mkern-5mu  -2\lambda_{\text{LSI}} (t-t_{\text{cut}}) - \frac{2\lambda_{\text{LSI}} \tilde{C}_5}{\alpha}(e^{-\alpha t} - e^{-\alpha t_{\text{cut}}}) \mkern-5mu  -  \mkern-5mu \frac{C_p}{\gamma_p}\sbra{e^{-\gamma_{p}t} - e^{-\gamma_{p}t_{cut}}}} \times \\
        & [\KL^\Sigma (\tilde{\rho}_{t_{\text{cut}}}^{\Sigma^{\circ}} || \rho_\Sigma) + \tilde{C}_7^{\Sigma^{\circ}} + \tilde{C}_8^{\Sigma^{\circ}}]
    \end{align*}
    for some constants $G_4^{\Sigma^{\circ}}, G_5, G_6^{\Sigma^{\circ}}, \tilde{C}_6 > 0$, $\tilde{C}_7^{\Sigma^{\circ}}:= (\tilde{C}_6 + \alpha G_4^{\Sigma^{\circ}} G_5 M_h) \frac{e^{-\alpha t_{\text{cut}}}}{\alpha - 2\lambda_{\text{LSI}}}$,
    \begin{equation*}
        \tilde{C}_5=\calO \sbra{1+ \frac{\tilde{C}_{L_A}M_h}{\kappa} +\sbra{\frac{\tilde{C}_{L_A}M_h}{\kappa}}^2}, \quad \tilde{C}_8^{\Sigma^{\circ}}:= (G_6^{\Sigma^{\circ}} V M_h^\beta) \frac{e^{-\alpha \beta t_{\text{cut}}}}{\alpha \beta - 2\lambda_{\text{LSI}}},
\end{equation*}
and with $\tilde{C}_{L_A}$ being the Lipschitz constant of $\nabla \pi(x) \Pi(x) $ on $\hat{U}_\delta(\Sigma)$, $\tilde{\rho_t}^{\Sigma^{\circ}}$ being the conditional laws of $Y_t$ given $Y_t \in \text{int}(\Sigma)$.
\end{corollary_ap}
\begin{proof}
    We slightly adjust the argument in the proof of \cref{thm:Convergence result for mixed-constrained OLLA} to show this statement, and the definitions of constants are shared with \cref{thm:Convergence result for mixed-constrained OLLA}. We first observe that the probability measure $\tilde{\rho}_t$ can be decomposed as follows:
    \begin{equation*}
        \tilde{\rho}_t  = (1-p_t) \tilde{\rho}_t^{\Sigma^\circ} + p_t \tilde{\rho}_t^{\partial \Sigma}, \quad t \geq t_{\text{cut}}
    \end{equation*}
    where $p_t:= \tilde{\rho}_t (\partial \Sigma)$ and 
    \begin{equation*}
        \tilde{\rho}_t^{\Sigma^\circ} := \frac{\tilde{\rho}_t \vert_{\Sigma^\circ}}{1-p_t}, \quad \tilde{\rho}_t^{\partial_\Sigma} := \frac{\tilde{\rho}_t \vert_{\partial_\Sigma}}{p_t},
    \end{equation*}
    are the conditional laws of $\tilde{\rho}_t$ restricted on $\Sigma$ and $\partial \Sigma$, respectively.
    
    Because $\tilde{\rho}_t$ is the convex combination of $\tilde{\rho}_t^{\Sigma^\circ}$ and $\tilde{\rho}_t^{\partial \Sigma}$, the convexity of the 2-Wasserstein distance implies
    \begin{equation*}
        W_2^\Sigma (\tilde{\rho}_t, \rho_\Sigma)^2 \leq (1-p_t)W_2^\Sigma(\tilde{\rho}_t^{\Sigma^\circ}, \rho_\Sigma)^2 + p_t W_2^\Sigma (\tilde{\rho}_t^{\partial \Sigma}, \rho_\Sigma)^2.
    \end{equation*}
    Here, we note that 
    \begin{equation*}
        W_2^\Sigma(\tilde{\rho}_t^{\Sigma^\circ}, \rho_\Sigma) \leq \sqrt{\frac{2}{\lambda_{\text{LSI}}} \KL^\Sigma (\tilde{\rho}_t^{\Sigma^\circ} || \rho_\Sigma)}
    \end{equation*}
    by \cref{lem:LSI implies Talagrand inequality} and the fact that $\tilde{\rho}_t^{\Sigma^\circ}$ has its support on $\text{int}(\Sigma)$. Also, we note that 
    \begin{equation*}
        \partial_t \tilde{\rho}_t^{\Sigma^\circ} = \frac{1}{1-p_t}\sbra{\partial_t \tilde{\rho}_t + (\partial_t p_t) \tilde{\rho}_t^{\Sigma^\circ}}
    \end{equation*}
    because $\tilde{\rho}_t = (1-p_t) \tilde{\rho}_t^{\Sigma^\circ}$ holds for $x \in \text{int}({\Sigma})$. Then, the same approach used in \cref{lem:Upper bound of KL_Sigma_mixed} gives
    \begin{align*}
        \partial_t \KL^\Sigma(\tilde{\rho}_t^{\Sigma^\circ} ||\rho_\Sigma) &= \partial_t \int_{\Sigma_t} \tilde{\rho}_t^{\Sigma^\circ} \ln \sbra{\frac{\tilde{\rho}_t^{\Sigma^\circ}}{\rho_\Sigma}} d\sigma_\Sigma \\
        &= \underbrace{\frac{1}{1-p_t}\int_{\Sigma_t} \partial_t \tilde{\rho}_t \ln \sbra{\frac{\tilde{\rho}_t^{\Sigma^\circ}}{\rho_\Sigma}} d\sigma_\Sigma}_{\text{Term (1)}} + \underbrace{\int_{\partial \Sigma_t} \tilde{\rho}_t^{\Sigma^\circ} \lbra{\ln \sbra{\frac{\tilde{\rho}_t^{\Sigma^\circ}}{\rho_\Sigma}} - 1} \inner{v_t^b, n_t} d\sigma_{\partial \Sigma_t}}_{\text{Term (2)}} \\
        &+   \underbrace{\frac{\partial_t p_t}{1-p_t}\KL^\Sigma \sbra{\tilde{\rho}_t^{\Sigma^\circ} || \rho_\Sigma}}_{\text{Term (3)}} + \underbrace{\partial_t \int_{\Sigma_t} \tilde{\rho}_t^{\Sigma^\circ} d\sigma_\Sigma}_{=0}
    \end{align*}
    where the last equality holds from the Leibniz integral rule with $v_t^b$ being the velocity vector of the boundary of $\Sigma_t := \pi(\mbra{ x \in \bbR^d \mid h(x) = h(X_0)e^{-\alpha t}, g(x) \leq 0 })$. Therefore, the expression of $\partial_t \tilde{\rho}_t$ implies that Term (1) becomes
    \begin{align*}
        \text{Term (1)} &= \underbrace{\int_{\Sigma_t} \lbra{\tilde{\rho}_t^{\Sigma^\circ} (\nabla_\Sigma \ln \rho_\Sigma + b_N) -\sum_{k=1}^d \divergence_\Sigma (\tilde{\rho}_t^{\Sigma^\circ} (f_k + \delta_k)) (f_k+\delta_k)}\nabla_\Sigma \ln \sbra{\frac{\tilde{\rho}_t^{\Sigma^\circ}}{\rho_\Sigma}} d\sigma_\Sigma}_{\text{Term (1-1)}} \\
        & \mkern-20mu - \underbrace{\int_{\partial \Sigma_t} \inner{\lbra{\tilde{\rho}_t^{\Sigma^\circ} (\nabla_\Sigma \ln \rho_\Sigma + b_N) -\sum_{k=1}^d \divergence_\Sigma (\tilde{\rho}_t^{\Sigma^\circ} (f_k + \delta_k)) (f_k+\delta_k)} \ln \sbra{\frac{\tilde{\rho}_t^{\Sigma^\circ}}{\rho_\Sigma}}, n_t} d\sigma_{\partial \Sigma_t}}_{\text{Term (1-2)}}
    \end{align*}
    using $\tilde{\rho}_t = (1-p_t)\tilde{\rho}_t^{\Sigma^\circ}$. Therefore, \cref{lem:Upper bound of KL_Sigma_mixed} implies Term (1) can be bounded by
    \begin{equation*}
        \abs{\text{Term (1)}} \leq -(1-\tilde{C}_5 e^{-\alpha t}) I^{\Sigma}(\tilde{\rho}_t^{\Sigma^\circ} || \rho_\Sigma) + \tilde{C}_6 e^{-\alpha t} + \alpha G_4^{\Sigma^{\circ}} G_5 M_h e^{-\alpha t}
    \end{equation*}
    with $G_4^{\Sigma^{\circ}} := G_3 M_\Sigma \max \mbra{\frac{1}{e}, \abs{\frac{G_1^{\Sigma^{\circ}}}{G_2} \ln \sbra{\frac{G_1^{\Sigma^{\circ}}}{G_2}}}}$ and $G_1^{\Sigma^{\circ}} := \sup_{t \geq t_{cut}, x \in \Sigma} \tilde{\rho}_t^{\Sigma^\circ} \leq 2 G_1$, and the Term (2) can be bounded by
    \begin{equation*}
        \abs{\text{Term (2)}} \leq G_6^{\Sigma^{\circ}} V M_h^\beta e^{-\alpha \beta t}.
    \end{equation*}
    with $G_6^{\Sigma^{\circ}} := G_4^{\Sigma^{\circ}} + G_1^{\Sigma^{\circ}} M_\Sigma$. Also, Term (3) is upper bounded by
    \begin{equation*}
        \text{Term (3)} \leq  C_p e^{-\gamma_p t} \KL^\Sigma (\tilde{\rho}_t^{\Sigma^\circ} || \rho_\Sigma)
    \end{equation*}
    Therefore, combining the bounds with the LSI condition gives the following inequality:
    \begin{align*}
        \mkern-5mu \partial_t \KL^\Sigma (\tilde{\rho}_t^{\Sigma^\circ} || \rho_\Sigma) \mkern-5mu  & \leq \mkern-5mu  -2 \lambda_{\text{LSI}} \sbra{1- \tilde{C}_5 e^{-\alpha t} \mkern-5mu - \frac{C_p}{2\lambda_{\text{LSI}}}e^{-\gamma_{p}t} } \KL^\Sigma (\tilde{\rho}_t^{\Sigma^\circ} || \rho_\Sigma) + \sbra{\tilde{C}_6  \mkern-5mu  + \alpha G_4^{\Sigma^{\circ}}G_5 M_h} e^{-\alpha t} \\
        & +G_6^{\Sigma^{\circ}} V M_h^\beta e^{-\alpha \beta t}.
    \end{align*}
    Hence, applying the Gr\"onwall-type inequality recovers the following inequality:
    \begin{align*}
        &\KL^\Sigma (\tilde{\rho}_t^{\Sigma^\circ} || \rho_\Sigma) \leq \exp \sbra{-2\lambda_{\text{LSI}} (t-t_{\text{cut}}) - \frac{2\lambda_{\text{LSI}} \tilde{C}_5}{\alpha}(e^{-\alpha t} - e^{-\alpha t_{\text{cut}}}) - \frac{C_p}{\gamma_p} \sbra{e^{-\gamma_{p}t} - e^{-\gamma_{p}t_{cut}}}} \times \\
        & [\KL^\Sigma (\tilde{\rho}_{t_{\text{cut}}}^{\Sigma^{\circ}} || \rho_\Sigma) +  \mkern-5mu \int_{t_{cut}}^t  \mkern-15mu  \exp  \mkern-5mu \sbra{2\lambda_{\text{LSI}} (s-t_{\text{cut}}) + \frac{2\lambda_{\text{LSI}} \tilde{C}_5}{\alpha}(e^{-\alpha s} - e^{-\alpha t_{\text{cut}}}) + \frac{C_p}{\gamma_p}\sbra{e^{-\gamma_{p}s} - e^{-\gamma_{p}t_{cut}}} \mkern-5mu } \times\\
        &\lbra{ \sbra{\tilde{C}_6  + \alpha G_4^{\Sigma^{\circ}}G_5 M_h} e^{-\alpha s} + G_6^{\Sigma^{\circ}} V M_h^\beta e^{-\alpha \beta s}} ds]
    \end{align*}
    which can again be summarized as follows:
    \vspace{-1cm}
    \begin{align*}
        \KL^\Sigma (\tilde{\rho}_t^{\Sigma^\circ} || \rho_\Sigma) \leq & \exp \sbra{-2\lambda_{\text{LSI}} (t-t_{\text{cut}}) - \frac{2\lambda_{\text{LSI}} \tilde{C}_5}{\alpha}(e^{-\alpha t} - e^{-\alpha t_{\text{cut}}}) - \frac{C_p}{\gamma_p}\sbra{e^{-\gamma_{p}t} - e^{-\gamma_{p}t_{cut}}}} \times \\
        & [\KL^\Sigma (\tilde{\rho}_{t_{\text{cut}}}^{\Sigma^{\circ}} || \rho_\Sigma) + \tilde{C}_7^{\Sigma^{\circ}} + \tilde{C}_8^{\Sigma^{\circ}}]
    \end{align*}
    where $\tilde{C}_7^{\Sigma^{\circ}}:= (\tilde{C}_6 + \alpha G_4^{\Sigma^{\circ}} G_5 M_h) \frac{e^{-\alpha t_{\text{cut}}}}{\alpha - 2\lambda_\text{LSI}}$ and $\tilde{C}_8^{\Sigma^{\circ}}:= (G_6^{\Sigma^{\circ}} V M_h^\beta) \frac{e^{-\alpha \beta t_{\text{cut}}}}{\alpha \beta - 2\lambda_\text{LSI}}$.
    
    Lastly, we observe that
    \begin{equation*}
        W_2^\Sigma (\tilde{\rho}_t^{\partial \Sigma}, \rho_\Sigma) := \sbra{ \inf\limits_{\pi \in \Pi(\tilde{\rho}_t^{\partial \Sigma}, \rho_\Sigma)}\int_{\Sigma \times \Sigma} d_\Sigma (p,q)^2 \pi(dp,dq)}^{\frac{1}{2}} \leq \text{diam}_\Sigma (\Sigma)
    \end{equation*}
    where $\text{diam}_\Sigma (\Sigma) := \sup \mbra{d_\Sigma (x,y) \mid  x,y \in \Sigma} < \infty$ due to the compactness of $\Sigma$. Therefore, combining the above upper bounds, we recover the following inequality:
    \begin{equation*}
        W_2^\Sigma(\tilde{\rho_t},\rho_\Sigma) \leq \sqrt{\frac{2}{\lambda_{\text{LSI}}} \KL^\Sigma (\tilde{\rho}_t^{\Sigma^\circ} || \rho_\Sigma) + C_p e^{-\gamma_{p}} \cdot \text{diam}_\Sigma(\Sigma)^2)}
    \end{equation*}
    Therefore, applying the same reasoning as in \cref{thm:Convergence result for mixed-constrained OLLA} completes the proof.
\end{proof}

\begin{remark}[Effect of number of Hutchinson probes $N$ on the decaying speed of constraint functions] 
     In this remark, under the single equality case, we demonstrate that exponential decay of $h$ via OLLA-H and analyze the effects of the number of Hutchinson probes $N$. When there is only a single equality constraint, we recall that the OLLA-H update rule is given as
     \begin{equation*}
         X_{k+1} = X_k + \sbra{b(X_k) + \epsilon_k(X_k)}\Delta t + \sqrt{2\Delta t} \Pi(X_k) \xi_k, \quad \xi_k \sim \calN(0,I_d) 
     \end{equation*}
    where $b(x) := -\Pi(x) \nabla f(x) - \alpha \nabla h(x) G(x)^{-1}h(x) - \nabla h(x) G^{-1}(x) \trace{\Pi(x) \nabla^2 h(x)}$
    and
    \begin{equation*}
        \hat{S}(x) = \frac{1}{N} \sum_{i=1}^N v_{k,i}^T \Pi(x) \nabla^2 h(x) v_{k,i}, \quad \epsilon_k(x) := -\nabla h(x) G(x)^{-1} \sbra{ \hat{S}(x)- \trace{S(x)}}
    \end{equation*}
    with $S(x) := \Pi(x) \nabla^2h(x)$ and $v_{k,i} \sim \calN(0,I_d)$. Also, we note that $\bbE \lbra{\epsilon_k(X_k) \mid X_k} = 0$ and
    \begin{equation*}
        \bbE \lbra{\epsilon_k(X_k)\epsilon_k(X_k)^T \mid X_k} \mkern-5mu  = \mkern-5mu  \frac{\nabla h(X_k) \nabla h(X_k)^T}{\norm{\nabla h(X_k)}^4} \cdot \text{Var} (\hat{S}(X_k) | X_k) = \frac{2\nabla h(X_k) \nabla h(X_k)^T}{N \norm{\nabla h(X_k)}^4} \norm{S(X_k)}_F^2.
    \end{equation*}
     
    \begin{proposition_ap}[Exponential decay of constraint function under OLLA-H]. Assume the single equality constraint scenario and the following conditions:
    \begin{itemize}[leftmargin=*]
        \item (Dissipativity of $b(x)$) \quad $\inner{x, b(x)} \leq -m \norm{x}^2 + c$ \ for some $m>0, c\geq 0$.
        \item (Linear growth of $b(x)$) \quad $\norm{b(x)}^2 \leq A + B \norm{x}^2$ \ for some $A, B \geq 0$.
        \item (Boundedness of tangential Hessian-gradient ratio) \quad $C_h = \sup_{x \in \bbR^d} \frac{\norm{\Pi(x) \nabla^2 h(x)}_F^2}{\norm{\nabla h(x)}^2} < \infty$.
        \item ($L_1$-smoothness of $h$) \quad For some $L_1 >0$,  $\norm{\nabla^2 h(x)} \leq L_1$ for any $x \in \bbR^d$.
        \item ($L_2$-Lipschitzness of $\nabla^2 h$) \quad For some $L_2 >0$, $\norm{\nabla^2 h(x) - \nabla^2 h(y)} \leq L_2 \norm{x-y}$ for any $x \in \bbR^d$.
        \item (Step size control) \quad $0 < \Delta t < \min \mbra{1, \frac{m}{B}, \frac{1}{2m}}$.
    \end{itemize}
    Then, the OLLA-H dynamics with $N$ Hutchinson probes have the following decaying property of $h$:
    \begin{equation*}
        \bbE \lbra{h(X_K)} \leq \sbra{1-\alpha \Delta t}^K \bbE \lbra{h(X_0)} + \calO \sbra{\frac{\Delta t}{\alpha} \sbra{d + \frac{1}{N}}}
    \end{equation*}
    for $K \geq 0$.
    \end{proposition_ap}
    \begin{proof}
        We first prove that $\sup_{k \geq 0} \bbE \norm{\Delta X_k}^2 <\infty$ where $\Delta X_k := X_{k+1} - X_k$. To show this, we define $\calF_k := \sigma(X_m, \xi_m, v_{m,j} \mid m \leq k, j \leq N)$ to be the $k$ th canonical filtration and observe that
        \begin{align*}
            \bbE \lbra{\norm{X_{k+1}}^2 \mid \calF_k} &= \norm{X_k}^2 + 2 \Delta t \inner{X_k, b(X_k)} + \Delta t^2 \norm{b(X_k)}^2 + \Delta t^2 \bbE \lbra{\norm{\epsilon_k(X_k)}^2 \mid X_k} \\
            &+ 2\Delta t \bbE \lbra{ \norm{\Pi(X_k) \xi_k}^2 \mid X_k}
        \end{align*}
        because $\bbE\lbra{\xi_k \mid X_k} = \bbE \lbra{\epsilon_k (X_k) \mid X_k}=0$ and $\Delta X_k = (b(X_k) + \epsilon_k(X_k))\Delta t + \sqrt{2\Delta t}\Pi(X_k) \xi_k$. Applying the disspativity and linear growth assumption, we have
        \begin{align*}
            \bbE \lbra{\norm{X_{k+1}}^2 \mid \calF_k} &\leq \sbra{1 - (2m - B \Delta t) \Delta t} \norm{X_k}^2 + \sbra{2c + 2(d-1) +(A+ \frac{2C_h}{N}) \Delta t}\Delta t
        \end{align*}
        because 
        \begin{equation*}
            \bbE \lbra{ \norm{\epsilon_k (X_k)}^2 \mid X_k} \leq \frac{2}{N} \frac{\norm{\Pi(X_k) \nabla^2 h(X_k)}_F^2}{\norm{\nabla h(X_k)}^2} \leq \frac{2C_h}{N}
        \end{equation*}
        holds from the tangential Hessian-gradient ratio assumption, and $\bbE \lbra{ \norm{\Pi(X_k) \xi_k}^2 \mid X_k} = d-1$. Because the step size control assumption guarantees $(1- (2m-B\Delta t)\Delta t) \in (0,1)$, we have the following inequality:
        \begin{equation*}
            \sup_{k\geq 0} \bbE[\norm{X_{k}}^2] \leq \frac{2c + 2(d-1) + (A+\frac{2C_h}{N})\Delta t}{2m-B \Delta t} :=  M_N< \infty
        \end{equation*}
        by taking expectations and iterating the recursion. Therefore, it holds that
        \begin{align*}
            \bbE \lbra{ \norm{\Delta X_k}^2 \mid \calF_k } =\bbE \lbra{ \norm{\Delta X_k}^2 \mid X_k }&= \Delta t^2 \norm{b(X_k)}^2 + \Delta t^2 \bbE \lbra{\norm{\epsilon_k(X_k)}^2 \mid X_k} + 2\Delta t (d-1) \\
            &\leq \Delta t^2 \sbra{A + B \norm{X_k}^2} + \frac{2C_h \Delta t^2 }{N} + 2\Delta t(d-1) 
        \end{align*}
        and 
        \begin{equation*}
            \sup_{k\geq0} \bbE \norm{\Delta X_k}^2 \leq \Delta t^2 \sbra{A + B M_N + \frac{2C_h}{N}} +2\Delta t (d-1) := M_1 \Delta t^2 + M_2 \Delta t< \infty
        \end{equation*}
        with $M_1 := A+BM_N + \frac{2C_h}{N}$ and $M_2 = 2(d-1)$.

    Similarly, let us define $B(X_k) := (b(X_k) + \epsilon_k(X_k))\Delta t$. Then, under the similar proof, it holds that
    \begin{align*}
        \bbE \lbra{\norm{B(X_k)}^2 \mid \calF_k} = \bbE \lbra{\norm{B(X_k)}^2 \mid X_k} &= \Delta t^2 \norm{b(X_k)}^2 + \Delta t^2 \bbE \lbra{\norm{\epsilon_k(X_k)}^2 \mid X_k} \\
        &\leq \Delta t^2 \sbra{A + B \norm{X_k}^2} + \frac{2C_h \Delta t^2 }{N}
    \end{align*}
    and 
    \begin{equation}
        \sup_{k \geq 0} \bbE \lbra{\norm{B(X_k)}^2} \leq \Delta t^2 \sbra{A+ BM_N + \frac{2C_h}{N}} = \Delta t^2 M_1.
    \end{equation}
    
    Next, we prove the decaying property of $h:\bbR^d \rightarrow \bbR$. From the 2nd order Taylor expansion on $h$, the following holds almost surely:
    \begin{align*}
        h(X_{k+1}) &= h(X_k) + \nabla h(X_k) \Delta X_k + \frac{1}{2}\Delta X_k^T \nabla^2 h( \tilde{X}_k) \Delta X_k \\
        &= (1 - \alpha \Delta t) h(X_k) - \trace{\Pi(X_k) \nabla^2 h(X_k)}\Delta t+ \sbra{\zeta_k^T \Pi(X_k) \nabla^2 h(X_k) \Pi(X_k) \zeta_k} \Delta t \\
        &+ \frac{1}{2}\Delta X_k^T \nabla^2 h( \tilde{X}_k) \Delta X_k - \sbra{\zeta_k^T \Pi(X_k) \nabla^2 h(X_k) \Pi(X_k) \zeta_k} \Delta t
    \end{align*}
    where $\tilde{X}_k$ is some point $\in \bbR^d$ between $X_k$ and $X_{k+1}$ and $\Delta X_k := X_{k+1} - X_k$. Therefore, by applying thet previous observations, we get the following:
    \begin{align*}
        \bbE[h(X_{k+1}) | \calF_k] &\leq (1-\alpha \Delta t) h(X_k) + \frac{L}{2} \bbE \lbra{\norm{B (X_k)}^2 \mid X_k} + L \Delta t \bbE \lbra{\norm{\Delta X_k}^2 \mid X_k}
    \end{align*}
    which implies
    \begin{align*}
        \bbE[h(X_{k+1})] \leq (1-\alpha \Delta t) \bbE \lbra{ h(X_k)} +\frac{L\Delta t^2 (M_1 (1+ 2\Delta t) + 2M_2)}{2}.
    \end{align*}
    By applying the telescoping sum with expectation, we recover the following formula:
    \begin{equation*}
        \bbE \lbra{h(X_K)} \leq \sbra{1-\alpha \Delta t}^K \bbE \lbra{h(X_0)} + \frac{L\Delta t (M_1 (1+ 2\Delta t) + 2M_2)}{2 \alpha}.
    \end{equation*}
    Finally, we note that 
    \begin{equation*}
        M_N = \calO \sbra{d + \Delta t \sbra{1+ \frac{1}{N}}}, \quad M_1 = \calO \sbra{d + \Delta t \sbra{1+ \frac{1}{N}} + \frac{1}{N}}, \quad M_2 = \calO \sbra{d}
    \end{equation*}
    and, therefore,
    \begin{equation*}
        \bbE \lbra{h(X_K)} \leq \sbra{1-\alpha \Delta t}^K \bbE \lbra{h(X_0)} + \calO \sbra{\frac{\Delta t}{\alpha} \sbra{d + \frac{1}{N}}}.
    \end{equation*}
    \end{proof}
\end{remark}
As we can see in the above proposition, the effect of $N$ scales with $\calO(1/N)$ and vanishes as $N \rightarrow \infty$. Therefore, the usage of the Hutchinson estimator does not affect the convergence speed of constraint functions under the single-equality scenario with sufficiently regular $h$ and $f$.

\begin{remark}[Geometric control of $\lambda_{\text{LSI}}$] 
\label{remark:Comment on LSI}
    We state the following result from geometric analysis that provides a lower bound for $\lambda_{\text{LSI}}$:
    \begin{theorem_ap}[Informal, \cite{ledoux2006concentration, wang2020fast}] Let $\Sigma$ be a compact Riemannian manifold with diameter $D$ and non-negative Ricci curvature. Then, $\lambda_{\text{LSI}} \geq \frac{\lambda_1}{1+2D \sqrt{\lambda_1}}$, or $\lambda_{\text{LSI}} \geq \frac{\pi^2}{(1+2\pi) D^2}$ holds, where $\lambda_1$ is the first eigenvalue of the Laplace-Beltrami operator on $\Sigma$.
    \end{theorem_ap}
    This theorem implies if the constraints shrink the diameter $D$ (or increase $\lambda_1$) of $\Sigma$, the lower bound of $\lambda_{\text{LSI}}$ increases, so the on-manifold decay $\KL^\Sigma(\tilde{\rho}_t || \rho_\Sigma)$ accelerates (dominated by its exponential rate $-2\lambda_{\text{LSI}}$ appearing in \cref{thm:Convergence result for equality-constrained OLLA} or \cref{thm:Convergence result for mixed-constrained OLLA}).
\end{remark}

%% file: appendix/Algorithms.tex
\section{Algorithms of OLLA, OLLA-H, and baselines}
\label{appendix:sec:Algorithm}
\begin{algorithm}
  \caption{Euler--Maruyama discretization of OLLA \& OLLA-H}
  \label{alg:em_olla}
  \begin{algorithmic}[1]
    \State \textbf{Input:} initial point $x_0\!\in\!\mathbb{R}^d$,
           step size $\Delta t$, number of steps $K$,
           landing rate $\alpha$, boundary repulsion rate $\epsilon$, potential $f$, constraints $\{h_i\}_{i=1}^m,\{g_j\}_{j=1}^{l}$,
           Hutchinson probe numbers $N$,
           \texttt{mode} $\in\{\text{OLLA},\text{OLLA-H}\}$

    \State \textbf{Output:} sample trajectory $\{x_k\}_{k=0}^{K}$
    \Statex
    \vspace{-0.3em}\noindent\hrulefill\vspace{-1.0em}
    \Statex
    \For{$k = 0,\dots,K-1$}                                  
      \State Evaluate constraints: $\{h_i(x_k)\}_{i=1}^{m},  \{g_j(x_k)\}_{j=1}^{l}$
      \State Compute: $\nabla f(x_k)$,  $J(x_k)$, $\nabla J(x_k)$
      \State Compute: $G(x_k) \gets \nabla J(x_k)\,\nabla J(x_k)^{\mathsf T}$  
      \State \textbf{if} $\operatorname{rank}(G)<m+I_x$: \textbf{then} $G^{-1} \gets G^{\dagger}$ \textbf{else} $G^{-1} \gets G^{-1}$ 
      \State \textbf{if} \texttt{mode} = OLLA:
        \Statex\hspace{\algorithmicindent} \quad Compute: $\texttt{Tr} \gets [\trace{\Pi(x_k) \nabla^2 h_1},..., \trace{\Pi(x_k) \nabla^2 g_{i_{\abs{I_x}}}}]^T$
      \State \textbf{else}
        \Statex\hspace{\algorithmicindent} \quad Compute: $\texttt{Tr} \gets \sum_{k=1}^N [v_k^T \Pi(x_k) \nabla^2 h_1 v_k,..., v_k^T \Pi(x_k) \nabla^2 g_{i_{\abs{I_x}}} v_k]^T$, \ $v_k \sim \calN(0,I_d)$
      \State Compute: $\calH(x_k) \gets \nabla J(x_k)^T G^{-1}(x_k) \texttt{Tr}$  
      \State Compute: $q(x_k) \gets - \Pi(x_k) \nabla f(x_k) -\alpha \nabla J(x_k)^T G^{-1}(x_k) J(x_k) + \calH(x_k)$
      \State Update: $x_{k+1}\gets x_k + q(x_k)\,\Delta t + \sqrt{2\Delta t} \Pi(x_k) \xi_k$ where $\xi_k\sim\mathcal N(0,I_d)$
    \EndFor
    \State \Return $\{x_k\}_{k=0}^{K}$
  \end{algorithmic}
\end{algorithm}

\begin{remark}[Pseudo-inverse of Gram matrix when LICQ fails] \ 
    A second issue of OLLA arises if the Gram matrix $G = \nabla J \nabla J^T$ becomes singular. This may happen when the LICQ condition momentarily fails near the neighborhood of $\Sigma$. In that case, we replace the inverse $G^{-1}$ with the Moore-Penrose pseudo-inverse $G^{\dagger}$. Because $G^\dagger$ still projects onto the row space of $\nabla J$ and annihilates its null space, it enforces the same exact orthogonality to constraint gradients, preserving the exponential landing behavior and numerical stability of OLLA. 
\end{remark}

\begin{algorithm}
  \caption{Constrained Langevin (CLangevin) with slack variables \citep{rousset2010free, lelievre2012langevin}} 
  \label{alg:clangevin}
  \begin{algorithmic}[1]
    \State \textbf{Input:} initial position $x_0 \in \Sigma$,
      step size $\Delta t$, number of steps $K$, potential $f$,
    constraints $\{h_i\}_{i=1}^{m}$, $\{g_j\}_{j=1}^{l}$,
      projection iterations $L$, tolerance $\tau$, regularization $\lambda$
    \State \textbf{Output:} sample trajectory $\{x_k\}_{k=0}^{K}$ 
    \Statex  \vspace{-0.3em}\noindent\hrulefill\vspace{-0.8em}\Statex
    \State Initialize slack variables: 
           $s_{0,j} \gets \sqrt{\max\{-2 g_j(x_0),0\}}$ for $j \in [l]$
    \State Set the extended state:  $s_0 \gets (s_{0,1},...,s_{0,l})\in \bbR^l$, $y_0 \gets (x_0,s_{0,1},...,s_{0,l})\in \bbR^{d+l}$

    \For{$k=0,\dots,K-1$}
      \State Draw $\xi_k\sim\mathcal N(0,I_{d+l})$
      \State Compute the augmented constraint vector:
      \begin{equation*}
            J(y_k)=\bigl[h_1(x_k),\dots,h_m(x_k), g_1(x_k)+\frac{1}{2}s_{k,1}^{2},\dots, g_l(x_k)+\frac{1}{2} s_{k,l}^{2}]^T
      \end{equation*}
      \State Constrained update: 
      \begin{align*}
           y_{k+1}\gets y_k - \nabla f_{\mathrm{ext}}(y_k) \Delta t  + \sqrt{2 \Delta t} \xi_k + \nabla J(y_k)^T \lambda_k
      \end{align*}
      \Statex \hspace{\algorithmicindent} where $f_{\mathrm{ext}}(x,s) = f(x)$ and $\lambda_k$ is chosen such that $\norm{J(y_{k+1})}_\infty \leq \tau$ by Newton's 
      \Statex \hspace{\algorithmicindent} method with regularization $\lambda$ and max iterations $L$
      \State Set : $x_k \gets y_k[1:d]$, $s_k \gets y_k[d+1:l]$
    \EndFor
    \State \textbf{Output:} $\{x_k\}_{k=0}^{K}$
  \end{algorithmic}
\end{algorithm}

\begin{algorithm}
  \caption{Constrained HMC (CHMC) with slack variables \citep{rousset2010free, lelievre2012langevin}} 
  \label{alg:chmc}
  \begin{algorithmic}[1]
    \State \textbf{Input:} 
      initial position $x_{0}\in \Sigma$,
      step size $\Delta t$, 
      number of steps $K$, 
      potential $f$,
      constraints $\{h_{i}\}_{i=1}^{m}$,$\{g_{j}\}_{j=1}^{l}$,
      projection iterations $L$, tolerance $\tau$, regularization $\lambda$, friction $\gamma$
    \State \textbf{Output:} sample trajectory $\{x_k\}_{k=0}^{K}$
    \Statex\hrulefill
    \State Initialize slack variable: 
           $s_{0,j} \gets \sqrt{\max\{-2 g_j(x_0),0\}}$ for $j \in [l]$
    \State Set the extended state:  $s_0 \gets (s_{0,1},...,s_{0,l}) \in \bbR^l$, $y_0 \gets (x_0,s_{0,1},...,s_{0,l})\in \bbR^{d+l}$
    \State Sample momentum $p_{0}\sim\calN(0,I_{d+l})$ such that $\nabla J(y_0)p_0 = 0$
    \For{$k=0,\dots,K-1$}
      \State Draw $\xi_{k}, \xi_{k+1/2} \sim \calN (0,I_{d+l})$
      \State Compute the augmented constraint vector:
      \begin{equation*}
            J(y_k)=\bigl[h_1(x_k),\dots,h_m(x_k), g_1(x_k)+\frac{1}{2}s_{k,1}^{2},\dots, g_l(x_k)+\frac{1}{2} s_{k,l}^{2}]^T
      \end{equation*}
      \State \textbf{Midpoint Euler step}:
      \begin{equation*}
          p_{k+1/4} = p_k - \frac{\Delta t}{4} \gamma (p_k + p_{k+1/4}) + \sqrt{\Delta t \gamma } \xi_k + \nabla J(y_k)^T \lambda_{k+1/4}
      \end{equation*}
      \Statex\hspace{\algorithmicindent} such that \ $\nabla J(y_{k}) p_{k+1/4}=0$
      \State \textbf{Verlet step - (1)}:
      \begin{align*}
          &p_{k+1/2} = p_{k+1/4} - \frac{\Delta t}{2} \nabla f_{\rm ext}(y_k) + \nabla J(y_k)^T \lambda_{k+1/2}\\
          &y_{k+1} = y_k +  p_{k+1/2} \Delta t
      \end{align*}
      \Statex\hspace{\algorithmicindent}such that \ $\norm{J(y_{k+1})}_\infty \leq \tau $ by Newton's method with regularization $\lambda$ and max iterations $L$
      \State \textbf{Verlet step - (2)}:
      \begin{equation*}
          p_{k+3/4} = p_{k+1/2} - \frac{\Delta t}{2}\nabla f_{\rm ext}(y_{k+1}) + \nabla J(y_{k+1}) \lambda_{k+3/4} 
      \end{equation*}
      \Statex\hspace{\algorithmicindent} such that $\nabla J(y_{k+1})^T p_{k+3/4} = 0$

      \State \textbf{Midpoint Euler step:}
      \begin{equation*}
          p_{k+1} =p_{k+3/4} - \frac{\Delta t}{4} \gamma (p_{k+3/4} + p_{k+1}) + \sqrt{\Delta t \gamma }\xi_{k+1} + \nabla J(y_{k+1}) \lambda_{k+1}
      \end{equation*}
      \Statex\hspace{\algorithmicindent}such that $\nabla J(y_{k+1})^T p_{k+1} = 0$
      \State Set : $x_k \gets y_k[1:d]$, $s_k \gets y_k[d+1:l]$
      
    \EndFor
    \State \textbf{Output:} $\{x_{k}\}_{k=0}^{K}$ 
  \end{algorithmic}
\end{algorithm}

\begin{algorithm}
  \caption{Constrained generalized HMC (CGHMC) with MH correction \citep{rousset2010free, lelievre2019hybrid}}
  \label{alg:cghmc}
  \begin{algorithmic}[1]
    \State \textbf{Input:} 
      initial position $x_{0}\in \Sigma$,
      step size $\Delta t$, 
      number of steps $K$, 
      potential $f$,
      equality constraints $\{h_{i}\}_{i=1}^{m}$, $\{g_{j}\}_{j=1}^{l}$,
      projection iterations $L$, tolerance $\tau$, 
      regularization $\lambda$, friction $\gamma$
    \State \textbf{Output:} sample trajectory $\{x_{k}\}_{k=0}^{K}$
    \Statex\hrulefill
    \State Sample momentum $p_{0}\sim\calN(0,I_{d})$ such that $\nabla J(x_0)p_0 = 0$ 
    \For{$k=0,\dots,K-1$}
      \State Draw $\xi_{k}, \xi_{k+1/2} \sim \calN (0,I_{d})$
      \State Compute the augmented constraint vector: $J(x_k)=\lbra{h_1(x_k),\dots,h_m(x_k)}$
      \State \textbf{Midpoint Euler step}:
      \begin{equation*}
          p_{k+1/4} = p_k - \frac{\Delta t}{4} \gamma (p_k + p_{k+1/4}) + \sqrt{\Delta t \gamma } \xi_k + \nabla J(x_k)^T \lambda_{k+1/4}
      \end{equation*}
      \Statex\hspace{\algorithmicindent} such that \ $\nabla J(x_{k}) p_{k+1/4}=0$
      \State Compute the Hamiltonian: $H(x_k, p_{k+1/4}) = f(x_k) + \frac{1}{2}\norm{p_{k+1/4}}_2^2$
      \State \textbf{Verlet step - (1)}:
      \begin{align*}
          &p_{k+1/2} = p_{k+1/4} - \frac{\Delta t}{2} \nabla f(x_k) + \nabla J(x_k)^T \lambda_{k+1/2}\\
          &\tilde{x}_{k+1} = x_k + p_{k+1/2} \Delta t
      \end{align*}
      \Statex\hspace{\algorithmicindent -1pt}such that \ $\norm{J(\tilde{x}_{k+1})}_\infty \leq \tau $ by Newton's method with regularization $\lambda$ and max iterations $L$
      \State \textbf{Verlet step - (2)}:
      \begin{equation*}
          \tilde{p}_{k+3/4} = p_{k+1/2} - \frac{\Delta t}{2}\nabla f(\tilde{x}_{k+1}) + \nabla J(\tilde{x}_{k+1}) \lambda_{k+3/4} 
      \end{equation*}
      \Statex\hspace{\algorithmicindent} such that $\nabla J(\tilde{x}_{k+1})^T \tilde{p}_{k+3/4} = 0$
      \State Compute the Hamiltonian: $H(\tilde{x}_{k+1}, \tilde{p}_{k+3/4}) = f(\tilde{x}_{k+1}) + \frac{1}{2}\norm{\tilde{p}_{k+3/4}}_2^2$
      \State \textbf{Metropolis-Hasting Correction:} With probability
      \begin{equation*}
         \min \mbra{\exp \sbra{- (H(\tilde{x}_{k+1}, \tilde{p}_{k+3/4})- H(x_k, p_{k+1/4}))} , 1}
      \end{equation*}
      \Statex\hspace{\algorithmicindent} set 
      \begin{equation*}
            (x_{k+1},p_{k+\frac34})=(\tilde{x}_{k+1}, \tilde{p}_{k+\frac34}), \quad \text{if $g(\tilde{x}_{k+1}) \leq 0$}
      \end{equation*} 
      \Statex\hspace{\algorithmicindent} Otherwise, reject and flip momentum  $(x_{k+1},p_{k+\frac34})=(x_{k},-\,p_{k+\frac14})$
      \State \textbf{Midpoint Euler step:}
      \begin{equation*}
          p_{k+1} =p_{k+3/4} - \frac{\Delta t}{4} \gamma (p_{k+3/4} + p_{k+1}) + \sqrt{\Delta t \gamma }\xi_{k+1} + \nabla J(x_{k+1}) \lambda_{k+1}
      \end{equation*}
      \Statex\hspace{\algorithmicindent}such that $\nabla J(x_{k+1})^T p_{k+1} = 0$
    \EndFor
    \State \textbf{Output:} $\{x_{k}\}_{k=0}^{K}$ 
  \end{algorithmic}
\end{algorithm}

%% file: appendix/Proof_of_differential_operator_properties_on_Sigma.tex
\section{Proof of Basic Properties on $\Sigma$}
\label{appendix:sec:Proof of differential operator properties on Sigma}
\subsection{Intrinsic Gradient on $\Sigma$} Recall that our target manifold $\Sigma$ is defined by $\Sigma := \mbra{x \in \bbR^d \mid h(x) = 0 ,g(x) \leq 0 }$.
For a smooth function $f :\Sigma \rightarrow \bbR$, the Riemannian gradient $\nabla_\Sigma f$ is defined by the relation
\begin{equation*}
    \inner{\nabla_\Sigma f(x), v}  = df_x(v), \text{\quad for every } v \in T_x\Sigma
\end{equation*}
To check $\Pi(x) \nabla f(x) = \nabla_\Sigma f(x)$, observe that $\Pi(x) \nabla f(x) \in T_x \Sigma$ and $\Pi(x) \nabla f(x)$ reproduces $df_x$ on every tangent vector. This is because if $v \in T_x \Sigma$, $\inner{\Pi(x) \nabla f(x), v} = \inner{\nabla f(x), \Pi(x) v} = \inner{\nabla f(x), v} = df_x(v)$ due to the symmetric nature of $\Pi(x)$. Therefore, the intrinsic gradient on $\Sigma$ can be extrinsically defined by $\nabla_\Sigma f(x) = \Pi(x) \nabla f(x)$.

\subsection{Intrinsic Divergence on $\Sigma$}
Recall that if $X \in \calX(\Sigma)$ where $\calX(\Sigma)$ is a set of smooth vector fields on $\Sigma$, then the divergence of $X$ on $\Sigma$ is defined by $\divergence_\Sigma X(x) = \sum_{i=1}^{d-(m+\abs{I_x})} \inner{\nabla^\Sigma_{E_i}  X, E_i}$ where $\nabla^\Sigma$ is Levi-Civita connection on $\Sigma$ and $\mbra{E_1,...E_{d-(m+\abs{I_x})}}$ is an orthonormal frame of $T_x \Sigma$. To check its extrinsic formula, first observe that the induced Levi-Civita connection is given by $\nabla_Y^\Sigma X(x) := (\nabla_Y X)^\top = \Pi(x) \nabla_Y X(x)$ where $X, Y \in \calX(\Sigma)$, $\top$ indicates tangential component on $\Sigma$, and $\nabla$ is the Levi-Civita connection (or usual directional gradient) in $\bbR^d$.
Then, if $X$ and $\mbra{E_1, ... E_{d-(m+I_x)}}$ are extended to the ambient space $\bbR^d$, it holds that 
\begin{align*}
    \divergence_\Sigma X = \sum_{i=1}^{d-(m+\abs{I_x})} \inner{\Pi \nabla_{E_i} X, E_i} = \sum_{i=1}^{d-(m+\abs{I_x})} \inner{\Pi\nabla X E_i, E_i} = \trace{\Pi \nabla X}
\end{align*}
where the last equation is obtained by taking a basis $\mbra{E_1, ...E_{d-(m+\abs{I_x})}, \nu_1,..., \nu_{m+\abs{I_x}}}$ of $\bbR^d$ with $\mbra{\nu_1,...,\nu_{m+\abs{I_x}}}$ being the orthonormal basis of $(T_x\Sigma)^\perp$, and applying the definition of trace. 

\subsection{Recoverable Tubular Neighborhood of Boundaryless Riemannian Manifold}
Let $U_\epsilon (\Sigma) := \mbra{x \in \bbR^d \mid \text{dist}(x, \Sigma) < \epsilon}$ be a tubular neighborhood of $\Sigma$ with reach $\epsilon$, whose existence is guaranteed by the compactness of $\Sigma$ \citep{lee2018introduction}. In the proof of \cref{thm:Convergence result for equality-constrained OLLA}, we require a property that enables us to recover the unique $x\in U_\epsilon (\Sigma)$ such that $y = \pi(x)$ and $h(x) = p$, given the information of $y\in \Sigma$ and $p\in \bbR^m$. 

The following theorem indicates the existence of such a nice neighborhood of $\Sigma$ such that any $x$ in this neighborhood can be recovered from the information of $y, p$. In addition to this, we need a regularity lemma to connect the decrease of $\norm{h(x)}_2$ with the decrease of $\text{dist}(x, \Sigma)$, which is a crucial property to show the convergence of the $W_2$ distance.

\begin{lemma_ap}[Regularity lemma] \ 
\label{lem:regularity lemma}
Assume $h \in C^2$ and LICQ condition is satisfied on $\Sigma$. Then, there exist constants $\hat{\epsilon}, \kappa > 0 $ such that for all $x \in U_{\hat{\epsilon}} (\Sigma) \subset U_\epsilon (\Sigma)$
\begin{equation*}
    \norm{h(x)}_2 \geq \kappa \norm{x-\pi(x)}_2
\end{equation*}
where $\kappa = \frac{1}{2} \min_{y\in \Sigma} \sigma_{min} (\nabla h(y)) > 0$.
\end{lemma_ap}
\begin{proof}
    For each $y \in \Sigma$, $\nabla h(y)$ has full rank by the LICQ condition, so its smallest singular value $\sigma_y := \sigma_{\text{min}}(\nabla h(y)) > 0$. Since, any $x \in U_\epsilon(\Sigma)$ can be decomposed into $x = y + v$ for some $y \in \Sigma, v \in N_y(\Sigma):= (T_y\Sigma)^\perp$, Taylor's theorem gives
    \begin{equation*}
        h(x) = h(y+v) = \nabla h(y)v + R_y(v)
    \end{equation*}
    where the norm of the remainder term $R_y(v)$ is bounded above by $\norm{R_y(v)}_2 \leq \frac{1}{2}M \norm{v}_2^2$ for $M := \sup_{y \in U_{\hat{\epsilon}} (\Sigma)} \norm{\nabla^2 h(y)}_2 < \infty$. Hence,
    \begin{equation*}
        \norm{h(x)}_2 = \norm{h(y+v)}_2 \geq \norm{\nabla h(y)v}_2 - \norm{R_y(v)}_2 \geq \sigma_y \norm{v}_2 -\frac{1}{2} M \norm{v}_2^2.
    \end{equation*}
    Furthermore, because $h \in C^2$ and $y \mapsto \sigma_{min} (\nabla h(y))$ is a continuous function on $\Sigma$, the compactness of $\Sigma$ implies there exists $\sigma = \min_{y \in \Sigma} \sigma_{min} (\nabla h(y)) > 0$.  Also, choose $0 < \hat{\epsilon} \leq \epsilon$ such that $M\hat{\epsilon} \leq \sigma$. Then whenever $x \in U_{\hat{\epsilon}}(\Sigma)$ so that $\norm{v}_2 < \hat{\epsilon}$, we get 
    \begin{equation*}
        \norm{h(x)}_2 \geq \sigma \norm{v}_2 - \frac{1}{2}M \norm{v}_2^2 \geq (\sigma - \frac{1}{2} M\hat{\epsilon}) \norm{v}_2 \geq \frac{\sigma}{2} \norm{v}_2 = \frac{\sigma}{2} \norm{x-\pi(x)}_2. 
    \end{equation*}
\end{proof}

\cref{lem:regularity lemma} shows that an upper bound on $\norm{h(x)}_2$ yields an upper bound on $\norm{x-\pi(x)}_2$; however, this guarantee holds only once $x$ has entered $U_{\hat{\epsilon}}$. Therefore, to ensure that $x$ indeed enters $U_{\hat{\epsilon}}$ whenever $\norm{h(x)}_2$ is sufficiently small, we appeal to \cref{lem:entranec cutoff}.

\begin{lemma_ap}[Entrance cutoff of $U_{\hat{\epsilon}} (\Sigma)$ in terms of $\norm{h(x)}_2$] \    
\label{lem:entranec cutoff}
Assume $h$ is continuous and coercive, i.e. $\norm{h(x)}_2 \rightarrow \infty$ whenever $\norm{x}_2 \rightarrow \infty$. Then, there exists $\delta^* > 0$ such that $\dist(x, \Sigma) < \hat{\epsilon}$ if $\norm{h(x)}_2 < \delta^*$
where $\hat{\epsilon} > 0 $ is the constant defined in \cref{lem:regularity lemma}.  
\end{lemma_ap}

\begin{proof}
    First, we show that $h$ is a proper map. Let $C$ be a compact set in $\bbR^{m}$ so that it is closed and bounded. Then, $h^{-1}(C)$ is closed because $h$ is continuous. Also, suppose $h^{-1}(C)$ were unbounded so that there is a sequence $\mbra{x_k}_{k\in \bbN} \in h^{-1}(C)$ with $\norm{x_k}_2 \rightarrow \infty$. Then, by coercivity of $h$, $\norm{h(x_k)}_2 \rightarrow \infty$ while every $h(x_k)$ lies in $C$, which is bounded and leads to a contradiction. Therefore, $h^{-1}(C)$ is bounded and closed and therefore is compact by the Heine-Borel theorem. This proves $h$ is a proper map.  
    
    Now, define a set $S_{\hat{\epsilon}} := \mbra{x \in \bbR^d  \mid \dist(x, \Sigma) \geq {\hat{\epsilon}}}$ and assume $\inf_{x \in S_{\hat{\epsilon}}} \norm{h(x)}_2 = 0$. In this case, there must be a sequence $\mbra{x_k}_{k \in \bbN} \subset S_{\hat{\epsilon}}$ with $\norm{h(x_k)}_2 \rightarrow 0$. This implies $x_k \in h^{-1}(\bar{B}(0,1))$ for $\forall k \geq K$ for some $K \in \bbN$. Since $h^{-1}(\bar{B}(0,1))$ is a compact set, there is a subsequence $\mbra{x_{k_j}}$ of $\mbra{x_k}$ converging to some $x_*$ from the Bolzano-Weierstrass theorem. Subsequently, the continuity of $h$ implies $h(x_*) = \lim_{j \rightarrow \infty} h(x_{k_j}) = 0$ and $x_* \in \Sigma$. However, because $S_{\hat{\epsilon}}$ is closed and every $x_{k_j} \in S_{\hat{\epsilon}}$, it should satisfy $x_* \in S_{\hat{\epsilon}} \Rightarrow \dist(x_*, \Sigma) \geq {\hat{\epsilon}}$, which is a contradiction that $x_* \in \Sigma \Leftrightarrow \dist(x,\Sigma) = 0$. Therefore, $\delta^* := \inf_{x \in S_{\hat{\epsilon}}} \norm{h(x)}_2 >0$ holds and $\dist(x, \Sigma) < {\hat{\epsilon}}$ if $\norm{h(x)}_2 < \delta^*$.
\end{proof}

\begin{theorem_ap}[Recoverable tubular neighborhood] \ 
\label{thm:recoverable tubular neighborhood}
Let $\Sigma:= \mbra{x \in \bbR^d \mid h(x) =0}$ be a compact and boundaryless Riemannian manifold with LICQ condition. Then, there exists a\textbf{ recoverable tubular neighborhood} of $\Sigma$ with width $\delta$,  $\hat{U}_\delta (\Sigma):= \mbra{x\in \bbR^d \mid \norm{h(x)}_2 < \delta} \subset  U_{\hat{\epsilon}}(\Sigma)$ \, for some $\delta >0$ such that
\begin{enumerate}
    \item The nearest-point projection map $\pi: \hat{U}_\delta (\Sigma) \rightarrow \Sigma, \pi(x) = \argmin_{y \in \Sigma} \norm{x-y}_2 $ is well-defined. 
    \item The following recovery map $\zeta(y,p)$ is well-defined 
    \begin{equation}
        \zeta(y,p): \Sigma \times B(0,\delta) \rightarrow \hat{U}_\delta(\Sigma), \quad \zeta(y,p) = y + \nabla h(y)^T L(y,p)
    \end{equation}
    where $L:\Sigma \times B(0,\delta) \rightarrow \bbR^m$ is the $C^1$ function such that $h(\zeta(y,p)) = p$ and  $\pi(\zeta(y,p)) = y$ \, for\,  $\forall (y,p) \in \Sigma \times B(0,\delta)$.
\end{enumerate}
In this case, we refer to $\hat{U}_\delta(\Sigma)$ as the \textbf{recoverable tubular neighborhood} of $\Sigma$.
\end{theorem_ap}
\begin{proof}
    Define $F((y,p), L):\bbR^{(d+m) + m} \rightarrow \bbR^m$ by $F((y,p), L) = h(y + \nabla h(y)^T L ) - p$. Then, for each $y\in \Sigma$, $F((y,0),0) = h(y) - 0 = 0$. Also, $\nabla_L F((y,p), L) = \nabla h(y+\nabla h(y)^T L) \nabla h(y)^T $ implies $\nabla_L F((y,0), 0) = \nabla h(y) \nabla h(y)^T$, which is invertible due to full rank assumption of $\nabla h(y)$. 
    
    Therefore, by the implicit function theorem, there exists $\delta_y >0$ and an open set $U_y := \mbra{ (y', p') \in \Sigma \times \bbR^m \mid d_\Sigma(y, y') < \delta_y, \norm{p'}_2 < \delta_y} \subset \bbR^{d+m}$ such that there exists a unique $C^1$ function $L: U_y \rightarrow \bbR^m$ satisfying $L(y,0) = 0$ and $F((y,p), L(y,p)) = 0$ for $\forall (y,p) \in U_y$. 
    
    Now, observe that $\cup_{y \in \Sigma} U_y$ is the open cover of $\Sigma \times {\mbra{0}}$, which is a compact set. Therefore, using its finite subcovers, we can pick $\hat{\delta} > 0$ such that for $\forall y \in \Sigma$ and $\norm{p}_2 < \hat{\delta}$, $L(y,p)$ is well-defined $C^1$ function on $\Sigma \times B(0,\hat{\delta})$ and the recovery map $\zeta(y,p) : \Sigma \times B(0,\hat{\delta}) \rightarrow \bbR^d$, $\zeta(y,p) = y+ \nabla h(y)^T L(y,p)$ is also well-defined and $C^1$ on $\Sigma \times B(0,\hat{\delta})$ such that $h(\zeta(y,p)) = p$ holds (from $F((y,p), L(y,p))= 0$).
    
    Finally, setting $\delta = \min \mbra{\hat{\delta}, \delta^*}$ (where the constant $\delta^*$ comes from \cref{lem:entranec cutoff}) and $\hat{U}_\delta (\Sigma):= \mbra{x\in \bbR^d \mid \norm{h(x)}_2 < \delta} \subset U_{\hat{\epsilon}}(\Sigma)$ enables the well-defined nearest-point projection map $\pi$ on $\hat{U}_\delta(\Sigma)$. By applying $\pi$ to $\zeta(y,p)$, we recover the formula $\pi(\zeta(y,p)) = \pi(y + \nabla h(y)^T L(y,p)) = y, \ \forall (y,p) \in \Sigma \times B(0,\delta)$.
\end{proof}

\subsection{Recoverable Tubular Neighborhood of Riemannian Manifold with Boundary}

The following results generalize the proceeding lemmas and theorems by incorporating inequality constraints alongside the equality constraints.

\begin{lemma_ap}[Regularity lemma with boundary] \ 
\label{lem:regularity lemma with boundary}
Let $\Sigma = \mbra{ x \in \bbR^d \mid h(x) =0 , g(x) \leq 0}$. Assume $h,g \in C^2$ and LICQ condition is satisfied on $\Sigma$. Then, there exist constants $\hat{\epsilon}, \kappa > 0 $ such that for all $x \in U_{\hat{\epsilon}} (\Sigma) \subset U_\epsilon(\Sigma)$
\begin{equation*}
    \norm{h(x)}_2 + \norm{g_{I_{\pi(x)}} (x)}_2 \geq \kappa \norm{x-\pi(x)}_2
\end{equation*}
where $\kappa = \frac{1}{2} \min_{y\in \Sigma} \sigma_{\min} \sbra{\lbra{\nabla h(y)^T, \nabla g_{I_y}(y)^T}} > 0$.
\end{lemma_ap}
\begin{proof}
    For every subset $I \subset [l]$, consider $\Sigma_I := \mbra{y \in \Sigma \mid I_y =I}$. Then for $y \in \Sigma_I$, $\nabla J_I(y) :=[\nabla h(y)^T, \nabla g_{I}(y)^T]^T \in \bbR^{(m+\abs{I}) \times d}$ has full rank due to the LICQ condition, so its smallest singular value $\sigma^I_y := \sigma_{\text{min}}(\nabla J_I(y))> 0$. Because $y \mapsto \sigma^I_y$ is continuous on $\Sigma_I$, we also have $\sigma := \min\limits_{I \subset [l]} \inf\limits_{y\in \Sigma_I} \sigma^I_y > 0$ by \cref{lem: inf sigma_y_I >0}. Also, since any $x \in U_\epsilon(\Sigma)$ can be decomposed into $x = y + v$ for some $y \in \Sigma, v \in N_y(\Sigma)$, Taylor's theorem gives
    \begin{equation*}
        h(x) = h(y+v) = \nabla h(y)v + R_h(v), \quad g_{I_y}(x) = g_{I_y}(y+v) = \nabla g_{I_y}(y)v+ R_g(v)
    \end{equation*}
    where the norm of the remainder term $R_h(v), R_g(v)$ is bounded above by $\norm{R_h(v)}_2 \leq \frac{1}{2}M \norm{v}_2^2, \norm{R_g(v)}_2 \leq \frac{1}{2} M \norm{v}_2^2$ for $M := \max\limits_{z\in U_\epsilon (\Sigma)} \mbra{ \norm{\nabla^2 h(z)}_2, \norm {\nabla^2 g(z)}_2} <\infty$. Now, we set $w = \nabla J(y) v $, which satisfies $\norm{w}_2 \geq \sigma \norm{v}_2$ by the definition of $\sigma$. Then, we observe
    \begin{align*}
        \norm{h(x)}_2 + \norm{g_{I_y}(x)}_2 = \norm{w}_2 - \sbra{ \norm{w}_2 - \norm{h(x)}_2  - \norm{g_{I_y}(x)}_2} \overset{(\circ)}{\geq} \sigma \norm{v}_2 - M\norm{v}_2^2
    \end{align*}
    where $(\circ)$ comes from the following observation:
    \begin{align*}
        \norm{w}_2 - (\norm{h(x)}_2 + \norm{g_{I_y}(x)}_2) &\leq \norm{w}_2 - \norm{[h(x)^T, g_{I_y}(x)^T]^T}_2 \leq \norm{w - [h(x)^T, g_{I_y}(x)^T]^T}_2 \\
        &= \norm{[R_h(v)^T, R_g(v)^T]^T}_2 \leq \norm{R_h(v)}_2 + \norm{R_g(v)}_2 \\
        &\leq M\norm{v}_2^2.
    \end{align*}
    Now, choose small $0 < \hat{\epsilon} \leq \epsilon$ such that $M\hat{\epsilon} \leq \frac{\sigma}{2}$. Then whenever $x \in U_{\hat{\epsilon}} (\Sigma)$ so that $\norm{v}_2 < \hat{\epsilon}$, we get 
    \begin{equation*}
        \norm{h(x)}_2 + \norm{g_{I_{\pi(x)}}(x)}_2 \geq \frac{\sigma}{2} \norm{v}_2 = \frac{\sigma}{2}\norm{x-\pi(x)}_2.
    \end{equation*}
\end{proof}

\begin{lemma_ap} \ 
    \label{lem: inf sigma_y_I >0}
    For every subset $I \subset [l]$, consider $\Sigma_I := \mbra{y \in \Sigma \mid I_y =I}$. For $y \in \Sigma_I$, define  $\nabla J_I(y):=[\nabla h(y)^T, \nabla g_{I}(y)^T]^T \in \bbR^{(m+\abs{I}) \times d}$ and $\sigma^I_y := \sigma_{min}(\nabla J_I(y)) > 0$. Then, $\inf\limits_{y\in \Sigma_I} \sigma^I_y > 0$ holds.
\end{lemma_ap}
\begin{proof}
    Suppose $\inf\limits_{y \in \Sigma_I} \sigma_y^I = 0$, then there exists a sequence $y_k \in \Sigma_I $ with $\sigma_{min}(\nabla J_I (y_k)) \rightarrow 0$. Because $\Sigma$ is compact, the sequence has a subsequence of $y_k$ converging to $y_* \in\Sigma$ by Bolzano–Weierstrass theorem. In this case, there are two possibilities:
    \begin{enumerate}
        \item When $I_{y_*} = I$,$\quad $ $\sigma_{min}(J_I(y_*)) >0$ by LICQ, contradicting the assumption $\inf\limits_{y \in \Sigma_I} \sigma_y^I = 0$.
        \item When $I_{y_*} = I \cup K$ for some non-empty set $K \subset [l]$, the LICQ condition implies $\quad $ $\nabla J_{I \cup K} (y_*) = [\nabla h(y_*)^T, \nabla g_I(y_*)^T, \nabla g_K (y_*)^T]^T$ has a full rank, thereby, imposing $\nabla J_I(y_*)$ to have full rank. This again implies $\inf\limits_{y \in \Sigma_I} \sigma_y^I > 0$, a contradiction.
    \end{enumerate}
    Note that $I_{y_*}$ does not deactivate already activated inequality index $i \in I$. This is because $g_i(y_k) = 0$ for $\forall i \in I$ and continuity of $g_i$ implies $g_i(y_*) = 0$. Thus, the previous argument on two possibilities completes the proof.  
\end{proof}

\begin{lemma_ap}[Entrance cutoff of $U_{\hat{\epsilon}}(\Sigma)$ in terms of $\norm{h(x)}_2$ and $g(x)$] \ 
\label{lem:entranec cutoff_boundary}
Let $\Sigma := \mbra{x \in \bbR^d \mid h(x) =0 , g(x) \leq 0}$. Assume $h,g$ are continuous and $h$ is coercive, i.e. $\norm{h(x)}_2 \rightarrow$ whenever $\norm{x}_2 \rightarrow \infty$. Then, there exists $\delta^* > 0$ such that $\dist(x, \Sigma) < {\hat{\epsilon}}$ if $\norm{h(x)}_2 < \delta^*$ and $ g_i(x) < \delta^*$ for $i \in [l]$,
where ${\hat{\epsilon}} > 0$ is defined in \cref{lem:regularity lemma with boundary}.
\end{lemma_ap}
\begin{proof}
    Define $S_{\hat{\epsilon}} := \mbra{x \in \bbR^d  \mid \dist(x, \Sigma) \geq {\hat{\epsilon}}}$ and let $\psi(x) := \max \mbra{ \norm{h(x)}_2, g_1(x), ..., g_l(x)}$. Assume $\inf_{x \in S_{\hat{\epsilon}}} \psi(x)= 0$. In this case, there must be a sequence $\mbra{x_k}_{k \in \bbN} \subset S_{\hat{\epsilon}}$ with $\norm{h(x_k)}_2 \rightarrow 0$ and $g_i(x_k) \leq \psi(x_k) \rightarrow 0$ for $\forall i \in [l]$.
    
    This implies $x_k \in h^{-1}(\bar{B}(0,1)), \ \forall k \geq K$ and for some $K \in \bbN$. Since $h^{-1}(\bar{B}(0,1))$ is a compact set due to the properness of $h$ (\cref{lem:entranec cutoff}), there exists a subsequence $\mbra{x_{k_j}}$of $\mbra{x_k}$ converging to some $x_*$ from the Bolzano-Weierstrass theorem. Subsequently, the continuity of $h$ implies $h(x_*) = \lim_{j \rightarrow \infty} h(x_{k_j}) = 0, g_i(x) \leq 0, \forall i \in [l]$, therefore, $x_* \in \Sigma$.
    
    However, because $S_{\hat{\epsilon}}$ is closed and every $x_k \in S_{\hat{\epsilon}}$, it should satisfy $x_* \in S_{\hat{\epsilon}} \Rightarrow \dist(x_*, \Sigma) \geq {\hat{\epsilon}}$, which is a contradiction that $x_* \in \Sigma \Leftrightarrow \dist(x,\Sigma) = 0$. Therefore, $\delta^* := \inf_{x \in S_{\hat{\epsilon}}} \psi(x) >0$ holds and $\dist(x, \Sigma) < {\hat{\epsilon}}$ if $\psi(x) < \delta^*$.
\end{proof}

\begin{theorem_ap}[Recoverable tubular neighborhood with boundary] \ 
\label{thm:recoverable tubular neighborhood with boundary}
Let $\Sigma:= \mbra{x \in \bbR^d \mid h(x) =0, g(x) \leq 0}$ be a compact Riemannian manifold with boundary. Assume LICQ condition on $\Sigma$. Then, there exists a\textbf{ recoverable tubular neighborhood} of $\Sigma$ with width $\delta$, $\hat{U}_\delta (\Sigma):= \mbra{x\in \bbR^d \mid \norm{h(x)}_2 < \delta, g(x) < \delta} \subset  U_{\hat{\epsilon}}(\Sigma)$ \, for some $\delta >0$ such that
\begin{enumerate}
    \item The nearest-point projection map $\pi: \hat{U}_\delta (\Sigma) \rightarrow \Sigma, \pi(x) = \argmin_{y \in \Sigma} \norm{x-y}_2 $ is well-defined. 
    \item The following recovery map $\zeta(y,p,q_{I_y})$ is well-defined 
    \begin{equation*}
        \zeta(y,p,q_{I_y}): \Sigma \times B_m(0,\delta) \times B_{\abs{I_y}}(0,\delta) \rightarrow \hat{U}_\delta(\Sigma), \  \zeta(y,p, q_{I_y}) = y + \nabla J(y)^T L(y,p, q_{I_y})
    \end{equation*}
    where $I_y$ is the index set of active inequalities at $y \in \Sigma$, $J(y) := [h(y) ,g_{I_y}(y)]\in \bbR^{m + \abs{I_y}}$, and  $L:\Sigma \times B_m(0,\delta) \times B_{\abs{I_y}}(0, \delta) \rightarrow \bbR^{m + \abs{I_y}}$ is the function such that 
    \begin{equation*}
        h(\zeta(y,p,q_{I_y})) = p,  \quad 
        \begin{cases}
            \begin{aligned}[t]
                &g_i(\zeta(y,p,q_{I_y})) = q_i, &&i\in I_y\\
                &g_i(\zeta(y,p,q_{I_y})) <0,  && i \notin I_y,
            \end{aligned}
        \end{cases}   \quad \pi(\zeta(y,p, q_{I_y})) = y
    \end{equation*}
    $\forall (y,p,q_{I_y}) \in \Sigma \times B_m(0,\delta) \times B_{\abs{I_y}}(0,\delta)$. Furthermore, when $y \in \text{int} (\Sigma)$, $L = L(y,p)$ is a $C^1$ function on $\Sigma \times B_m(0,\delta)$.
\end{enumerate}
\end{theorem_ap}
\begin{proof}
    Let us first define $F((y,p,q_{I_y}),L): \bbR^{(d+m + \abs{I_y}) + (m+\abs{I_y})} \rightarrow \bbR^{m+\abs{I_y}}$, $F((y,p,q_{I_y}),L):= \sbra{h(y+\nabla J(y)^T L)-p, g_{I_y}(y+ \nabla J(y)^T L) - q_{I_y}} \in \bbR^{m + \abs{I_y}}$ and consider the stratum $\Sigma_{I} := \mbra{y \in \Sigma \mid I_y = I}$ for each subset $I$ of inequality indices.
    
    Then for  $y \in \Sigma_I$, we observe that $F((y,0,0),0) = 0$ and $\nabla_L F((y,0,0),0) = \nabla J(y) \nabla J(y)^T$ which is invertible by the LICQ condition. Therefore, the implicit function theorem ensures the existence of $\delta_y$ > 0, $U_y:= \mbra{ (y', p', q_{I}') \in \Sigma_I \times \bbR^m \times \bbR^{I} \mid   d_\Sigma(y, y') < \delta_y, \norm{p'}_2 < \delta_y, \norm{q'_I}_2 < \delta_y}$ such that there exists a unique $C^1$ map $L: U_y \rightarrow \bbR^{m + \abs{I}}$ satisfying $F((y,p, q_{I}), l(y, p, q_{I})) = 0, \ \forall (y,p,q_{I}) \in U_y$.

    Now, observe that $\cup_{y \in \Sigma_I} U_y$ is the open cover of $\Sigma_I \times \mbra{0}_m \times \mbra{0}_{I}$, which is a compact set. Therefore, using the finite subcovers of $\Sigma_I$, we can pick $\hat{\delta}_I > 0$ such that for $\forall y \in \Sigma_I$, $\norm{p}_2 < \hat{\delta}_I$, and $\norm{g_I}_2 < \hat{\delta}_I$, $L(y,p, q_I)$ is well-defined $C^1$ function on $\Sigma_I \times B_m(0,\hat{\delta}_I) \times B_{\abs{I}}(0, \hat{\delta}_I)$ and the recovery map $\zeta(y,p, q_I) : \Sigma \times B_m(0,\hat{\delta}) \times B_{\abs{I}
    }(0,\hat{\delta}) \rightarrow \bbR^d$, $\zeta(y,p, q_I) = y+ \nabla J(y)^T L(y,p,q_I)$ is also well-defined $C^1$ map on $\Sigma_I \times B_m(0,\hat{\delta}_I) \times B_{\abs{I}} (0, \hat{\delta}_I)$ such that $h(\zeta(y,p, q_I)) = p$ and $g_i(\zeta(y,p,q_I)) = q_i$ for $i \in I$, which comes from the property $F(y,p,q_I), l(y,p,q_I))= 0$. 
    
    Furthermore, from the continuity of $g_i, i \notin I$, there exists $\gamma_I> 0$ such that $g_i(z) \leq - \gamma_I$ for all $z \in B(y, \gamma_I)$ and the compactness of $\bar{U}_{\hat{\epsilon}} (\Sigma)$ gives $G$-Lipschitzness of $g_i$ on $U_{\hat{\epsilon}}(\Sigma)$ for all $i \in [l]$, for some $G >0$. Also, we recall that $L(y,p,q_I)$ is $C^1$ map on $\mbra{(y,p,q_I) \mid y \in \Sigma_I, \norm{p}_2 < \hat{\delta}_I, \norm{g_I}_2 < \hat{\delta}_I}$ with $L(y,0,0) = 0$. 
    
    Hence, the Taylor expansion of $L(y,p, q_I)$ with boundedness of $\norm{\nabla_p L(y,p,q_I)}_2, \norm{\nabla_{q_I} L(y,p,q_I)}_2$ on $\mbra{(y,p,q_I) \mid y \in \Sigma_I, \norm{p}_2 < \hat{\delta}_I, \norm{g_I}_2 < \hat{\delta}_I}$ and boundedness of $ \norm{\nabla J(y)}_2$ on $\Sigma$ gives
    \begin{equation*}
        \norm{L(y,p,q_I)} \leq C_I (\norm{p}_2 + \norm{q_I}_2) \quad \Rightarrow \quad \norm{\nabla J(y)^T L(y,p,q_I)}_2 \leq C_I' (\norm{p}_2 + \norm{q_I}_2)
    \end{equation*}
    for some $C_I, C_I' > 0$. Then, choosing $\tilde{\delta}_I := \min \mbra{\hat{\delta}_I, \frac{\gamma_I}{4C_I' G}}$ concludes that whenever $\norm{p}_2 < \tilde{\delta}_I, \norm{q_I}_2 < \tilde{\delta}_I$, the inequality
    \begin{equation*}
        g_i(\zeta(y,p,q_I)) = g_i(y+\nabla J(y)^T L(y,p,q_I)) \leq g_i(y) + G \norm{\nabla J(y)^T L(y,p,q_I)}_2 \leq -\frac{\gamma_I}{2} < 0 
    \end{equation*}
    holds for all $i\notin I$. Finally, setting $\delta = \min \mbra{\min_{I\subset [l]} \tilde{\delta}_I, \delta^*}$ (where the constant $\delta^*$ comes from \cref{lem:entranec cutoff_boundary}) and $\hat{U}_\delta (\Sigma):= \mbra{x\in \bbR^d \mid \norm{h(x)}_2 < \delta, \norm{g(x)}_2 <\delta} \subset U_{\hat{\epsilon}}(\Sigma)$ enable well-defined nearest-point projection map $\pi$ on $\hat{U}_\delta(\Sigma)$. By applying $\pi$ to $\zeta(y,p, g_{I_y})$, we recover the formula $\pi(\zeta(y,p, q_{I_y})) = \pi(y + \nabla J(y)^T l(y,p,q_{I_y})) = y, \ \forall (y,p, q_{I_y}) \in \Sigma \times B_m(0,\delta) \times B_{\abs{I_y}}(0,\delta)$.
\end{proof}

%% file: appendix/Construction_of_SDE_decaying_constraints_exponentially_fast.tex
\section{Construction of SDE with Exponentially Fast Decaying Constraints }
\label{appendix:sec:Construction of SDE decaying constraints exponentially fast}

\begin{proposition*}[Construction of OLLA and its closed form SDE] \
    \label{prop:Construction of OLLA and its closed form SDE_app}
    Consider the following SDE:
    \begin{equation}
    \label{eqn:mixed_OLLA_premitive_form_app}
    dX_t = q(X_t)dt + Q(X_t) dW_t
    \end{equation}
    where
    \begin{align}
    Q&:= \argmin\limits_{\bar{Q}\in\bbR^{d\times d}} \norm{\sqrt{2}I-\bar{Q}}_F^2 \quad \text{s.t} \quad     \begin{cases}
        \bar{Q}  \nabla h_i = 0, \ \forall i \in [m],\\
          \bar{Q} \nabla g_j = 0, \ \forall j \in I_x.
    \end{cases} \notag \\
    q&:= \argmin\limits_{\bar{q}\in\bbR^{d}} \norm{\bar{q} + \nabla f}_2^2 \quad \text{s.t} \quad 
    \begin{cases}
        \nabla h_i^T \bar{q} + \frac{1}{2} \trace{\nabla^2h_i QQ^T} + \alpha h_i = 0, \qquad \ \ \  \forall i \in [m], \notag \\
        \nabla g_j^T \bar{q} + \frac{1}{2} \trace{\nabla^2g_j QQ^T} + \alpha (g_j + \epsilon) = 0, \ \forall j \in I_x
    \end{cases}
    \end{align}
    Then, there exists a closed form SDE of \eqref{eqn:mixed_OLLA_premitive_form_app} given by:
    \begin{align*}
        dX_t &= -[\Pi(X_t) \nabla f(X_t) +\alpha \nabla J(X_t)^T G^{-1}(X_t) J(X_t)]dt + \calH(X_t)dt + \sqrt{2}\Pi(X_t),&& \text{(Ito)} \\
        dX_t &= -[\Pi(X_t) \nabla f(X_t) +\alpha \nabla J(X_t)^T G^{-1}(X_t) J(X_t)]dt + \sqrt{2}\Pi(X_t) \circ dW_t,&& \text{(Strato.)}
    \end{align*}
    where $\circ$ denotes the Stratonovich integral and
    \begin{equation*}
        \calH := -\nabla J^T G^{-1} \lbra{\trace{\nabla^2 h_1 \Pi}, ..., \trace{\nabla^2 h_{m} \Pi} \trace{\nabla^2 g_{i_1} \Pi} ,..., \trace{\nabla^2 g_{i_{\abs{I_x}}} \Pi}}^T
    \end{equation*} is the related Ito-Stratonovich correction, or mean curvature of $\mbra{ x \in \bbR^d \mid h(x)=0 , g_{I_x}(x)= 0 }$.
\end{proposition*}
\begin{proof}
Define $J(x) := [h(x)^T, g_{I_x}^T +\epsilon \eye_{\abs{I_x}}]^T \in \bbR^{m + \abs{I_x}}$ and $\nabla J(x) := [\nabla h(x)^T, \nabla g_{I_x}(x)^T]^T \in \bbR^{(m+ \abs{I_x}) \times d}$. Then, the Lagrangian function associated with the optimization problem for $Q$ becomes  $\calL(\bar{Q}, \Lambda) := \norm{\sqrt{2}I_d - \bar{Q}}_F^2 + \trace{\Lambda^T \nabla J \bar{Q} }$ where $\Lambda \in \bbR^{(m+\abs{I_x}) \times d}$ is a Lagrangian multiplier. Then, the stationarity condition with respect to $\bar{Q}$ gives
\begin{equation*}
    0 = \frac{\partial L}{\partial \bar{Q}} = -2(\sqrt{2}I- \bar{Q}) + \Lambda^T \nabla J  \quad \Rightarrow \quad \bar{Q} = \sqrt{2}I - \frac{1}{2} \Lambda^T \nabla J.
\end{equation*}
Also, the constraint condition implies
\begin{equation*}
    0 = \bar{Q} \nabla J^T =  (\sqrt{2}I - \frac{1}{2}\Lambda^T \nabla J) \nabla J^T  \quad \Rightarrow \quad \Lambda^T = 2\sqrt{2} \nabla J^T (\nabla J \nabla J^T)^{-1},
\end{equation*}
where $\nabla J \nabla J^T$ is invertible due to the LICQ condition. Therefore, we get optimal $Q = \Pi :=  \sqrt{2}\sbra{I - \nabla J^T (\nabla J \nabla J^T)^{-1} \nabla J}$. 
For the $q$ part, we set 
\begin{equation*}
    b := \frac{1}{2}\lbra{\trace{\nabla^2 h_1 QQ^T} + \alpha h_1, ..., \trace{\nabla^2 g_{i_{\abs{I_x}}} QQ^T } + \alpha (g_{i_{\abs{I_x}}} +\epsilon)}^T
\end{equation*}
and the associated Lagrangian function $L(\bar{q}, \lambda) := \norm{\bar{q} + \nabla f(x)}^2 + \lambda^T (\nabla J \bar{q} + b)$ with $\lambda \in \bbR^{m + \abs{I_x}}$ being an Lagrangian multiplier. Then, the stationarity with respect to $\bar{q}$ gives
\begin{equation*}
    0 = \frac{\partial L}{\partial \bar{q}} = 2(\bar{q} + \nabla f) + \nabla J^T \lambda  \quad \Rightarrow \quad \bar{q} = -\nabla f - \nabla J^T \lambda.
\end{equation*}
Again, the constraint condition implies
\begin{equation*}
    0 = \nabla J \bar{q} + b =-\nabla J \nabla f - (\nabla J \nabla J^T ) \lambda + b \quad \Rightarrow \quad \lambda = (\nabla J \nabla J^T)^{-1} [b - \nabla J \nabla f]
\end{equation*}
using the invertibility of $\nabla J \nabla J^T$. By plugging this expression to $\bar{q}$, we recover the optimal $q: = -\Pi \nabla f - \nabla J^T (\nabla J \nabla J^T)^{-1} b$. Therefore, the Ito version of closed form SDE is given by  
\begin{equation*}
    dX_t = -[\Pi(X_t) \nabla f(X_t) + \nabla J(X_t)^T G(X_t)^{-1} b(X_t)]dt + \sqrt{2} \Pi(X_t)dW_t
\end{equation*}
where $G:= \nabla J \nabla J^T$ is the associated Gram matrix. Also, some tensor-calculus computation (Equation 3.46 in \cite{rousset2010free})  gives
\begin{equation*}
     -\nabla J^T G^{-1} \lbra{\trace{\nabla^2 h_1 \Pi}, ..., \trace{\nabla^2 h_{m} \Pi}, \trace{\nabla^2 g_{i_1} \Pi} ,..., \trace{\nabla^2 g_{i_{\abs{I_x}}} \Pi }}^T = \nabla \Pi(x) \Pi(x)
\end{equation*}
which is the Ito-Stratonovich correction term. Furthermore, applying the technique in \cite{rousset2010free} (Remark 3.17), we can recover that this expression is equal to the mean curvature term $\calH(x) $ of a manifold defined by $\Sigma_{I_x} := \mbra{x \in \bbR^d \mid h(x)=0 , g_{I_x}(x) = 0}$. Therefore, we get the following closed form expression of the SDE:
\begin{align*}
    dX_t &= -[\Pi(X_t) \nabla f(X_t) +\alpha \nabla J(X_t)^T G^{-1}(X_t) J(X_t)]dt + \calH(X_t)dt + \sqrt{2}\Pi(X_t),&& \text{(Ito)} \\
    dX_t &= -[\Pi(X_t) \nabla f(X_t) +\alpha \nabla J(X_t)^T G^{-1}(X_t) J(X_t)]dt + \sqrt{2}\Pi(X_t) \circ dW_t,&& \text{(Strato.)}
\end{align*}
where $\calH = -\nabla J^T G^{-1} \lbra{\trace{\nabla^2 h_1 \Pi}, ..., \trace{\nabla^2 h_{m} \Pi}, \trace{\nabla^2 g_{i_1} \Pi} ,..., \trace{\nabla^2 g_{i_{\abs{I_x}}} \Pi}}^T$ is the associated Ito-Stratonovich correction term (or mean curvature term of $\Sigma_{I_x}$). 
\end{proof}

\begin{lemma*}[Exponential decay of constraint functions] \ 
    \label{lem:expoential decay of constraint functions_app}
    The dynamics induced by \eqref{eqn:mixed_OLLA_premitive_form_app} satisfies the following properties almost surely for $\forall i \in [m], \forall j \in I_{X_0}$:
    \begin{equation}
        h_i(X_t) = h_i(X_0)e^{-\alpha t}, \quad t\geq0
    \end{equation}
    and
    \begin{equation*}
        \begin{cases}
            \begin{aligned}[t]
            g_j(X_t) &= -\epsilon + (g_j(X_0) + \epsilon) e^{-\alpha t}, \quad && t\leq \frac{1}{\alpha} \ln \sbra{ \frac{g_j(X_0)+\epsilon}{\epsilon}}\\
            g_j(X_t) &\leq 0,\quad && t\geq \frac{1}{\alpha} \ln \sbra{ \frac{g_j(X_0)+\epsilon}{\epsilon}}
            \end{aligned}
        \end{cases}
    \end{equation*}
    with $g_j(X_t) \leq 0, \forall t \geq 0$ for $j \notin I_{X_0}$, where $I_x := \mbra{k \in [l] \mid g_k(x) \geq 0}$ is the index set of active inequality constraints.
\end{lemma*}
\begin{proof}
    Observe that, for each $k \in [m]$, the Stratonovich chain rule implies
    \begin{align*}
        dh_k(X_t) &= \nabla h_k(X_t)^T \lbra{(-\Pi(X_t)\nabla f(X_t) - \alpha \nabla J(X_t)^T G^{-1}(X_t) J(X_t))dt + \sqrt{2} \Pi(X_t) \circ dW_t} \\
        & \overset{(1)}{=} -\alpha \nabla h_k(X_t)^T \nabla J(X_t)^T G^{-1}(X_t)J(X_t) dt \\
        &= -\alpha \sum_{i,j=1}^{m+I_x} [G(X_t)]_{ki}[G^{-1}(X_t)]_{ij} h_j(X_t)dt = -\alpha h_k(X_t)dt,
    \end{align*}
    where (1) holds due to the fact $\nabla h_k(x)^T \Pi(x) = 0$. By integrating both side with respect to $t$, we recover $h(X_t) = h(X_0) e^{-\alpha t}, \, t\geq 0$ almost surely. Repeating the same calculation for $g_j$ for $k \in I_x$, we obtain
    \begin{equation*}
        dg_k(X_t) = -\alpha \sum_{i,j = 1}^{m+I_x} [G(X_t)]_{ki}[G^{-1}(X_t)]_{ij}(g_j(X_t) +\epsilon) dt = -\alpha (g_k(X_t) + \epsilon) dt,
    \end{equation*}
    which again recovers $g_k(X_t) = -\epsilon  + (g_k(X_0) + \epsilon)e^{-\alpha t}$. Furthermore, once $g_j(X_t) \leq 0$, it is instantaneously reflected into interior of $\Sigma$ whenever it hits the boundary $\partial\Sigma$. Therefore, $g_j(X_t) \leq 0$ holds for $\forall t\geq0, j \in I_{X_0}$.
\end{proof}

%% file: appendix/Proof_of_theoretical_results_Equality_constraint_LOLD.tex
\section{Proof of Theoretical Results - Equality-constraint OLLA}
\label{appendix:sec:Proof of theoretical results - Equality-constraint OLLA}
Observe that when the constraints are only equality constraints, the equality-constraint OLLA (\ref{eqn:mixed_OLLA_closed_form}) is given by 
\begin{equation}
    \label{eqn:equality_constrained OLLA}
    dX_t = -[\Pi(X_t)\nabla f(X_t) + \alpha \nabla h(X_t) G^{-1}(X_t) h(X_t)]dt +\sqrt{2} \Pi(X_t) \circ dW_t.
\end{equation}

The high-level proof idea of \cref{thm:Convergence result for equality-constrained OLLA} is to decompose the convergence analysis into two parts: (1) Convergence of $W_2$ distance between $\rho_t$ and $\tilde{\rho}_t$, (2) Convergence of $\KL^\Sigma$ between $\tilde{\rho}_t$ and $\rho_{\Sigma}$ where $\rho_t, \tilde{\rho}_t$ are the law of $X_t, Y_t (:= \pi(X_t))$ respectively and $\rho_{\Sigma}$ is the law of the target 
distribution, which satisfies $d\rho_\Sigma \propto \exp(-f(x)) d\sigma_\Sigma$.
\subsection{Upper Bound of $W_2(\rho_t, \tilde{\rho}_t)$}
\begin{lemma_ap}[Upper bound of $W_2(\rho_t, \tilde{\rho}_t)$] \ 
    \label{lem:bound of W_2_eq_only}
    Let $\rho_t$ be the law of $X_t$ which follows equality-constrained OLLA (\ref{eqn:equality_constrained OLLA}) and define $t_0 := \frac{1}{\alpha} \ln \sbra{ \frac{1}{\delta}}$. For $t \geq t_0$, the law $\tilde{\rho}_t$ of $Y_t:= \pi(X_t)$ is well-defined and it holds that
    \begin{equation}
        W_2(\rho_t, \tilde{\rho}_t) \leq \frac{M_h}{\kappa} e^{-\alpha t}.
    \end{equation}
\end{lemma_ap}
\begin{proof}
    For $t \geq t_0$, observe that $\norm{X_t-Y_t}_2 = \norm{X_t -\pi(X_t)}_2 \leq \frac{1}{\kappa}\norm{h(X_t)}_2 \leq \frac{M_h}{\kappa} e^{-\alpha t}$  by \cref{lem:regularity lemma} and \cref{lem:expoential decay of constraint functions_app}. Then, by integrating both sides with respect to optimal coupling of $\rho_t$ and $\tilde{\rho}_t$, we get
    \begin{equation*}
        W_2(\rho_t, \tilde{\rho}_t) \leq \sbra{\mean{\norm{X_t-Y_t}}_2^2}^{\frac{1}{2}} \overset{(\circ)}{\leq} \mean{\norm{X_t- Y_t}_2}  \leq \frac{M_h}{\kappa} e^{-\alpha t},
    \end{equation*}
    where $(\circ)$ holds by Jensen's inequality.
\end{proof}
\clearpage
\subsection{Upper Bound of $\KL^\Sigma(\tilde{\rho}_t || \rho_\Sigma)$}
\begin{lemma_ap}[SDE representation of projected process] \
    \label{lem:SDE representation of projected process}
    Let $X_t$ be the stochastic process following the SDE:
    \begin{equation*}
        dX_t = b(X_t, t) dt + \sqrt{2}\Pi(X_t) \circ dW_t
    \end{equation*}
    where $X_t \in \hat{U}_\delta (\Sigma), \ \forall t\geq 0$. Then, the stochastic process $Y_t$ defined by $Y_t = \pi(X_t)$ follows the SDE below:
    \begin{align*}
        dY_t &= \Pi(Y_t)b(Y_t,t)dt + \sqrt{2}\Pi(Y_t)\circ dW_t + \lbra{\nabla \pi(X_t)b(X_t, t) - \nabla \pi(Y_t)b(Y_t, t)}dt \\ 
        &+ \sqrt{2}\lbra{\nabla \pi(X_t)\Pi(X_t) - \nabla \pi(Y_t) \Pi(Y_t)} \circ dW_t
    \end{align*}
\end{lemma_ap}
\begin{proof}
    From the the Stratonovich chain rule, we observe that 
    \begin{align*}
        dY_t = d\pi(X_t)&= \nabla \pi(X_t)b(X_t, t)dt + \sqrt{2}\nabla\pi(X_t) \Pi(X_t) \circ dW_t.
    \end{align*}
    This expression can be re-written as follows
    \begin{align*}
        dY_t&= \lbra{\nabla \pi(Y_t) b(Y_t,t) dt +\sqrt{2}\nabla \pi(Y_t) \Pi(Y_t) \circ dW_t} + \lbra{\nabla \pi(X_t)b(X_t, t) - \nabla \pi(Y_t)b(Y_t,t)}dt \\
        &+ \sqrt{2}\lbra{\nabla \pi(X_t)\Pi(X_t) - \nabla \pi(Y_t) \Pi(Y_t)} \circ dW_t\\
        &\overset{(\circ)}= \lbra{\Pi(Y_t) b(Y_t,t) dt +\sqrt{2}\Pi(Y_t) \circ dW_t} + \lbra{\nabla \pi(X_t)b(X_t, t) - \nabla \pi(Y_t)b(Y_t,t)}dt \\
        &+ \sqrt{2}\lbra{\nabla \pi(X_t)\Pi(X_t) - \nabla \pi(Y_t) \Pi(Y_t)} \circ dW_t
    \end{align*}
    where $(\circ)$ holds because $\nabla \pi(y) = \Pi(y)$, $\Pi(y)^2 = \Pi(y)$ (idempotent) for $\forall y \in \Sigma$ and $Y_t \in \Sigma$.
\end{proof}
\begin{corollary_ap}[SDE representation of projected process from equality-constrained OLLA] \ 
    \label{cor:SDE representation of projected process from equality-constrained OLLA}
    Let $X_t$ be the stochastic process following equality-constrained OLLA (\ref{eqn:equality_constrained OLLA}). Then, for $t \geq t_0 (:=\frac{1}{\alpha} \ln \sbra{\frac{1}{\delta}})$, the projected process $Y_t := \pi(X_t)$ follows the following SDE:
    \begin{align*}
        dY_t = \lbra{-\Pi(Y_t)\nabla f(Y_t)+b_N(Y_t,t)}dt +\sqrt{2}\Pi(Y_t)(I+A_N(Y_t,t)) \circ dW_t 
    \end{align*}
    where $\norm{b_N(Y_t,t)}_2 = C_{b_N} e^{-\alpha t}, \norm{A_N(Y_t,t)} = C_{A_N} e^{-\alpha t}$ for $t\geq 0$ almost surely for some constant $C_{b_N}, C_{A_N}:= \frac{C_{L_A}M_h}{\kappa} > 0$ with $C_{L_A}$ being the Lipschitz constant of $\nabla \pi(x) \Pi(x) $ on $\hat{U}_\delta(\Sigma)$
\end{corollary_ap}
\begin{proof}
    By applying \cref{lem:SDE representation of projected process} to the SDE (\ref{eqn:equality_constrained OLLA}), it holds that
    \begin{align*}
        dY_t &= \Pi(Y_t)b(Y_t,t)dt + \sqrt{2}\Pi(Y_t)\circ dW_t + \lbra{\nabla \pi(X_t)b(X_t, t) - \nabla \pi(Y_t)b(Y_t,t)}dt\\ 
        &+ \sqrt{2}\lbra{\nabla \pi(X_t)\Pi(X_t) - \nabla \pi(Y_t) \Pi(Y_t)} \circ dW_t
    \end{align*}
    for $b(x, t) = b(x) :=  -\lbra{\nabla \Pi(x) \nabla f(x) + \alpha \nabla h(x)^T G^{-1} (x) h(x)}$. By using \cref{lem:expoential decay of constraint functions_app}, \cref{thm:recoverable tubular neighborhood}, we can set $X_t = \zeta(Y_t, h(X_0)e^{-\alpha t})$ where $\zeta:\Sigma \times \bbR^m \rightarrow \hat{U}_\delta(\Sigma)$ is the recovery map. Now, since $X_t, Y_t \in \hat{U}_\delta(\Sigma)$ and the closure of $\hat{U}_\delta(\Sigma)$ is compact, $\nabla \pi (x) b(x)$ and $\nabla \pi(x) \Pi(x)$ is $C_{L_b}, C_{L_A}$-Lipschitz on $\hat{U}_\delta(\Sigma)$, respectively for some $C_{L_b}, C_{L_A} > 0 $. Therefore, it holds that
    \begin{equation*}
        \norm{b_N(Y_t,t)}_2 \leq C_{L_b}\norm{\zeta(Y_t, h(X_0)e^{-\alpha t})- Y_t}_2 \leq \frac{C_{L_b} \norm{h(X_0)}_2}{\kappa}e^{-\alpha t} \leq \frac{C_{L_b}M_h}{\kappa}e^{-\alpha t}
    \end{equation*}
    where $b_N(Y_t,t):= \nabla \pi(\zeta(Y_t, h(X_0) e^{-\alpha t}) b(\zeta(Y_t, h(X_0) e^{-\alpha t})) - \nabla \pi(Y_t) b(Y_t)$ and the second last inequality comes from \cref{lem:regularity lemma}. Similarly, we obtain the bound of $A_N(Y_t,t) := \nabla \pi(\zeta(Y_t, h(X_0)e^{-\alpha t}))\Pi(\zeta(Y_t, h(X_0)e^{-\alpha t})) - \nabla \pi(Y_t) \Pi(Y_t)$ as follows:
    \begin{equation*}
        \norm{A_N(Y_t,t)}_2 \leq C_{L_A}\norm{\zeta(Y_t, h(X_0)e^{-\alpha t}) - Y_t}_2 \leq \frac{C_{L_A} \norm{h(X_0)}_2}{\kappa} e^{-\alpha t} \leq \frac{C_{L_A}M_h}{\kappa}e^{-\alpha t}
    \end{equation*}
    Finally, we complete the proof by setting $C_{b_N} := \frac{C_{L_b}M_h}{\kappa}, C_{A_N} := \frac{C_{L_A}M_h}{\kappa}$ and observing that $\nabla \pi(x) = \Pi(\pi(x)) \nabla \pi(x)$ for $\forall x \in \hat{U}_\delta(\Sigma)$, which implies $A_N(Y_t,t) = \Pi(Y_t) A_N(Y_t,t)$.
\end{proof} 
The following theorem is a Fokker-Planck equation of a Stratonovich SDE defined on a Riemannian manifold. We will rely on this theorem to describe the time derivative of $\tilde{\rho}_t$.
\begin{theorem_ap}[Fokker-Planck equation on Riemannian manifold \citep{chirikjian2009stochastic, huang2022riemannian}] \
\label{thm:Fokker-Planck equation in Riemannian manifold}
Let $X_t \in \Sigma$ be a stochastic process following the SDE:
\begin{equation*}
    dX_t = V_0 dt + \sum_{k=1}^d V_k \circ dB_t^k,
\end{equation*}
where $V_0, V_k$ are smooth vector fields on $\Sigma$ for each $k \in [d]$ and $B_t^k$ are $k$th components of Brownian motion $B_t$. Then, the law $\rho_t$ of the stochastic process $X_t$ satisfies the following Fokker-Planck equation:
\begin{equation*}
    \partial_t \rho_t = -\divergence_\Sigma(\rho_t V_0) + \frac{1}{2}\sum_{k=1}^d \divergence_\Sigma(\divergence_\Sigma(\rho_t V_k)V_k).
\end{equation*}
\end{theorem_ap}
\begin{lemma_ap}[Upper bound of $\KL^\Sigma(\tilde{\rho}_t  || \rho_\Sigma)$] \ 
    \label{lem:Upper bound of KL_Sigma}
    Assume that $\rho_\Sigma$ satisfies the LSI condition with constant $\lambda_{\text{LSI}}$. Let $X_t$ be the stochastic process following equality-constrained OLLA (\ref{eqn:equality_constrained OLLA}) and $\tilde{\rho}_t$ be the law of $Y_t := \pi(X_t)$ after $t \geq t_{\text{cut}}, t_{\text{cut}}:= \max \mbra{\frac{1}{\alpha} \ln \delta, \frac{1}{\alpha} \ln (C_5)}$.
    Then, for $\alpha \neq 2\lambda_{\text{LSI}}$, the following non-asymptotic convergence rate of $\KL^\Sigma(\tilde{\rho}_t  || \rho_\Sigma)$ can be obtained as follows
    \begin{align*}
        \KL^\Sigma (\tilde{\rho}_t || \rho_\Sigma) & \leq \exp \sbra{-2 \lambda_{\text{LSI}} (t-t_{\text{cut}}) - \frac{2\lambda_{\text{LSI}}C_5}{\alpha}(e^{-\alpha t} -e^{-\alpha t_{\text{cut}}}) }[\KL^\Sigma (\tilde{\rho}_{t_{\text{cut}}} || \rho_\Sigma) \\
        &+ C_6 \int_{t_{\text{cut}}}^t \exp \sbra{ 2\lambda (s -{t_{\text{cut}}}) + \frac{2\lambda_{\text{LSI}}C_5}{\alpha} (e^{-\alpha s} - e^{-\alpha t_{\text{cut}}})}e^{-\alpha s} ds]
    \end{align*}
    In particular, if $\alpha > 2 \lambda_{\text{LSI}}$, it becomes 
    \begin{equation*}
        \KL^\Sigma (\tilde{\rho}_t || \rho_\Sigma) \leq \exp \sbra{-2 \lambda_{\text{LSI}} (t-t_{\text{cut}}) - \frac{2\lambda_{\text{LSI}}C_5}{\alpha}(e^{-\alpha t} -e^{-\alpha t_{\text{cut}}}) }[\KL^\Sigma (\tilde{\rho}_{t_{\text{cut}}} || \rho_\Sigma) + C_7]
    \end{equation*}
    for some constants $C_5=\calO(1+C_{A_N}+C_{A_N}^2), C_6, C_7 := \frac{C_6 e^{-\alpha t_{\text{cut}}}}{\alpha - 2\lambda_{\text{LSI}}}  > 0$.
\end{lemma_ap}
\begin{proof}
    By \cref{thm:Fokker-Planck equation in Riemannian manifold}, \cref{cor:SDE representation of projected process from equality-constrained OLLA}, and the choice of $\nabla f = -\nabla \ln \rho_\Sigma$, we know that the projected process $Y_t$ is given by
    \begin{equation*}
        dY_t = \lbra{-\Pi(Y_t)\nabla f(Y_t)+b_N(Y_t,t)}dt +\sqrt{2}\Pi(Y_t)(I+A_N(Y_t,t)) \circ dW_t 
    \end{equation*}
    and its associated Fokker-Planck equation can be written as follows:
    \begin{equation*}
        \partial_t \tilde{\rho}_t = -\divergence_\Sigma (\tilde{\rho}_t \sbra{\nabla_\Sigma \ln \rho_\Sigma + b_N}) + \sum_{k=1}^d \divergence_\Sigma \sbra{ \divergence_\Sigma (\tilde{\rho}_t (f_k + \delta_k)) (f_k +\delta_k)}
    \end{equation*}
    where $f_k = \Pi e_k, \delta_k = \Pi A_N e_k$, and $e_k$ is $k$th standard basis vector for $\bbR^d$. Now observe the following equations:
    \begin{align*}
        \partial_t \KL^\Sigma (\tilde{\rho}_t ||\rho_\Sigma) &= \int_\Sigma \partial_t \tilde{\rho}_t \cdot \ln \sbra{\frac{\tilde{\rho}_t}{\rho_\Sigma}} d\sigma_\Sigma + \partial_t \int \tilde{\rho}_t d\sigma_\Sigma =  \int_\Sigma \partial \tilde{\rho}_t \ln \sbra{\frac{\tilde{\rho}_t}{\rho_\Sigma}} d\sigma_\Sigma\\
        &= \underbrace{\int_\Sigma \lbra{ -\divergence_\Sigma (\tilde{\rho}_t \nabla_\Sigma \ln \rho_\Sigma) + \sum_{k=1}^d \divergence_\Sigma (\divergence_\Sigma(\tilde{\rho}_t f_k)f_k) } \ln \sbra{\frac{\tilde{\rho}_t}{\rho_\Sigma}} d\sigma_\Sigma}_{\text{Term (1)}} \\
        & + \underbrace{\int_\Sigma \lbra{-\divergence_\Sigma (\tilde{\rho}_t b_N)} \ln \sbra{\frac{\tilde{\rho}_t}{\rho_\Sigma}} d\sigma_\Sigma}_{\text{Term (2)}} \\
        &+ \int_\Sigma \underbrace{\sum_{k=1}^d \lbra{\divergence_\Sigma (\divergence_\Sigma(\tilde{\rho}_t \delta_k)f_k + \divergence_\Sigma(\tilde{\rho}_t f_k)\delta_k + \divergence_\Sigma(\tilde{\rho}_t \delta_k)\delta_k} \ln \sbra{\frac{\tilde{\rho}_t}{\rho_\Sigma}} d \sigma_\Sigma}_{\text{Term (3)}}.
    \end{align*}
    \textbf{Analysis of Term (1) - induced by the main SDE}.\quad From integration by parts, we obtain 
    \begin{align*}
        \text{Term (1)} &= \int_\Sigma \tilde{\rho}_t \inner{ \nabla_\Sigma \ln \rho_\Sigma, \nabla_\Sigma \ln \sbra{\frac{\tilde{\rho}_t}{\rho_\Sigma}}} d\sigma_\Sigma - \sum_{k=1}^d \int_\Sigma \divergence_\Sigma (\tilde{\rho}_t f_k) \inner{f_k, \nabla_\Sigma \ln \sbra{ \frac{\tilde{\rho}_t}{\rho_\Sigma}}} d \sigma_\Sigma \\
        &= \int_\Sigma \tilde{\rho}_t \inner{ \nabla_\Sigma \ln \rho_\Sigma, \nabla_\Sigma \ln \sbra{\frac{\tilde{\rho}_t}{\rho_\Sigma}}} d\sigma_\Sigma -\sum_{k=1}^d  \int_\Sigma \inner{ \nabla_\Sigma \tilde{\rho}_t, f_k} \inner{f_k, \nabla_\Sigma \ln \sbra{\frac{\tilde{\rho}_t}{\rho_\Sigma}}} d\sigma_\Sigma \\
        & -\sum_{k=1}^d \int_\Sigma \tilde{\rho}_t \divergence_\Sigma (f_k) \inner{f_k, \nabla_\Sigma \ln \sbra{\frac{\tilde{\rho}_t}{\rho_\Sigma}}} d\sigma_\Sigma\\
        &\overset{(\triangle)}{=} \int_\Sigma \tilde{\rho}_t \inner{ \nabla_\Sigma \ln \rho_\Sigma, \nabla_\Sigma \ln \sbra{\frac{\tilde{\rho}_t}{\rho_\Sigma}}} d\sigma_\Sigma -\int_\Sigma \tilde{\rho}_t \inner{\nabla_\Sigma \ln \tilde{\rho}_t, \nabla_\Sigma \ln \sbra{\frac{\tilde{\rho}_t}{\rho_\Sigma}}} d\sigma_\Sigma \\
        &= -\int_\Sigma \tilde{\rho}_t \norm{\nabla_\Sigma \ln \sbra{\frac{\tilde{\rho}_t}{\rho_\Sigma}}}_2^2 d\sigma_\Sigma = - I^\Sigma (\tilde{\rho}_t || \rho_\Sigma) 
    \end{align*}
    where $(\triangle)$ holds using \autoref{lem:sum of divf_k f_k = 0} (the third term = 0) and the fact that $\tilde{\rho}_t \nabla_\Sigma \ln \tilde{\rho}_t = \nabla_\Sigma \tilde{\rho}_t$.\\
    \\
    \noindent \textbf{Analysis of Term (2) - induced by the noise drift $b_N$}.\quad Again, using the integration by parts, we observe that
    \begin{align*}
        \abs{\text{Term (2)}} &= \abs{\int_\Sigma \divergence_\Sigma (\tilde{\rho}_t b_N) \ln \sbra{\frac{\tilde{\rho}_t}{\rho_\Sigma}} d\sigma_\Sigma} = \abs{\int_\Sigma \tilde{\rho}_t \inner{b_N, \nabla_\Sigma \ln \sbra{\frac{\tilde{\rho}_t}{\rho_\Sigma}}} d\sigma_\Sigma} \\
        & \leq \int_\Sigma \tilde{\rho}_t \norm{b_N}_2 \norm{\nabla_\Sigma \ln \sbra{\frac{\tilde{\rho}_t}{\rho_\Sigma}}}_2 d\sigma_\Sigma \leq C_{b_N} e^{-\alpha t} \int_\Sigma \tilde{\rho}_t \cdot 1 \cdot \norm{\nabla_\Sigma \ln \sbra{\frac{\tilde{\rho}_t}{\rho_\Sigma}}}
        _2 d\sigma_\Sigma \\
        & \overset{(\square)}{\leq} C_{b_N} e^{-\alpha  t} \mkern-5mu \int_\Sigma \tilde{\rho}_t \lbra{ \frac{C_{b_N}}{4} + \frac{1}{C_{b_N}} \norm{\nabla_\Sigma \ln \sbra{\frac{\tilde{\rho}_t}{\rho_\Sigma}}}_2^2} d\sigma_\Sigma  = e^{-\alpha t} I^\Sigma (\tilde{\rho}_t || \rho_\Sigma) + \frac{C_{b_N}^2}{4} e^{-\alpha t},
    \end{align*}
    where $(\square)$ inequality holds using the AM-GM inequality.\\
    \\
    \noindent \textbf{Analysis of Term (3) - induced by noise diffusion $A_N$}. \quad To analyze Term (3), we apply integration by parts and the chain rule of the divergence, i.e., $\divergence_\Sigma (\tilde{\rho}_t a_k) = \inner{\nabla_\Sigma \tilde{\rho}_t, a_k} + \tilde{\rho}_t \divergence_\Sigma (a_k)$ for a vector field $a_k$ on $\Sigma$:
    \begin{align*}
        \text{Term (3)} &= -\sum_{k=1}^d \int_\Sigma \tilde{\rho}_t  \inner{(\divergence_\Sigma \delta_k)\delta_k, \mkern-5mu \nabla_\Sigma \ln \sbra{ \frac{\tilde{\rho}_t}{\rho_\Sigma}}} d\sigma_\Sigma -\sum_{k=1}^d \int_\Sigma \tilde{\rho}_t  \inner{(\divergence_\Sigma \delta_k) f_k, \mkern-5mu \nabla_\Sigma \ln \sbra{ \frac{\tilde{\rho}_t}{\rho_\Sigma}}} d\sigma_\Sigma\\
        &-\sum_{k=1}^d \int_\Sigma \tilde{\rho}_t  \inner{(\divergence_\Sigma f_k) \delta_k, \nabla_\Sigma \ln \sbra{ \frac{\tilde{\rho}_t}{\rho_\Sigma}}} d\sigma_\Sigma -\sum_{k=1}^d \int_\Sigma \inner{\nabla_\Sigma \tilde{\rho}_t, \delta_k} \inner{f_k, \nabla_\Sigma \ln \sbra{\frac{\tilde{\rho}_t}{\rho_\Sigma}}} d\sigma_\Sigma \\
        &- \sum_{k=1}^d \int_\Sigma \inner{\nabla_\Sigma \tilde{\rho}_t, f_k} \inner{\delta_k, \nabla_\Sigma \ln \sbra{\frac{\tilde{\rho}_t}{\rho_\Sigma}}} d\sigma_\Sigma - \sum_{k=1}^d \int_\Sigma \inner{\nabla_\Sigma \tilde{\rho}_t, \delta_k} \inner{\delta_k, \nabla_\Sigma \ln \sbra{\frac{\tilde{\rho}_t}{\rho_\Sigma}}} d\sigma_\Sigma.
    \end{align*}
    Now, note that $\sum_{k=1}^d \delta_k f_k^T = \Pi A_N \Pi$, $\sum_{k=1}^d f_k \delta_k^T = \Pi A_N^T \Pi$, and $\sum_{k=1}^d \delta_k \delta_k^T = \Pi A_N A_N^T \Pi$. Then it holds that
    \begin{align*}
         &\abs{-\sum_{k=1}^d  \int_\Sigma \inner{ \nabla_\Sigma \tilde{\rho}_t, \delta_k} \inner{f_k, \nabla_\Sigma \ln \sbra{\frac{\tilde{\rho}_t}{\rho_\Sigma}}} d\sigma_\Sigma} = \abs{\int_\Sigma \tilde{\rho}_t (\nabla_\Sigma \ln \tilde{\rho}_t)^T \Pi A_N \Pi (\nabla_\Sigma \ln \sbra{\frac{\tilde{\rho}_t}{\rho_\Sigma}})} d\sigma_\Sigma \\
        &\leq C_{A_N} e^{-\alpha t} \int_\Sigma \tilde{\rho}_t \norm { \nabla_\Sigma \ln \tilde{\rho}_t}_2 \norm{\nabla_\Sigma \ln \sbra{\frac{\tilde{\rho}_t}{\rho_\Sigma}}}_2 d\sigma_\Sigma \\ &\overset{(\times)}{\leq} C_{A_N} e^{-\alpha t}  I^\Sigma (\tilde{\rho}_t || \rho_\Sigma) + C_{A_N}C_3 e^{-\alpha t} \int_\Sigma \tilde{\rho}_t \norm{\nabla_\Sigma \ln \sbra{\frac{\tilde{\rho}_t}{\rho_\Sigma}}}_2 d\sigma_\Sigma \\
        & \overset{(\circ)}{\leq} C_{A_N} e^{-\alpha t} I^{\Sigma} (\tilde{\rho}_t || \rho_\Sigma) +  C_{A_N} C_3 e^{-\alpha t} \int_\Sigma \tilde{\rho}_t \lbra{ \frac{C_3}{4} + \frac{1}{C_3} \norm{\nabla_\Sigma \ln \sbra{\frac{\tilde{\rho}_t}{\rho_\Sigma}}}^2_2} d\sigma_\Sigma \\
        & \leq 2C_{A_N} e^{-\alpha t} I^\Sigma (\tilde{\rho}_t || \rho_\Sigma) + \frac{C_{A_N}C_3^2}{4} e^{-\alpha t},
    \end{align*}

    where $C_3 := \max_{x \in \Sigma} \norm{\nabla_\Sigma \ln \rho_\Sigma}_2 < \infty$ by compactness of $\Sigma$, and  $(\times)$ holds using the triangle inequality; $\norm{\nabla_\Sigma \ln \tilde{\rho}_t}_2 \leq \norm{\nabla_\Sigma \ln \sbra{\frac{\tilde{\rho}_t}{\rho_\Sigma}}}_2 +\norm{\nabla_\Sigma \ln \rho_\Sigma}_2$. Also, $(\circ)$ comes from the AM-GM inequality. By following the same reasoning, the following inequalities are obtained:
    \begin{align*}
        \abs{-\sum_{k=1}^d \int_\Sigma \inner{ \nabla_\Sigma \tilde{\rho}_t, f_k} \inner{\delta_k, \nabla_\Sigma \ln \sbra{\frac{\tilde{\rho}_t}{\rho_\Sigma}}} d\sigma_\Sigma} &\leq 2C_{A_N} e^{-\alpha t} I^\Sigma (\tilde{\rho}_t || \rho_\Sigma) + \frac{C_{A_N}C_3^2}{4} e^{-\alpha t}  \\
        \abs{-\sum_{k=1}^d \int_\Sigma \inner{ \nabla_\Sigma \tilde{\rho}_t, \delta_k} \inner{\delta_k, \nabla_\Sigma \ln \sbra{\frac{\tilde{\rho}_t}{\rho_\Sigma}}} d\sigma_\Sigma} &\leq 2C_{A_N}^2 e^{-2 \alpha t} I^{\Sigma} (\tilde{\rho}_t || \rho_\Sigma) + \frac{C_{A_N}^2C_3^2}{4} e^{- 2\alpha  t} \\
        & \leq 2C_{A_N}^2 e^{-\alpha t}  I^{\Sigma} (\tilde{\rho}_t || \rho_\Sigma) + \frac{C_{A_N}^2C_3^2}{4} e^{-\alpha t}
    \end{align*}

    Next, we observe that the following terms decay exponentially fast :
    \begin{align*}
        &\abs{\sum_{k=1}^d \int_\Sigma \tilde{\rho}_t \inner{\divergence_\Sigma (f_k)\delta_k, \nabla_\Sigma \ln \sbra{\frac{\tilde{\rho}_t}{\rho_\Sigma}}} d\sigma_\Sigma}  \leq \int_\Sigma \tilde{\rho}_t \norm{\sum_{k=1}^d \divergence_\Sigma (f_k) \delta_k}_2 \norm{\nabla_\Sigma \ln \sbra{\frac{\tilde{\rho}_t}{\rho_\Sigma}}}_2 d\sigma_\Sigma \\
        & \overset{(\oplus)}{\leq} \int_\Sigma \tilde{\rho}_t \norm{\Pi A_N}_2 \norm{\divergence_\Sigma (\Pi)}_2 \norm{\nabla_\Sigma \ln \sbra{\frac{\tilde{\rho}_t}{\rho_\Sigma}}}_2 d\sigma_\Sigma \\
        &\leq C_{A_N} C_4 e^{-\alpha  t} \int_\Sigma \tilde{\rho}_t 1 \cdot \norm{\nabla_\Sigma \ln \sbra{\frac{\tilde{\rho}_t}{\rho_\Sigma}}}_2 d\sigma_\Sigma \\
        & \leq C_{A_N} C_4 e^{-\alpha  t} \int_\Sigma \tilde{\rho}_t \lbra{\frac{C_{A_N}C_4}{4} + \frac{1}{C_{A_N} C_4}\norm{\nabla_\Sigma \ln \sbra{\frac{\tilde{\rho}_t}{\rho_\Sigma}}}_2^2 } d\sigma_\Sigma \\
        &\leq e^{-\alpha t} I^\Sigma (\tilde{\rho}_t || \rho_\Sigma) + \frac{C_{A_N}^2 C_4^2}{4}e^{-\alpha t },
    \end{align*}
    where $C_4$ := $\max_{x \in \Sigma} \norm{\divergence_\Sigma (\Pi(x))}_2 < \infty$, $(\oplus)$ holds because $\sum_{k=1}^d \divergence_\Sigma(f_k) \delta_k = \divergence_\Sigma (\Pi)^T \Pi A_N$ once $\divergence_\Sigma(\Pi)$ is given by $(\divergence_\Sigma \Pi)_k := \divergence_\Sigma (f_k)$ for each $k\in [d]$.
    Similarly, we obtain the following bound by using \autoref{lem:exponential decay of div(Pi A_N)} :
    \begin{align*}
        &\abs{\sum_{k=1}^d \int_\Sigma \tilde{\rho}_t \inner{\divergence_\Sigma (\delta_k) f_k, \nabla_\Sigma \ln \sbra{\frac{\tilde{\rho}_t}{\rho_\Sigma}}} d\sigma_\Sigma}  \leq \int_\Sigma \tilde{\rho}_t \norm{\Pi}_2 \norm{\divergence(\Pi A_N)}_2 \norm{\nabla_\Sigma \ln \sbra{\frac{\tilde{\rho}_t}{\rho_\Sigma}}} d\sigma_\Sigma \\
        &\overset{\text{Lem} \ref{lem:exponential decay of div(Pi A_N)}}{\leq} C_{div} e^{-\alpha t} \int_\Sigma \tilde{\rho}_t \norm{\nabla_\Sigma \ln \sbra{\frac{\tilde{\rho}_t}{\rho_\Sigma}}}_2 d\sigma_\Sigma \leq e^{-\alpha t} I^\Sigma(\tilde{\rho}_t || \rho_\Sigma) + \frac{C_{div}^2}{4}e^{-\alpha t}
    \end{align*}
    and
    \begin{align*}
        &\abs{\sum_{k=1}^d \int_\Sigma \tilde{\rho}_t \inner{\divergence_\Sigma (\delta_k) \delta_k, \nabla_\Sigma \ln \sbra{\frac{\tilde{\rho}_t}{\rho_\Sigma}}} d\sigma_\Sigma}  \leq \int_\Sigma \tilde{\rho}_t \norm{\Pi A_N}_2 \norm{\divergence(\Pi A_N)}_2 \norm{\nabla_\Sigma \ln \sbra{\frac{\tilde{\rho}_t}{\rho_\Sigma}}} d\sigma_\Sigma \\
        &\overset{\text{Lem} \ref{lem:exponential decay of div(Pi A_N)}}{\leq} C_{div}C_{A_N} e^{-2\alpha t}\int_\Sigma \tilde{\rho}_t \norm{\nabla_\Sigma \ln \sbra{\frac{\tilde{\rho}_t}{\rho_\Sigma}}} d\sigma_\Sigma \leq e^{-\alpha t} I^\Sigma (\tilde{\rho}_t || \rho_\Sigma) + \frac{C_{div}^2 C_{A_N}^2}{4} e^{-\alpha t},
    \end{align*}
    where $\divergence_\Sigma(\Pi A_N)$ is similarly defined by $(\divergence_\Sigma \Pi A_N)_k := \divergence_\Sigma (\delta_k)$ for each $k\in [d]$.
    
    \textbf{Applying Gronwall-type inequality}. \quad By summing all the bounds, we arrive at 
    \begin{align*}
        \partial_t \KL^\Sigma (\tilde{\rho}_t || \rho_\Sigma )  \overset{\text{LSI}}{\leq} -2\lambda_{\text{LSI}} (1-C_5 e^{-\alpha t}) \KL^\Sigma (\tilde{\rho}_t || \rho_\Sigma) + C_6 e^{-\alpha t},
    \end{align*}
    with $C_5:= (4 + 4C_{A_N} + 2 C_{A_N}^2)$, $C_6 := \sbra{ \frac{C_{b_N}^2}{4} + \frac{C_{A_N}C_3^2}{2} + \frac{C_{A_N}^2 C_3^2}{4} + \frac{C_{A_N}^2 C_4^2}{4} + \frac{C_{div}^2}{4} + \frac{C_{div}^2 C_{A_N}^2}{4}}$.
    Therefore, the Gr\"onwall-type inequality gives for $t>t_{\text{cut}}, t_{\text{cut}}:= \max \mbra{\frac{1}{\alpha} \ln \delta, \frac{1}{\alpha} \ln (C_5)}$:
    \begin{align*}
        \KL^\Sigma (\tilde{\rho}_t || \rho_\Sigma) & \leq \exp \sbra{-2 \lambda_{\text{LSI}} (t-t_{\text{cut}}) - \frac{2\lambda_{\text{LSI}}C_5}{\alpha}(e^{-\alpha t} -e^{-\alpha t_{\text{cut}}}) }[\KL^\Sigma (\tilde{\rho}_{t_{\text{cut}}} || \rho_\Sigma) \\
        &+ C_6 \int_{t_{\text{cut}}}^t \exp \sbra{ 2\lambda_{\text{LSI}} (s -{t_{\text{cut}}}) + \frac{2\lambda_{\text{LSI}}C_5}{\alpha} (e^{-\alpha s} - e^{-\alpha t_{\text{cut}}})}e^{-\alpha s} ds].
    \end{align*}
     In particular, if $\alpha > 2\lambda_{\text{LSI}}$, it holds that 
    \begin{equation*}
        \KL^\Sigma (\tilde{\rho}_t || \rho_\Sigma) \leq \exp \sbra{-2 \lambda_{\text{LSI}} (t-t_{\text{cut}}) - \frac{2\lambda_{\text{LSI}}C_5}{\alpha}(e^{-\alpha t} -e^{-\alpha t_{\text{cut}}}) }[\KL^\Sigma (\tilde{\rho}_{t_{\text{cut}}} || \rho_\Sigma) + C_7]
    \end{equation*}
    where $C_7 := \frac{e^{-\alpha t_{\text{cut}}}}{\alpha - 2\lambda_{\text{LSI}}}$ from the fact that
    \begin{align*}
        \int_{t_{\text{cut}}}^\infty  \mkern-10mu \exp \sbra{ 2\lambda_{\text{LSI}} (s -{t_{\text{cut}}}) + \frac{2\lambda_{\text{LSI}}C_5}{\alpha} (e^{-\alpha s} - e^{-\alpha t_{\text{cut}}})}e^{-\alpha s} ds &\leq \int_{t_{\text{cut}}}^\infty \mkern-10mu \exp(2\lambda_{\text{LSI}} (s-t_{\text{cut}})) e^{-\alpha s} ds \\
        & = \frac{e^{-\alpha t_{\text{cut}}}}{\alpha - 2\lambda_{\text{LSI}}} < \infty.
    \end{align*}
\end{proof}

\begin{lemma_ap} \ 
\label{lem:sum of divf_k f_k = 0}
Let $\mbra{f_k}_{k=1}^d$ be a set of vectors defined by $f_k = \Pi(x)e_k$, where 
$\Pi(x)$ is the orthogonal projector onto $T_x\Sigma$ and $e_k$ is the $k$th standard basis vector of $\bbR^d$. Then, it holds that
\begin{equation*}
    \sum_{k=1}^d (\divergence_\Sigma f_k)f_k = 0.
\end{equation*}
\end{lemma_ap} 
\begin{proof}
    Recalling that $\Pi(x) = I -\nabla h(x)^T (\nabla h(x) \nabla h(x)^T)^{-1} \nabla h(x)$, we define $N(x) = \nabla h(x)^T (\nabla h(x) \nabla h(x)^T)^{-\frac{1}{2}} \in \bbR^{d\times m }$ so that $N(x)^T N(x) = I_m$ and $\Pi(x) = I - N(x)N(x)^T$. If we let the columns of $N(x)$ to be  $\mbra{n_1(x), ... , n_m(x)}$, then these produce an orthonormal basis of $N_x\Sigma$. This is because $\text{Im}(N(x)) = \text{Im}(\nabla h(x)^T
    )$ (from the invertibility $(\nabla h(x) \nabla h(x)^T)^{-\frac{1}{2}}$) implies $\mbra{n_1(x), ... n_m(x)}$ span $N_x\Sigma$ and $N(x)^T N(x) =I $ guarantees the orthonormality. 
    
    Next, we define a vector field $F(x)$ by $F(x) = \Pi(x) \divergence_\Sigma(\Pi(x))$ where $(\divergence_\Sigma \Pi(x))_k := \divergence_\Sigma (f_k(x))$ for each $k\in [d]$. With this definition, we have $\divergence_\Sigma \Pi = - \divergence_\Sigma (NN^T) = -\sum_{k=1}^m \divergence_\Sigma (n_k n_k^T)$. Now observe that for $l\in [d]$,
    \begin{align*}
        (\divergence_\Sigma (n_k n_k^T))_l &= \trace{\Pi \nabla ((n_k n_k^T)_l)} = \sum_{i,j=1}^d \Pi_{ij} \partial_j (n_kn_k^T)_{il} = \sum_{i,j}^d \lbra{ \Pi_{ij} \partial_j n_{ki} n_{kl} + \Pi_{ij}n_{ki}\partial_jn_{kl}}\\
        &= (\divergence_\Sigma n_k) n_{kl} + \sum_{j=1}^d \underbrace{(n_k \Pi)_j}_{=0}\partial_j n_{kl} = (\divergence_\Sigma n_k) n_{kl}
    \end{align*}
    where $n_{kl}$ is the $l$th component of $n_k$. From this fact, we have the following result: 
    \begin{equation*}
        \divergence_\Sigma \Pi = -\sum_{k=1}^d \divergence_\Sigma(n_kn_k^T) = -\sum_{k=1}^d (\divergence_\Sigma n_k) n_k \ \Rightarrow \  F = \Pi \divergence _\Sigma (\Pi) = 0.
    \end{equation*}
    Finally, the definition of $F$ gives $\sum_{k=1}^d \divergence_\Sigma (f_k) f_k = F$, which is zero by the above argument. 
\end{proof}

\begin{lemma_ap} \ 
    \label{lem:exponential decay of div(Pi A_N)}
    Let $\tilde{\rho}_t$ be the law of the projected process $Y_t$ of $X_t$, where $X_t$ follows equality-constrained OLLA. Define $\delta_k(t,x) := \Pi(x) A_N(x,t) e_k$ and denote $\divergence_\Sigma(\Pi A_N)$ as a vector in $\bbR^d$ such that $(\divergence_\Sigma \Pi A_N)_k := \divergence_\Sigma (\delta_k)$ for each $k\in [d]$. Then, it holds almost surely 
    \begin{equation*}
        \norm{\divergence_\Sigma (\Pi(Y_t)A_N (Y_t,t))}_2 \leq C_{div} e^{-\alpha t} 
    \end{equation*}
    for $t\geq t_0 (:= \frac{1}{\alpha}\ln \sbra{\frac{1}{\delta}})$ and some constant $C_{div}>0$.
\end{lemma_ap}
\begin{proof}
    First, for each $k \in [d]$, observe that $\nabla \delta_k = \nabla \Pi(y)A_N(y,t) e_k + \Pi(y) \nabla A_N(y,t) e_k$ and 
    \begin{equation}
        \label{eqn:div_Sigma delta_k(y)}
        \divergence_\Sigma(\delta_k(y)) = \underbrace{\trace{\Pi(y) \nabla \Pi(y) A_N(y,t) e_k}}_{\text{Term (1)}} + \underbrace{\trace{\Pi(y) \nabla A_N(y,t) e_k}}_{\text{Term (2)}}
    \end{equation}
    where, for a matrix-valued function $G(y)$, $\nabla G(y)$ is the third-order tensor defined by $(\nabla G_{ij}(y))_{ijk} = \frac{\partial G_{ij}(y)}{\partial y_k}$ for $i,j,k \in [d]$, and the gradient $\nabla$ is taken over $y$.
   
    For the Term (1), we know that when $y = Y_t, t \geq t_0$,
    \begin{align*}
        \abs{\text{Term (1)}} \leq \norm{\Pi (\nabla \Pi) (A_N e_k)}_F \leq \norm{(\nabla \Pi) (A_N e_k)}_F \leq K_1 \norm{A_Ne_k}_2 \leq K_1 C_{A_N}e^{-\alpha t} \quad a.s.
    \end{align*}
    where $K_1 := \sup_{y \in \Sigma, \norm{v}_2 =1} \norm{ \nabla \Pi(y)v}_F <\infty$.

    For the Term (2), recall that $A_N(y,t) := \nabla \pi(\zeta(y, h(X_0)e^{-\alpha t}))\Pi(\zeta(y, h(X_0)e^{-\alpha t})) - \nabla \pi(y) \Pi(y)$ conditionally on $X_0$, from \cref{cor:SDE representation of projected process from equality-constrained OLLA}. For the notational convenience, we define $\eta(y,t) := \zeta(y, h(X_0)e^{-\alpha t})$. Then, it follows that
    \begin{align*}
        \nabla A_N(y,t) &= \nabla^2 \pi(\eta(y,t))\nabla \eta(y,t) \Pi(\eta(y,t)) + \nabla \pi (\eta(y,t)) \nabla \Pi(\eta(y,t)) \nabla \eta(y,t) \\
        &- \nabla^2 \pi(y) \Pi(y) - \nabla \pi(y) \nabla \Pi(y) \\
        &= \nabla^2 \pi(\eta(y,t))(\nabla \eta(y,t) - I) \Pi(\eta(y,t)) + \nabla \pi(\eta(y,t))\nabla \Pi(\eta(y,t)) (\nabla \eta(y,t) - I)\\
        &+ \nabla^2 \pi(\eta(y,t)) \Pi(\eta(y,t)) - \nabla^2 \pi(y)\Pi(y) + \nabla \pi(\eta(y,t)) \nabla \Pi(\eta (y,t))- \nabla \pi(y) \nabla \Pi(y).
    \end{align*}
    At this moment, from the recovery map $\zeta$ in \cref{thm:recoverable tubular neighborhood}, the integral form of the remainder gives 
    \begin{equation*}
        \norm{\nabla \eta(y,t) - I}_2 = \norm{\nabla \sbra{\zeta(y,p) - \zeta(y,0)}}_2 = \norm{\nabla_y \zeta(y,p) - \nabla_y \zeta(y,0)}_2 \leq K_2 \norm{p}_2,
    \end{equation*}
    where $p:= h(X_0)e^{-\alpha t}$, $K_2 := \sup_{(y,p) \in \Sigma \times B(0,\delta)}\norm{\nabla_p \nabla_y \zeta(y,p)}_2 < \infty$.

    By combining these results with the previous expression of $\nabla A_N(y,t)$, we get
    \begin{equation*}
        \norm{\nabla A_N(Y_t,t)}_2 \leq D_4 e^{-\alpha t}
    \end{equation*}
    for some $D_4> 0 $, using the boundedness of $\norm{\nabla^2 \pi}_2, \norm{\nabla \pi}_2, \norm{\nabla \Pi}_2$ on $\hat{U}_\delta$, the Lipschitzness of $(\nabla^2\pi) \Pi$, $\nabla \pi \nabla \Pi$ on $\hat{U}_\delta$, and the contraction of $\norm{\eta(Y_t,t) - Y_t)}_2 \leq \frac{M_h}{\kappa} e^{-\alpha t}$. Finally, when $y= Y_t$, we get the following upper bound of Term (2) by applying the previous result:
    \begin{equation*}
        \abs{\text{Term (2)}} \leq \norm{\Pi \nabla A_N e_k}_F \overset{(\circ)}{\leq} \sqrt{d-m} \norm{\nabla A_N e_k}_2 = \sqrt{d-m} D_4 e^{-\alpha t} = D_5 e^{-\alpha t} \quad a.s.
    \end{equation*}
    where $(\circ)$ comes from the fact that $\norm{\Pi(y)}_F^2 = \trace{\Pi(y)} = \text{rank}(\Pi(y)) = d-m$ and $D_5 := \sqrt{d-m} D_4 $.
    Therefore, by combining results for Term (1) and Term (2), we obtain
    \begin{equation*}
        \abs{\divergence_\Sigma (\delta_k(Y_t))} \leq D_6 e^{-\alpha t} \quad \Rightarrow \quad \norm{\divergence_\Sigma (\Pi(Y_t) A_N(Y_t,t))}_2 = \sqrt{\sum_{k=1}^d (\divergence_\Sigma (\delta_k (Y_t)))^2} \leq D_7 e^{-\alpha t}
    \end{equation*}
    for $t\geq t_0$ and $D_6:=K_1 C_{A_N} + D_5 ,D_7:= \sqrt{d} D_6$. Because the final result holds without any dependency on $X_0$, the result holds almost surely without conditioning on $X_0$.
\end{proof}

\clearpage
\begin{lemma_ap}[LSI implies Talagrand inequality \citep{otto2000generalization, bobkov1999exponential, rousset2010free}]
    \label{lem:LSI implies Talagrand inequality} \ 
    For the probability measures  $\mu, \nu$ defined on a smooth complete Riemannian manifold $\Sigma$, define the $W_2^\Sigma$ distance between $\mu$ and $\nu$ in $\Sigma$ by
    \begin{equation*}
        W_2^\Sigma(\mu, \nu) := \sbra{ \inf\limits_{\pi \in \Pi(\mu, \nu)}\int_{\Sigma \times \Sigma} d_\Sigma (p,q)^2 \pi(dp,dq)}^{\frac{1}{2}}
    \end{equation*}
    where $\Pi(\mu ,\nu)$ denotes the set of coupling probability measures of $\mu, \nu$, and $d_\Sigma$ denotes the geodesic distance on $\Sigma$ so that for $p, q \in \Sigma$
    \begin{equation*}
        d_\Sigma(p,q) := \inf \mbra{ \sbra{\int_{0}^1 \norm{\dot{\gamma}(t)}_g^2 dt }^{\frac{1}{2}} \mid \gamma \in C^1 ([0,1],\Sigma), \gamma(0) = p, \gamma (1) = q} 
    \end{equation*}
    with $\norm{\cdot}_g$ being the induced metric on $\Sigma$. Then, the probability $\nu$ is said to satisfy the Talagrand inequality $(T)$ with constant $\lambda_{T}$ >0 if for all probability measures $\mu$ with $\mu \ll \nu$, it holds that
    \begin{equation*}
        W_2^\Sigma(\mu, \nu) \leq \sqrt{\frac{2}{\lambda_T} \KL^\Sigma(\mu || \nu)}.
    \end{equation*}  
    Particularly, if $\nu$ satisfies a Logarithmic Sobolev Inequality (LSI) with constant $\lambda_{\text{LSI}}$, then $\nu$ satisfies the Talagrand inequality with constant $\lambda_{\text{LSI}}$.
\end{lemma_ap}

\begin{theorem*}[Convergence result for equality-constrained OLLA] \
    Suppose assumptions \ref{asm:C1} to \ref{asm:C4} hold. Let $X_t$ be the stochastic process following the equality-constrained OLLA (\ref{eqn:equality_constrained OLLA}) and let $\rho_t, \tilde{\rho}_t$ be the law of $X_t$ and its projection $Y_t = \pi(X_t)$ on $\Sigma$ for $t \geq t_{\text{cut}}$, $t_{\text{cut}} := \max \mbra{\frac{1}{\alpha} \ln \delta, \frac{1}{\alpha} \ln C_5}$, respectively.
    Then, for all $t \geq t_{\text{cut}}$, it holds that
    \begin{align*}
        W_2(\rho_t, \rho_\Sigma) \leq \frac{M_h}{\kappa}e^{-\alpha t} + \sqrt{\frac{2}{\lambda_{\text{LSI}}} \KL^\Sigma (\tilde{\rho}_t || \rho_\Sigma)} 
    \end{align*}
    where 
   \begin{align*}
        \KL^\Sigma (\tilde{\rho}_t || \rho_\Sigma) & \leq \exp \sbra{-2 \lambda_{\text{LSI}} (t-t_{\text{cut}}) - \frac{2\lambda_{\text{LSI}}C_5}{\alpha}(e^{-\alpha t} -e^{-\alpha t_{\text{cut}}}) }[\KL^\Sigma (\tilde{\rho}_{t_{\text{cut}}} || \rho_\Sigma) \\
        &+ C_6 \int_{t_{\text{cut}}}^t \exp \sbra{ 2\lambda (s -{t_{\text{cut}}}) + \frac{2\lambda_{\text{LSI}}C_5}{\alpha} (e^{-\alpha s} - e^{-\alpha t_{\text{cut}}})}e^{-\alpha s} ds]
    \end{align*}
    In particular, if $\alpha > 2 \lambda_{\text{LSI}}$, it holds that 
    \begin{equation*}
        \KL^\Sigma (\tilde{\rho}_t || \rho_\Sigma) \leq \exp \sbra{-2 \lambda_{\text{LSI}} (t-t_{\text{cut}}) - \frac{2\lambda_{\text{LSI}}C_5}{\alpha}(e^{-\alpha t} -e^{-\alpha t_{\text{cut}}}) }[\KL^\Sigma (\tilde{\rho}_{t_{\text{cut}}} || \rho_\Sigma) + C_7]
    \end{equation*}
    for some constants $C_5=\calO \sbra{1+ \frac{C_{L_A}M_h}{\kappa} +\sbra{\frac{C_{L_A}M_h}{\kappa}}^2}, C_6, C_7 := \frac{C_6 e^{-\alpha t_{\text{cut}}}}{\alpha - 2\lambda_{\text{LSI}}} > 0$ with $C_{L_A}$ being the Lipschitz constant of $\nabla \pi(x) \Pi(x) $ on $\hat{U}_\delta(\Sigma)$.
\end{theorem*}
\begin{proof}
    First, observe that $d_2(p,q) = \norm{p-q}_2 \leq d_\Sigma (p,q), \ \forall p, q \in \Sigma$ because $\Sigma$ is the submanifold of $\bbR^d$ with Euclidean metric. Thus, $W_2 (\tilde{\rho}_t, \rho_\Sigma) \leq W_2^\Sigma (\tilde{\rho}_t, \rho_\Sigma)$ holds and we have
    \begin{equation*}
        W_2(\rho_t, \rho_\Sigma) \overset{\triangle-\text{ineq}} \leq W_2(\rho_t, \tilde{\rho}_t) + W_2 (\tilde{\rho}_t, \rho_\Sigma) \leq W_2(\rho_t, \tilde{\rho}_t) + W_2^\Sigma (\tilde{\rho}_t, \rho_\Sigma).
    \end{equation*}
    Now, recall that $\Sigma:= \mbra{x \in \bbR^d \mid h(x) = 0}$ is a smooth compact and connected Riemannian manifold. Therefore, it is complete by the Hopf-Rinow theorem. Thus, \cref{lem:LSI implies Talagrand inequality} implies 
    \begin{equation*}
        W_2^\Sigma (\tilde{\rho}_t, \rho_\Sigma) \leq \sqrt{\frac{2}{\lambda_{\text{LSI}}} \KL^\Sigma (\tilde{\rho}_t || \rho_\Sigma)} \ \Rightarrow \ W_2 (\rho_t, \rho_\Sigma) \leq W_2(\rho_t, \tilde{\rho}_t) + \sqrt{\frac{2}{\lambda_{\text{LSI}}} \KL^\Sigma (\tilde{\rho}_t || \rho_\Sigma)}.
    \end{equation*}
    Hence, we conclude the proof by borrowing the results of \cref{lem:bound of W_2_eq_only} and \cref{lem:Upper bound of KL_Sigma}
\end{proof}

%% file: appendix/Proof_of_theoretical_results_Inequality_constraint_LOLD.tex
\section{Proof of Theoretical Results - Inequality-constrained OLLA}
\label{appendix:sec:Proof of theoretical results - Proof of theoretical results - Inequality-constraint OLLA}
In this section, we analyze the non-asymptotic convergence rate of inequality-constrained OLLA. Note that \cref{prop:Construction of OLLA and its closed form SDE_app} gives
\begin{align*}
    dX_t &= -\nabla f(X_t)dt + \sqrt{2}dW_t,&& \text{if } g(X_t) <0\\ 
    dX_t &= -\Pi(X_t) \nabla f(X_t)dt -\alpha \nabla g_{I_x}^T G(X_t)^{-1} (g_{I_x} + \epsilon \eye_{I_x})dt + \sqrt{2}\Pi(X_t) \circ dW_t, && \text{otherwise} 
\end{align*}
as the closed form SDE of inequality-constrained OLLA.
\subsection{Convergence Result for Inequality-constrained OLLA}
\begin{lemma_ap}[Boundary behavior of $\rho_t$ of inequality-constrained OLLA] \ 
    \label{lem:Boundary behaviors of rho_t of inequality-constrained OLLA}
     Let $X_t$ be the stochastic process following the inequality-constrained OLLA and $\rho_t$ be the law of $X_t$. Also, denote $J_t$ be the probability current density defined by $\partial_t \rho_t = -\nabla \cdot J_t$. Then, for $t \geq t_{\text{cut}}$, $t_{\text{cut}}: = \frac{1}{\alpha} \ln \sbra{\frac{M_g+\epsilon}{\epsilon}}$, it holds that 
     \begin{equation*}
         \inner{n(x), J_t(x)} = 0, \quad \rho_t(x) = 0, \quad \forall x \in \partial\Sigma, 
     \end{equation*}
     where $n$ is the unit normal vector of $\Sigma$ and $\sigma_{\partial 
    \Sigma}$ is the surface measure of $\partial \Sigma$.
\end{lemma_ap}
\begin{proof}
    From the Fokker-Planck equation of the inequality-constrained OLLA, we know that 
    \begin{equation*}
        J_t = q \rho_t -\frac{1}{2} \nabla \cdot [Q Q^T \rho_t] = \lbra{q - \frac{1}{2} \nabla \cdot (QQ^T)}\rho_t - \frac{1}{2} QQ^T \nabla \rho_t,
    \end{equation*}
    where the last equality comes from the chain rule of the matrix divergence. Then, for each $\nabla g_j$, $j \in I_x$, observe that
    \begin{align*}
        \nabla g_j^T (\nabla \cdot (QQ^T)) &= \sum_{k=1}^d (\partial_k g_j) \lbra{ \nabla \cdot (QQ^T)}_k = \sum_{i,k=1}^d (\partial_k g_j) \partial_i(QQ^T)_{ik} \\
        &\overset{(\circ)}{=} \sum_{i=1}^d \partial_i \lbra{\underbrace{\sum_{k=1}^d (QQ^T)_{ik} \partial_k g_j}_{=0}} - \sum_{i,k=1}^d (QQ^T)_{ik} \partial_i \partial_k g_j = -\trace{\nabla^2 g_j QQ^T}
    \end{align*}
    holds for $x \in \partial \Sigma$, where $(\circ)$ holds due to the $\nabla g_k^T Q = 0$ condition of \cref{prop:Construction of OLLA and its closed form SDE_app} and $Q^2 = Q$.
    Therefore, for each $j\in I_x$, we have
    \begin{equation*}
        \inner{J_t,\nabla g_j} = \nabla g_j^T J_t = \lbra{\nabla g_j^T q + \frac{1}{2}\trace{\nabla^2 g_j QQ^T}}\rho_t \overset{(\triangle)}{=} -\alpha (g+ \epsilon)\rho_t= -\alpha \epsilon \rho_t \leq 0,
    \end{equation*}
    where $(\triangle)$ holds from the $\nabla g_j q + \frac{1}{2} \trace{\nabla^2 g_j QQ^T} + \alpha (g+\epsilon) =0 $ condition of \cref{prop:Construction of OLLA and its closed form SDE_app}. Finally, we conclude the proof by observing the following: 
    \begin{equation*}
        0 \overset{(+)}{=} \frac{d}{dt} \int_\Sigma \rho_t d\sigma_\Sigma = -\int_\Sigma \nabla \cdot J_t d\sigma_\Sigma = -\int_{\partial \Sigma} n^T J_t d \sigma_{\partial \Sigma} = \int_{\partial \Sigma} \underbrace{\alpha \epsilon \rho_t}_{\geq 0} d \sigma_{\partial \Sigma} \Rightarrow \rho_t = 0,
    \end{equation*}
    where $n$ is the outward unit normal vector of $\partial\Sigma$ and $(+)$ holds because $\text{supp}(\rho_t) \subset \Sigma$ for $t\geq t_{\text{cut}}$ implies $\int_\Sigma \rho_t d\sigma = 1$ for $t \geq t_{\text{cut}}$. 
\end{proof}
\clearpage
\begin{theorem*}[Convergence result for inequality-constrained OLLA] \ 
     Assume that $\rho_\Sigma$ satisfies the LSI condition with constant $\lambda_{\text{LSI}}$. Let $X_t$ be the stochastic process following the inequality-constrained OLLA  and let $\rho_t$ be the law of $X_t$. Then, for $t\geq t_{\text{cut}}$, $t_{\text{cut}}: = \frac{1}{\alpha} \ln \sbra{\frac{M_g+\epsilon}{\epsilon}}$, the following holds
    \begin{align*}
        W_2(\rho_t, \rho_\Sigma) \leq  \sqrt{\frac{2}{\lambda_{\text{LSI}}} \KL^\Sigma (\rho_t || \rho_\Sigma)},
    \end{align*}
    where 
    \begin{equation*}
        \KL^\Sigma(\rho_t ||\rho_\Sigma) \leq e^{-2\lambda_{\text{LSI}} (t-t_{\text{cut}})}
        \KL^\Sigma(\rho_{t_{\text{cut}}} || \rho_\Sigma).
    \end{equation*}
\end{theorem*}
\begin{proof}
    Observe that 
    \begin{align*}
        \partial_t \KL^\Sigma (\rho_t ||& \rho_\Sigma) = \int_\Sigma \partial_t \rho_t  \ln \sbra{\frac{\rho_t}{\rho_\Sigma}} d\sigma_\Sigma + \partial_t \int \rho_t d\sigma_\Sigma =  \int_\Sigma \partial \rho_t \ln \sbra{\frac{\rho_t}{\rho_\Sigma}} d\sigma_\Sigma\\
        &= \int_\Sigma (- \nabla \cdot J_t) \ln \sbra{\frac{\rho_t}{\rho_\Sigma}} d\sigma_\Sigma  = \int_\Sigma J_t \nabla \ln  \sbra{\frac{\rho_t}{\rho_\Sigma}}d\sigma_\Sigma - \int_{\partial \Sigma} \underbrace{\inner{J_t, n}}_{=0} \ln \sbra{\frac{\rho_t}{\rho_\Sigma}} d\sigma_{\partial \Sigma}\\
        &\overset{(\circ)}{=} \int_\Sigma J_t \nabla \ln \sbra{\frac{\rho_t}{\rho_\Sigma}} d\sigma_\Sigma \overset{(\triangle)}{=} -I^\Sigma(\rho_t ||\rho_\Sigma) \leq -2\lambda_{\text{LSI}} \KL_\Sigma(\rho_t || \rho_\Sigma),
    \end{align*}
    where $(\circ)$ holds due to \cref{lem:Boundary behaviors of rho_t of inequality-constrained OLLA} for $t \geq t_{\text{cut}}$ and $(\triangle)$ comes from the fact that $J_t = \nabla \ln \rho_\Sigma(x) \rho_t - \rho_t \nabla \ln \rho_t  = - \rho_t \nabla \ln \sbra{ \frac{\rho_t}{\rho_\Sigma}}$ almost everywhere in $\Sigma$.
    Therefore, we recover the upper bound of $\KL^\Sigma (\rho_t || \rho_\Sigma)$ by applying the Gronwall-type inequality as follows:
    \begin{equation*}
        \KL^\Sigma (\rho_t || \rho_\Sigma) \leq e^{-2 \lambda_{\text{LSI}}(t-t_{\text{cut}})} \KL^\Sigma (\rho_{t_{\text{cut}}} || \rho_\Sigma).
    \end{equation*}
    Also, we recall that $\Sigma:= \mbra{x \in \bbR^d \mid g(x) \leq 0}$ is a smooth compact and connected Riemannian manifold, thereby it is complete by the Hopf-Rinow theorem. Thus, \cref{lem:LSI implies Talagrand inequality} implies 
    \begin{equation*}
        W_2(\rho_t, \rho_\Sigma) \leq W_2^\Sigma (\rho_t, \rho_\Sigma) \leq \sqrt{\frac{2}{\lambda_{\text{LSI}}} \KL^\Sigma (\rho_t || \rho_\Sigma)}.
    \end{equation*}
    Hence, we conclude the proof by applying the previous upper bound of $\KL^\Sigma (\rho_t || \rho_\Sigma)$.
\end{proof}

%% file: appendix/Proof_of_theoretical_results_Mixed_constraint_OLLA.tex
\section{Proof of Theoretical Results - Mixed-constrained OLLA}
\label{appendix:sec:Proof of theoretical results - Mixed constrained OLLA}
\subsection{Upper Bound of $W_2(\rho_t, \tilde{\rho}_t)$} 
\begin{lemma_ap}[Upper bound of $W_2(\rho_t, \tilde{\rho}_t)$] \ 
    \label{lem:bound of W_2_mixed}
    Let $\rho_t$ be the law of $X_t$ which follows mixed-constrained OLLA (\ref{eqn:mixed_OLLA_premitive_form_app}) and define $t_1 := \max \mbra{\frac{1}{\alpha} \ln \sbra{\frac{M_g +\epsilon}{\epsilon}}, \frac{1}{\alpha} \ln \sbra{\frac{M_h}{\delta}}} $. For $t \geq t_1$, the law $\tilde{\rho}_t$ of $Y_t:= \pi(X_t)$ is well-defined and it holds that
    \begin{equation}
        W_2(\rho_t, \tilde{\rho}_t) \leq \frac{M_h}{\kappa} e^{-\alpha t}
    \end{equation}
\end{lemma_ap}
\begin{proof}
    For $t \geq t_1$, observe that 
    \begin{equation*}
        \norm{X_t -\pi(X_t)}_2 \leq \frac{1}{\kappa}\sbra{\norm{h(X_t)}_2  + \underbrace{\norm{g_{I_{\pi(X_t)}}(X_t)}_2}_{=0}} \leq \frac{M_h}{\kappa}  e^{-\alpha t}
    \end{equation*} by \cref{lem:regularity lemma with boundary} and \cref{lem:expoential decay of constraint functions_app}. Then, by integrating both side with respect to optimal coupling of $\rho_t$ and $\tilde{\rho}_t$, we get
    \begin{equation*}
        W_2(\rho_t, \tilde{\rho}_t) \leq \sbra{\mean{\norm{X_t-Y_t}}_2^2}^{\frac{1}{2}} \overset{(\circ)}{\leq} \mean{\norm{X_t- Y_t}_2}  \leq \frac{M_h}{\kappa} e^{-\alpha t}
    \end{equation*}
    where $(\circ)$ holds by Jensen's inequality.
\end{proof}
\subsection{Upper Bound of $\KL^\Sigma(\tilde{\rho}_t || \rho_\Sigma)$}

\begin{corollary_ap}[SDE representation of projected process from mixed-constrained OLLA] \ 
    \label{cor:SDE representation of projected process from mixed constrained OLLA}
    Let $X_t$ be the stochastic process following mixed-constrained OLLA \eqref{eqn:mixed_OLLA_premitive_form_app}. Define $t_1 := \max \mbra{\frac{1}{\alpha} \ln \sbra{\frac{M_g +\epsilon}{\epsilon}}, \frac{1}{\alpha} \ln \sbra{\frac{M_h}{\delta}}}$ and assume $Y_t = \pi(X_t) \in \text{int}(\Sigma)$ for $t \geq t_1$ via the assumption \ref{asm:M1}. Then, the projected process $Y_t$ follows the following SDE:
    \begin{align*}
        dY_t = \lbra{-\Pi(Y_t)\nabla f(Y_t)+b_N(Y_t, t)}dt +\sqrt{2}\Pi(Y_t)(I+A_N(Y_t, t)) \circ dW_t,
    \end{align*}
    where $\norm{b_N(Y_t, t)}_2 = \tilde{C}_{b_N} e^{-\alpha t}, \norm{A_N(Y_t, t)} = \tilde{C}_{A_N} e^{-\alpha t}$ for $t\geq t_1$ almost surely for some constant $\tilde{C}_{b_N}, \tilde{C}_{A_N}:= \frac{\tilde{C}_{L_A}M_h}{\kappa} > 0$ with $\tilde{C}_{L_A}$ being the Lipschitz constant of $\nabla \pi(x) \Pi(x) $ on $\hat{U}_\delta(\Sigma)$.
\end{corollary_ap}
\begin{proof}
    By applying \cref{lem:SDE representation of projected process} to the SDE (\ref{eqn:mixed_OLLA_closed_form}), it holds that
    \begin{align*}
        dY_t &= \Pi(Y_t)b(Y_t, t)dt + \sqrt{2}\Pi(Y_t)\circ dW_t + \lbra{\nabla \pi(X_t)b(X_t, t) - \nabla \pi(Y_t)b(Y_t, t)}dt\\ 
        &+ \sqrt{2}\lbra{\nabla \pi(X_t)\Pi(X_t) - \nabla \pi(Y_t) \Pi(Y_t)} \circ dW_t
    \end{align*}
    for $b(x, t) = b(x) :=  -\lbra{\Pi(x) \nabla f(x) + \alpha \nabla J(x)^T G^{-1} (x) J(x)}$. Now note that $\pi(X_t) \in \text{int}(\Sigma)$ for $t \geq t_1$ and the recovery map $\zeta$ (\cref{thm:recoverable tubular neighborhood with boundary}) is $C^1$ for $\pi(x) \in \text{int}(\Sigma)$, $x\in \hat{U}_\delta(\Sigma)$. Therefore, we can set $X_t = \zeta(Y_t, h(X_0)e^{-\alpha t})$  by using \cref{lem:Exponential decay of constraint functions} and  \cref{thm:recoverable tubular neighborhood with boundary}. Again, since $X_t, Y_t \in \hat{U}_\delta(\Sigma)$ and the closure of $\hat{U}_\delta(\Sigma)$ is compact, $\nabla \pi (x) b(x)$ and $\nabla \pi(x) \Pi(x)$ is $\tilde{C}_{L_b}, \tilde{C}_{L_A}$-Lipschitz on $\hat{U}_\delta(\Sigma)$, respectively for some $\tilde{C}_{L_b}, \tilde{C}_{L_A} > 0 $. Therefore, it holds that
    \begin{equation*}
        \norm{b_N(Y_t,t)}_2 \leq \tilde{C}_{L_b}\norm{\zeta(Y_t, h(X_0)e^{-\alpha t})- Y_t}_2 \leq \frac{\tilde{C}_{L_b} \norm{h(X_0)}_2}{\kappa}e^{-\alpha t} \leq \frac{\tilde{C}_{L_b}M_h}{\kappa}e^{-\alpha t},
    \end{equation*}
    where $b_N(Y_t,t):= \nabla \pi(\zeta(Y_t, h(X_0) e^{-\alpha t})) b(\zeta(Y_t, h(X_0) e^{-\alpha t})) - \nabla \pi(Y_t) b(Y_t)$ and the last inequality comes from \cref{lem:regularity lemma}. Similarly, we obtain the bound of $A_N(Y_t,t) := \nabla \pi(\zeta(Y_t, h(X_0)e^{-\alpha t}))\Pi(\zeta(Y_t, h(X_0)e^{-\alpha t})) - \nabla \pi(Y_t) \Pi(Y_t)$ as follows:
    \begin{equation*}
        \norm{A_N(Y_t,t)}_2 \leq \tilde{C}_{L_A}\norm{\zeta(Y_t, h(X_0)e^{-\alpha t}) - Y_t}_2 \leq \frac{\tilde{C}_{L_A} \norm{h(X_0)}_2}{\kappa} e^{-\alpha t} \leq \frac{\tilde{C}_{L_A}M_h}{\kappa}e^{-\alpha t}.
    \end{equation*}
    Hence, we complete the proof by setting $C_{b_N} := \frac{\tilde{C}_{L_b}M_h}{\kappa}, C_{A_N} := \frac{\tilde{C}_{L_A}M_h}{\kappa}$ and noting that $\nabla \pi(x) = \Pi(\pi(x)) \nabla \pi(x)$ for $ \pi(x) \in \text{int}(\Sigma), x \in \hat{U}_\delta(\Sigma)$, which implies $A_N(Y_t,t) = \Pi(Y_t) A_N(Y_t,t)$.
\end{proof} 
\clearpage

\begin{lemma_ap}[Upper bound of $\KL^\Sigma(\tilde{\rho}_t  || \rho_\Sigma)$] \ 
    \label{lem:Upper bound of KL_Sigma_mixed}
    Assume that $\rho_\Sigma$ satisfies the LSI condition with constant $\lambda_{LSI}$. Let $X_t$ be the stochastic process following mixed-constrained OLLA (\ref{eqn:mixed_OLLA_premitive_form_app}) and $\tilde{\rho}_t$ be the law of $Y_t := \pi(X_t)$ after $t \geq t_{\text{cut}}, t_{\text{cut}}:= \max \mbra{\frac{1}{\alpha} \ln \sbra{\frac{M_g +\epsilon}{\epsilon}}, \frac{1}{\alpha} \ln \sbra{\frac{M_h}{\delta}}, \frac{1}{\alpha} \ln (\tilde{C_5})}$. Suppose 
    \begin{itemize}[leftmargin=*]
        \item (Regularity of $\Sigma_p$) \quad  $\Sigma_p := \pi(\mbra{x \in \bbR^d \mid h(x) = p, g(x) \leq 0}) \subset \text{int}(\Sigma)$ for $ 0 < \norm{p}_2 \leq \delta$.
        \item (Regularity of $\partial \Sigma_p$) \quad The boundary velocity $v_p^b$ of $\partial \Sigma_p$ satisfies $\sup_{x \in \partial \Sigma_p} \norm{v_p^b}_2 \leq V \norm{p}_2^\beta$ for some $V>0, \beta >0$. Also, assume $M_\Sigma := \sup_{\norm{p}_2 < \delta} \sigma_{\partial \Sigma_p} (\partial \Sigma_p) < \infty$.
        \item (Bound on $\rho_t, \rho_\Sigma$) \quad $G_1 := \sup_{t \geq 0, x \in \Sigma} \tilde{\rho}_t <\infty$ and $0 < G_2 \leq \rho_\Sigma \leq G_3$ for $x \in \Sigma$.
    \end{itemize}
    Then, for $\alpha \neq 2\lambda_\text{LSI}$, the following non-asymptotic convergence rate of $\KL^\Sigma(\tilde{\rho}_t  || \rho_\Sigma)$ can be obtained as follows
    \begin{align*}
        \KL^\Sigma (\tilde{\rho}_t ||&\rho_\Sigma) \leq \exp \sbra{-2\lambda_\text{LSI} (t-t_{\text{cut}}) - \frac{2\lambda_\text{LSI} \tilde{C}_5}{\alpha}(e^{-\alpha t} - e^{-\alpha t_{\text{cut}}})} \times \\
        &[\KL^\Sigma (\tilde{\rho}_{t_{\text{cut}}} || \rho_\Sigma) + \int_{t_{\text{cut}}}^t \exp \sbra{2\lambda_\text{LSI}(s-t_{\text{cut}}) + \frac{2\lambda_\text{LSI} \tilde{C}_5}{\alpha} (e^{-\alpha s} - e^{-\alpha t_{\text{cut}}})} \times \\
        &\lbra{\sbra{\tilde{C}_6 + \alpha G_4 G_5 M_h }e^{-\alpha s} +G_4 V M_h^\beta e^{-\alpha \beta s}} ds].
    \end{align*}
    In particular, if $\alpha > 2 \lambda_\text{LSI}$,  $\beta \geq 1$, the inequality becomes 
    \begin{equation*}
        \KL^\Sigma (\tilde{\rho}_t ||\rho_\Sigma) \leq \exp \sbra{-2\lambda_\text{LSI} (t-t_{\text{cut}}) - \frac{2\lambda_\text{LSI} \tilde{C}_5}{\alpha}(e^{-\alpha t} - e^{-\alpha t_{\text{cut}}})} [\KL^\Sigma (\tilde{\rho}_{t_{\text{cut}}} || \rho_\Sigma) + \tilde{C}_7 + \tilde{C}_8] 
    \end{equation*}
    for some constants $\tilde{C}_5=\calO(1+\tilde{C}_{A_N}+\tilde{C}_{A_N}^2), G_4, G_5, G_6, \tilde{C}_6, \tilde{C}_7 > 0$, and $\tilde{C_7}:= (\tilde{C}_6 + \alpha G_4 G_5 M_h) \frac{e^{-\alpha t_{\text{cut}}}}{\alpha - 2\lambda_\text{LSI}}$ and $\tilde{C_8}:= (G_6 V M_h^\beta) \frac{e^{-\alpha \beta t_{\text{cut}}}}{\alpha \beta - 2\lambda_\text{LSI}}$.
\end{lemma_ap}
\begin{proof} 
    First, \cref{cor:SDE representation of projected process from mixed constrained OLLA}, and the choice of $\nabla f = -\nabla \ln \rho_\Sigma$ gives the following SDE of the projected process $Y_t$ of $X_t$ for $t\geq t_{\text{cut}}$:
    \begin{equation*}
        dY_t = \lbra{-\Pi(Y_t)\nabla f(Y_t)+b_N(Y_t, t)}dt +\sqrt{2}\Pi(Y_t)(I+A_N(Y_t,t)) \circ dW_t 
    \end{equation*}
    where $b_N(x,t):= \nabla \pi(\zeta(x, h(X_0) e^{-\alpha t})) b(\zeta(x, h(X_0) e^{-\alpha t})) - \nabla \pi(x) b(x)$, $A_N(x,t) := \nabla \pi(\zeta(x, h(X_0)e^{-\alpha t}))\Pi(\zeta(x, h(X_0)e^{-\alpha t})) - \nabla \pi(x) \Pi(x)$ for $x \in \Sigma$, conditionally on $X_0$.
    
    Hence, from \cref{cor:SDE representation of projected process from mixed constrained OLLA}, its associated Fokker-Planck equation can be written as follows:
    \begin{equation*}
        \partial_t \tilde{\rho}_t = -\divergence_\Sigma (\tilde{\rho}_t \sbra{\nabla_\Sigma \ln \rho_\Sigma + b_N}) + \sum_{k=1}^d \divergence_\Sigma \sbra{ \divergence_\Sigma (\tilde{\rho}_t (f_k + \delta_k)) (f_k +\delta_k)}.
    \end{equation*}
    
    Defining $\Sigma_t := \pi(\mbra{x \in \bbR^d \mid h(x) = h(X_0)e^{-\alpha t}, g(x) \leq 0})$ conditionally on $X_0$,  we observe
    \begin{align*}
        \partial_t \KL^\Sigma(\tilde{\rho}_t ||\rho_\Sigma) &= \partial_t \int_{\Sigma_t} \tilde{\rho}_t \ln \sbra{\frac{\tilde{\rho}_t}{\rho_\Sigma}} d\sigma_\Sigma \\
        &= \underbrace{\int_{\Sigma_t} \partial_t \tilde{\rho}_t \ln \sbra{\frac{\tilde{\rho}_t}{\rho_\Sigma}} d\sigma_\Sigma}_{\text{Term (1)}} + \underbrace{\int_{\partial \Sigma_t} \tilde{\rho}_t\lbra{ \ln \sbra{\frac{\tilde{\rho}_t}{\rho_\Sigma}} - 1} \inner{v_t^b, n_t} d\sigma_{\partial \Sigma_t} }_{\text{Term (2)}}+ \underbrace{\partial_t \int_{\Sigma_t} \tilde{\rho}_t d\sigma_\Sigma}_{=0} 
    \end{align*}
    where the last equality holds from the Leibniz integral rule with $v_t^b$ being the velocity vector of the boundary of $\Sigma_t$ and $n_t$ being the outward unit normal vector of $\partial\Sigma_t$. Therefore, the expression of $\partial_t \tilde{\rho}_t$ implies that Term (1) becomes
    \begin{align*}
        \text{Term (1)} &= \underbrace{\int_{\Sigma_t} \inner{\lbra{\tilde{\rho}_t (\nabla_\Sigma \ln \rho_\Sigma + b_N) -\sum_{k=1}^d \divergence_\Sigma (\tilde{\rho}_t (f_k + \delta_k)) (f_k+\delta_k)}, \nabla_\Sigma \ln \sbra{\frac{\tilde{\rho}_t}{\rho_\Sigma}}} d\sigma_\Sigma}_{\text{Term (1-1)}} \\
        &- \underbrace{\int_{\partial \Sigma_t} \inner{\lbra{\tilde{\rho}_t (\nabla_\Sigma \ln \rho_\Sigma + b_N) -\sum_{k=1}^d \divergence_\Sigma (\tilde{\rho}_t (f_k + \delta_k)) (f_k+\delta_k)} \ln \sbra{\frac{\tilde{\rho}_t}{\rho_\Sigma}}, n_t} d\sigma_{\partial \Sigma_t}}_{\text{Term (1-2)}}
    \end{align*}
    by integration by parts, where $f_k := \Pi e_k, \delta_k := \Pi A_N \delta_k$. Now, we observe that Term (1-1) can be bounded as 
    \begin{equation*}
        \text{Term (1-1)} \leq -(1-\tilde{C}_5 e^{-\alpha t}) I^{\Sigma} (\tilde{\rho}_t || \rho_\Sigma) + \tilde{C}_6 e^{-\alpha t} 
    \end{equation*}
    following the same proof of \cref{lem:Upper bound of KL_Sigma} with different constants $\tilde{C}_5 = \calO(1 + \tilde{C}_{A_N} + \tilde{C}_{A_N}^2), \tilde{C}_6 > 0$ (note that we ignored the integrand at $\partial \Sigma$ which is measure zero with respect to $d \sigma_\Sigma$). 
    
    For the analysis of Term (1-2), we first observe the following fact: 
    \begin{align*}
        \int_{\partial \Sigma_t} \abs{\tilde{\rho}_t \ln \sbra{\frac{\tilde{\rho}_t}{\rho_\Sigma}}} d\sigma_{\partial \Sigma_t} & \leq G_3 \max \mbra{\frac{1}{e}, \abs{\frac{G_1}{G_2} \ln \sbra{\frac{G_1}{G_2}}}} \sigma_{\partial \Sigma_t} (\partial \Sigma_t) \\
        & \leq G_3 M_\Sigma \max \mbra{\frac{1}{e}, \abs{\frac{G_1}{G_2} \ln \sbra{\frac{G_1}{G_2}}}}:= G_4
    \end{align*}
    from the assumptions of $G_1 := \sup_{t\geq0 , x \in \Sigma}\tilde{\rho}_t < \infty$, $0 < G_2 \leq  \rho_\Sigma \leq G_3$, and the regularity of $\partial \Sigma_t$ such that $ \sup_{t \geq t_{\text{cut}} }\sigma_{\partial \Sigma_t}(\partial \Sigma_t) \leq M_\Sigma < \infty$.

    Next, from \cref{cor:SDE representation of projected process from mixed constrained OLLA},  we note that the following holds conditionally on $X_0$:
    \begin{equation*}
        f_k(x) + \delta_k(x) = \Pi(x)(I + A_N(t, x)) e_k = \nabla \pi (\zeta(x, h(X_0)e^{-\alpha t})) \Pi(\zeta (x, h(X_0) e^{-\alpha t}))e_k.
    \end{equation*}
   Because $\Pi(\zeta(x, h(X_0)e^{-\alpha t}))e_k $ is a tangent vector on $\mbra{x \in \bbR^d \mid h(x) = h(X_0)e^{-\alpha t}, g(x) \leq 0}$, $f_k(x) + \delta_k(x) = \nabla \pi(\zeta(x, h(X_0)e^{-\alpha t})) \Pi(\zeta(x, h(X_0) e^{-\alpha t}))e_k$ becomes a tangent vector of $\Sigma_t$. Similarly, it also becomes a tangent vector of $\partial \Sigma_t$ on the boundary because $\Pi$ is the orthogonal projector induced by $h$ and active $g$. Hence, $\inner{f_k+\delta_k, n_t} = 0$ holds, where $n_t$ is the outward unit normal vector of $\partial\Sigma_t$. 
   
   Therefore, Term (1-2) becomes
    \begin{align*}
        \abs{\text{Term (1-2)}} \overset{(1)}{\leq} \alpha G_5M_h e^{-\alpha t} \int_{\partial \Sigma_t} \abs{\tilde{\rho}_t \ln \sbra{\frac{\tilde{\rho}_t}{\rho_\Sigma}}} d\sigma_{\partial \Sigma_t} \leq \alpha G_4 G_5 M_h e^{-\alpha t},
    \end{align*}
    where (1) holds because \cref{cor:SDE representation of projected process from mixed constrained OLLA} gives
    \begin{equation*}
        \nabla_\Sigma \ln \rho_\Sigma(x) + b_N(x,t) = \nabla \pi(\zeta(x, h(X_0)e^{-\alpha t}))b(\zeta(t, h(X_0)e^{-\alpha t}))
    \end{equation*}
    with $b(x) := -\lbra{\Pi(x) \nabla f(x) + \alpha \nabla J(x)^T G^{-1}(x) J(x)}$ and, therefore,
    \begin{equation*}
        \abs{\inner{\nabla_\Sigma \ln \rho_\Sigma(x)  +  b_N(x, t), n_t(x)}} \leq \alpha G_5 \norm{J(\zeta(x, h(X_0)e^{-\alpha t}))}_2 \leq \alpha G_5 M_h e^{-\alpha t}
    \end{equation*}
    for $x \in \partial \Sigma_t $ with $G_5 :=  \sup_{x \in \Sigma, \norm{p}_2 < \delta} \norm{\nabla \pi(\zeta(x, p)) \nabla J(\zeta(x,p))^T G^{-1}(\zeta(x,p))}_2 <\infty$.

    Also, similarly, Term (2) is bounded as follows:
    \begin{align*}
        \abs{\text{Term (2)}} \leq \sup_{x \in \partial\Sigma_t}\norm{v_t^b}_2 \int_{\partial \Sigma_t} \sbra{\abs{\tilde{\rho}_t \ln \sbra{\frac{\tilde{\rho}_t}{\rho_\Sigma}}} + \abs{\tilde{\rho}_t}} d\sigma_{\partial \Sigma_t} \leq G_6 V M_h^\beta e^{-\alpha \beta t}
    \end{align*}
    with $G_6 := G_4 + G_1 M_\Sigma$.
    Therefore, combining the results with the LSI condition gives the following inequality:
    \begin{align*}
        \partial_t \KL^\Sigma (\tilde{\rho}_t ||\rho_\Sigma) &\leq -2\lambda_\text{LSI}(1- \tilde{C}_5 e^{-\alpha t})\KL^\Sigma (\tilde{\rho}_t || \rho_\Sigma) +\sbra{\tilde{C}_6 +\alpha G_4 G_5 M_h}e^{-\alpha t} + G_6 V M_h^\beta e^{-\alpha \beta t}
    \end{align*}
    where the last inequality comes from the LSI condition. Finally, applying the Gr\"onwall-type inequality recovers the following inequality:
    \begin{align*}
        \KL^\Sigma (\tilde{\rho}_t ||\rho_\Sigma) &\leq \exp \sbra{-2\lambda_\text{LSI} (t-t_{\text{cut}}) - \frac{2\lambda_\text{LSI} \tilde{C}_5}{\alpha}(e^{-\alpha t} - e^{-\alpha t_{\text{cut}}})} [\KL^\Sigma (\tilde{\rho}_{t_{\text{cut}}} || \rho_\Sigma) \\
        &+ \int_{t_{\text{cut}}}^t \exp \sbra{2\lambda_\text{LSI}(s-t_{\text{cut}}) + \frac{2\lambda_\text{LSI} \tilde{C}_5}{\alpha} (e^{-\alpha s} - e^{-\alpha t_{\text{cut}}})} \times
        \\ &\lbra{\sbra{\tilde{C}_6 + \alpha G_4 G_5 M_h }e^{-\alpha s} +G_6 V M_h^\beta e^{-\alpha \beta s}} ds].
    \end{align*}
    Also, similarly in the last argument of \cref{lem:Upper bound of KL_Sigma}, we observe that if $\alpha > 2\lambda_\text{LSI}$ and $\beta \geq 1$
    \begin{align*}
        \int_{t_{\text{cut}}}^\infty \mkern-10mu\exp \sbra{2 \lambda_\text{LSI} (s-t_{\text{cut}}) + \frac{2\lambda_\text{LSI} \tilde{C}_5}{\alpha} (e^{-\alpha s} - e^{-\alpha t_{\text{cut}}})}e^{-\alpha s} ds &\leq \mkern-5mu \int_{t_{\text{cut}}}^\infty \exp \sbra{2\lambda_\text{LSI} (s-t_{\text{cut}})} e^{-\alpha s} ds \\
        & \leq \frac{e^{-\alpha t_{\text{cut}}}}{\alpha -2 \lambda_\text{LSI}} <\infty
    \end{align*}
    and
    \begin{align*}
        \int_{t_{\text{cut}}}^\infty \exp \sbra{2 \lambda_\text{LSI} (s-t_{\text{cut}}) + \frac{2\lambda_\text{LSI} \tilde{C}_5}{\alpha} (e^{-\alpha s} - e^{-\alpha t_{\text{cut}}})}e^{-\alpha \beta s} ds &\leq \frac{e^{-\alpha \beta t_{\text{cut}}}}{\alpha \beta - 2\lambda_\text{LSI}} <\infty.
    \end{align*}
    Therefore, there exists $\tilde{C_7}, \tilde{C_8} < \infty$ such that 
    \begin{align*}
        \KL^\Sigma (\tilde{\rho}_t ||\rho_\Sigma) \leq \exp \sbra{-2\lambda_\text{LSI} (t-t_{\text{cut}}) - \frac{2\lambda_\text{LSI} \tilde{C}_5}{\alpha}(e^{-\alpha t} - e^{-\alpha t_{\text{cut}}})} [\KL^\Sigma (\tilde{\rho}_{t_{\text{cut}}} || \rho_\Sigma) + \tilde{C}_7 + \tilde{C}_8],
    \end{align*}
    where $\tilde{C_7}:= (\tilde{C}_6 + \alpha G_4 G_5 M_h) \frac{e^{-\alpha t_{\text{cut}}}}{\alpha - 2\lambda_\text{LSI}}$ and $\tilde{C_8}:= (G_6 V M_h^\beta) \frac{e^{-\alpha \beta t_{\text{cut}}}}{\alpha \beta - 2\lambda_\text{LSI}}$.
\end{proof}

\clearpage

\begin{theorem*}[Convergence result for mixed-constrained OLLA]
    Assume that $\rho_\Sigma$ satisfies the LSI condition with constant $\lambda_\text{LSI}$. Let $X_t$ be the stochastic process following mixed-constrained OLLA (\ref{eqn:mixed_OLLA_closed_form}) and $\tilde{\rho}_t$ be the law of $Y_t := \pi(X_t)$ after $t \geq t_{\text{cut}}, t_{\text{cut}}:= \max \mbra{\frac{1}{\alpha} \ln \sbra{\frac{M_g +\epsilon}{\epsilon}}, \frac{1}{\alpha} \ln \sbra{\frac{M_h}{\delta}}, \frac{1}{\alpha} \ln (\tilde{C_5})}$. Suppose 
    \begin{itemize}[leftmargin=*]
        \item (Regularity of $\Sigma_p$) \quad  $\Sigma_p := \pi(\mbra{x \in \bbR^d \mid h(x) = p, g(x) \leq 0}) \subset \text{int}(\Sigma)$ for $ 0 < \norm{p}_2 \leq \delta$.
        \item (Regularity of $\partial \Sigma_p$) \quad The boundary velocity $v_p^b$ of $\partial \Sigma_p$ satisfies $\sup_{x \in \partial \Sigma_p} \norm{v_p^b}_2 \leq V \norm{p}_2^\beta$ for some $V>0, \beta >0$. Also, assume $M_\Sigma := \sup_{\norm{p}_2 < \delta} \sigma_{\partial \Sigma_p} (\partial \Sigma_p) < \infty$.
        \item (Bound on $\rho_t, \rho_\Sigma$) \quad $G_1 := \sup_{t \geq 0, x \in \Sigma} \tilde{\rho}_t <\infty$ and $0 < G_2 \leq \rho_\Sigma \leq G_3$ for $x \in \Sigma$.
    \end{itemize}
    Then, for $\alpha \neq 2\lambda_\text{LSI}$ the following non-asymptotic convergence rate of $W_2(\rho_t, \rho_\Sigma)$ can be obtained as follows
    \begin{align*}
        W_2(\rho_t, \rho_\Sigma) \leq \frac{M_h}{\kappa}e^{-\alpha t} + \sqrt{\frac{2}{\lambda_\text{LSI}} \KL^\Sigma (\tilde{\rho}_t || \rho_\Sigma)} 
    \end{align*}
    where 
    \begin{align*}
        &\KL^\Sigma (\tilde{\rho}_t ||\rho_\Sigma) \leq \exp \sbra{-2\lambda_\text{LSI} (t-t_{\text{cut}}) - \frac{2\lambda_\text{LSI} \tilde{C}_5}{\alpha}(e^{-\alpha t} - e^{-\alpha t_{\text{cut}}})} \times \\
        &[\KL^\Sigma (\tilde{\rho}_{t_{\text{cut}}} || \rho_\Sigma) + \int_{t_{\text{cut}}}^t \exp \sbra{2\lambda_\text{LSI}(s-t_{\text{cut}}) + \frac{2\lambda_\text{LSI} \tilde{C}_5}{\alpha} (e^{-\alpha s} - e^{-\alpha t_{\text{cut}}})} \times \\
        &\lbra{\sbra{\tilde{C}_6 + \alpha G_4 G_5 M_h }e^{-\alpha s} +G_6 V M_h^\beta e^{-\alpha \beta s}} ds] 
    \end{align*}
    In particular, if $\alpha > 2 \lambda_\text{LSI}$ and $\beta \geq 1$, the previous bound simplifies to 
    \begin{equation*}
        \KL^\Sigma (\tilde{\rho}_t ||\rho_\Sigma) \leq \exp \sbra{-2\lambda_\text{LSI} (t-t_{\text{cut}}) - \frac{2\lambda_\text{LSI} \tilde{C}_5}{\alpha}(e^{-\alpha t} - e^{-\alpha t_{\text{cut}}})} [\KL^\Sigma (\tilde{\rho}_{t_{\text{cut}}} || \rho_\Sigma) + \tilde{C}_7 + \tilde{C}_8] 
    \end{equation*}
    for some constants $\tilde{C}_5=\calO \sbra{1+ \frac{\tilde{C}_{L_A}M_h}{\kappa} +\sbra{\frac{\tilde{C}_{L_A}M_h}{\kappa}}^2}, G_4, G_5, G_6, \tilde{C}_6, \tilde{C}_7 > 0$, and $\tilde{C_7}:= (\tilde{C}_6 + \alpha G_4 G_5 M_h) \frac{e^{-\alpha t_{\text{cut}}}}{\alpha - 2\lambda_\text{LSI}}$ and $\tilde{C_8}:= (G_6 V M_h^\beta) \frac{e^{-\alpha \beta t_{\text{cut}}}}{\alpha \beta - 2\lambda_\text{LSI}}$, with $\tilde{C}_{L_A}$ being the Lipschitz constant of $\nabla \pi(x) \Pi(x) $ on $\hat{U}_\delta(\Sigma)$.
\end{theorem*}
\begin{proof}
    We note that $\tilde{\rho}_t(x) = 0$ on $\partial \Sigma$ for $t \geq t_{\text{cut}}$ holds from the regularity of the $\Sigma_p$ assumption. Therefore, by the same approach in \cref{thm:Convergence result for equality-constrained OLLA} and \cref{lem:LSI implies Talagrand inequality}, it holds that
    \begin{equation*}
        W_2^\Sigma (\tilde{\rho}_t, \rho_\Sigma) \leq \sqrt{\frac{2}{\lambda_\text{LSI}} \KL^\Sigma (\tilde{\rho}_t || \rho_\Sigma)} \ \Rightarrow \ W_2 (\rho_t, \rho_\Sigma) \leq W_2(\rho_t, \tilde{\rho}_t) + \sqrt{\frac{2}{\lambda_\text{LSI}} \KL^\Sigma (\tilde{\rho}_t || \rho_\Sigma)}.
    \end{equation*}
    Therefore, we conclude the proof by combining the results of \cref{lem:bound of W_2_mixed} and \cref{lem:Upper bound of KL_Sigma_mixed}.
\end{proof}


%% file: appendix/Experiment_setting.tex
\section{Experiment Settings and Supplementary Results}
\textbf{Settings.}\quad The first two experiments were executed on a desktop with an AMD Ryzen 9 7900X CPU (12 cores) with 32 GB RAM. Runs were implemented in WSL2 (Ubuntu) environment (CPU-only), using the Python and the PyTorch \citep{paszke2019pytorch} framework.

\subsection{Experiment Settings and Supplementary Results for Synthetic 2D Examples}
\label{app:Experiment setup for synthetic 2D examples}

\textbf{Experiment Settings.}\quad  In this experiment, we compare four samplers (OLLA, OLLA-H, CLangevin, CHMC) on the following synthetic 2D examples:

\begin{enumerate}
    \item \textbf{(Star)} \quad a star-shaped equality manifold with uniform density:
    \begin{equation*}
        f(x) = 0, \ h(x) = \sqrt{x_1^2 + x_2^2} - (1.5 + 0.3 \cos(5\theta)), \  \theta = \arctan2 (x_2,x_1).
    \end{equation*}
    \item \textbf{(Two Lobes)} \quad a two-lobe inequality manifold (from \citep{zhang2024functional}) with uniform density:
    \begin{equation*}
        f(x) = 0, \ g(x) = -\ln q(x) - 2, \ q(x) = \frac{e^{-2(x_1-3)^2+ e^{-2(x_1+3)^2}}}{e^{2(\norm{x}_2 - 3)^2}}.
    \end{equation*}
    \item \textbf{(Quadratic Poly)} \quad a quadratic curve defined by mixed polynomial equality and inequality under a standard Gaussian target:
    \begin{equation*}
        f(x) = \frac{1}{2}\norm{x}_2^2, \ h(x) = x_1^4x_2^2 + x_1^2 + x_2 -1, \ g(x) = x_1^3-x_2^3 - 1.
    \end{equation*}
    \item \textbf{(Mixture Gaussian)} \quad a nine-Gaussian mixture restricted by a seven-lobe manifold:
    \begin{align*}
        f(x) = -\ln \sbra{\sum_{i=1}^9 \exp \sbra{-5 \norm{x-c_i}_2^2}}, \ h(x) = \sqrt{x_1^2 + x_2^2} - (3 + \cos (7 \theta)).
    \end{align*}
    \begin{equation*}
        g(x) = (x_1-2)^2 - 5x_1x_2^3 + 0.5 x_2^5 -40, \ \theta = \arctan2 (x_2,x_1).
    \end{equation*}
    with $\mbra{c_i}_{i=1}^9 = \mbra{-2, 0, 2}^2$.
\end{enumerate}
where $x = [x_1, x_2]^T \in \bbR^2$. For each 2D example, we run $200$ independent chains for $K = 5000$ steps each. From each chain, we retain only the state at step $K$, yielding 200 samples per sampler.

\begin{wrapfigure}{r}{0.5\textwidth}
\vspace{-12pt}
\captionof{table}{Hyperparameter settings for 2D synthetic examples ($\Delta t = 5 \times 10^{-4}$)}
\label{tab:hyperparams_2d}
\centering
\begin{tabular}{l l}
\toprule
\textbf{Method} & \textbf{Hyperparameters} \\
\midrule
OLLA      & $\alpha = 200,\;\epsilon = 1$ \\
OLLA–H    & $\alpha = 200,\;\epsilon = 1,\;N = 5$ \\
CLangevin & $L = 3,\;\tau = 10^{-4}, \lambda = \mbra{1, 0.1}$ \\
CHMC      & $\gamma = 1,\;L = 3,\;\tau = 10^{-4} ,\lambda = 0$ \\
CGHMC     & $\gamma = 1,\;L = 3,\;\tau = 10^{-4} , \lambda = 0$ \\
\bottomrule
\end{tabular}
\vspace{-12pt}
\end{wrapfigure}

To provide a fixed target distribution for distance calculation, we generate $200$ samples using CGHMC. This reference is held constant across all comparisons. For each sampler, we compute $W_2^2$, energy distance as well as the mean constraint violations $\bbE[ \abs{h(x)}]$ and $\bbE[\max g(x)^+]$, over the 200 samples. 

The hyperparameter setup is provided in \cref{tab:hyperparams_2d}. To mitigate ill-conditioning in the Newton solver of CLangevin, we add Tikhonov regularization with Tikhonov matrix $\Gamma = \sqrt{\lambda} I$, $\lambda = 1.0$ for the Mixture Gaussian example, and $\lambda = 0.1$ for the other three. For the CLangevin, CHMC, CGHMC, $X_0$ is initialized exactly on $\Sigma$ to ensure the stability of algorithms while OLLA and OLLA-H have noisy initialization $X_0 = Y_0 + \calN(0,I), Y_0 \in \Sigma$ and $X_t$ progressively approaches to $\Sigma$ by the landing mechanism. 

\begin{remark} \quad 
    The results shown in \cref{fig:scatter plots of 2D synthetic examples} and \cref{fig:Convergence diagnostics on the 9 Gaussian mixture on 7 Lobe manifold} were obtained under a different setup described above: a larger step size $\Delta t$ was used; all methods were initialized from the same initialization point $X_0 \in \Sigma$, rather than via random sampling near $\Sigma$; and the regularization parameter $\lambda$ in CLangevin was increased for improved numerical stability under large step size.
\end{remark}

\begin{figure}
    \centering
    \includegraphics[width=1\linewidth]{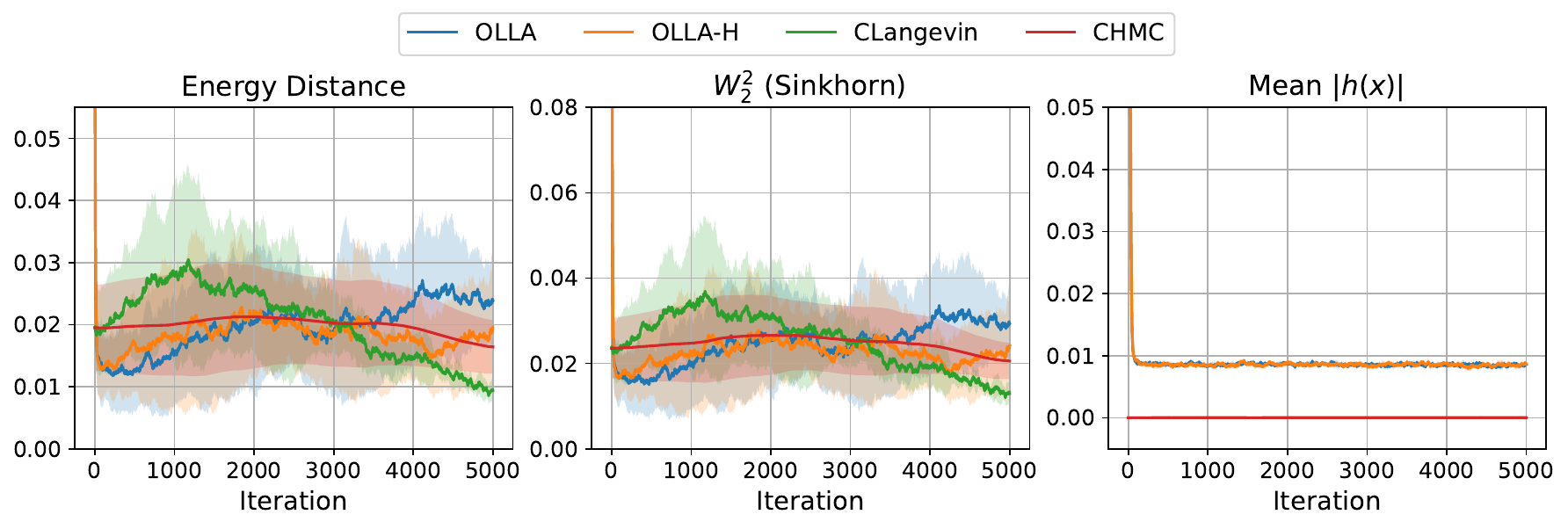}
    \caption{Convergence diagnostics on the Star example (equality-only case). (1) energy distance to CGHMC samples (left),  (2) $W_2^2$ distance to CGHMC samples (center), and  (3) mean of $\abs{h(x)}$ (right)}
    \label{fig:Convergence diagnostics on the Star example_app}
\end{figure}

\begin{figure}
    \centering
    \includegraphics[width=1\linewidth]{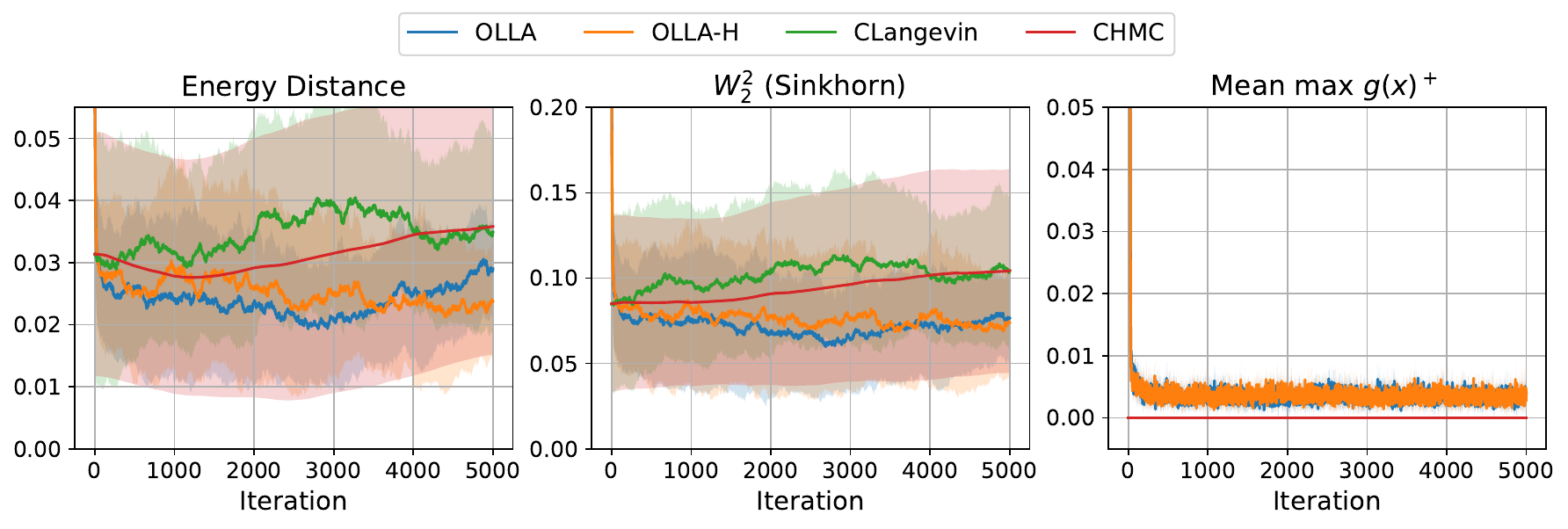}
    \caption{Convergence diagnostics on the Two Lobes example (inequality-only case). (1) energy distance to CGHMC samples (left), (2) $W_2^2$ distance to CGHMC samples (center), and (3) mean of $\max g(x)^+$ (right)}
    \label{fig:Convergence diagnostics on the Two Lobes example_app}
\end{figure}

\begin{figure}
    \centering
    \includegraphics[width=0.7\linewidth]{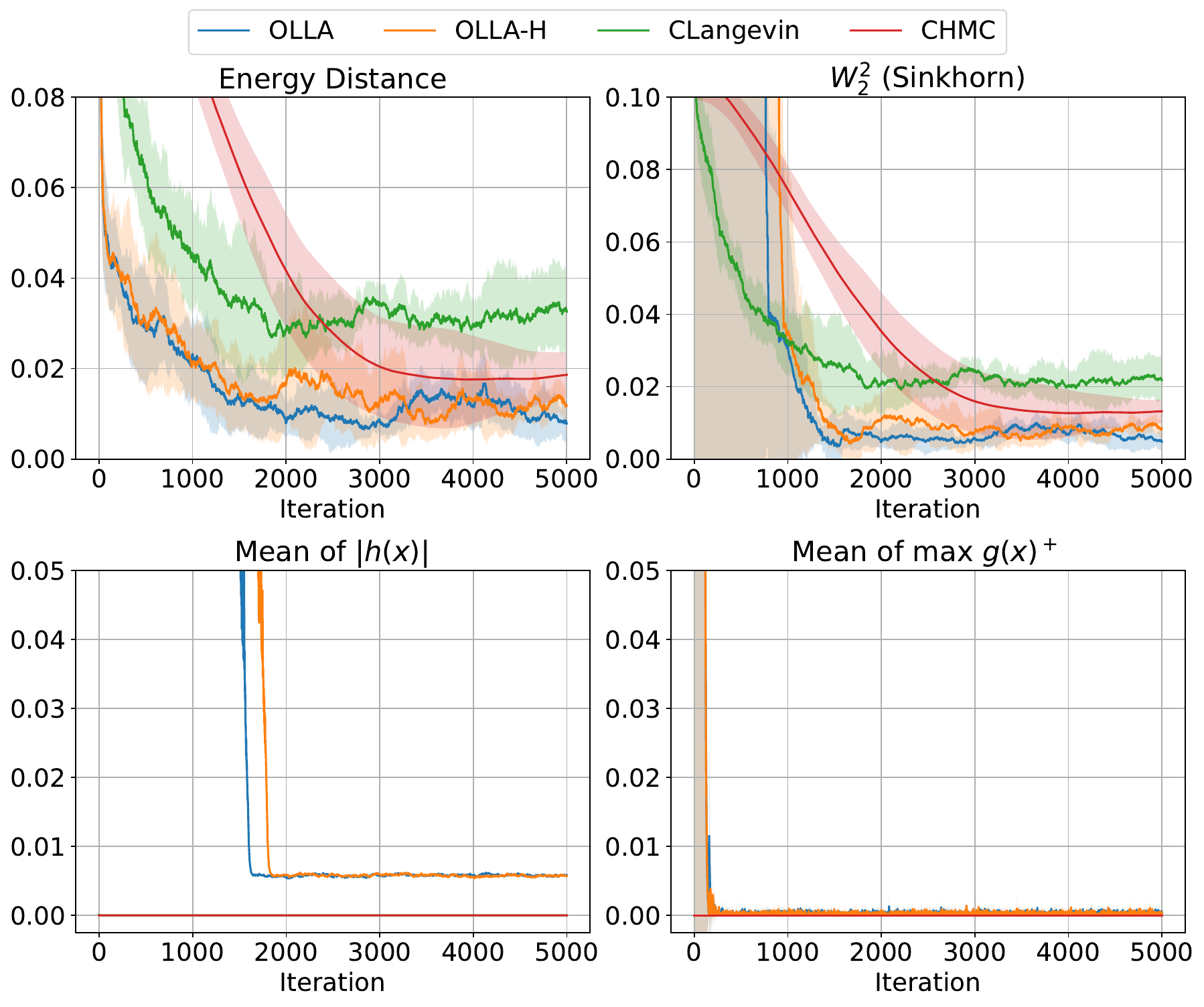}
    \caption{Convergence diagnostics on the Quadratic Poly example (mixed-case). (1) energy distance to CGHMC samples (top left) and (2) $W_2^2$ distance to CGHMC samples (top right). (3) mean of $|h(x)|$ (bottom left) and (4) mean of $\max g(x)^+$ (bottom right)}
    \label{fig:Convergence diagnostics on the Quadratic poly example_app}
\end{figure}
\clearpage

\begin{figure}
    \centering
    \includegraphics[width=0.7\linewidth]{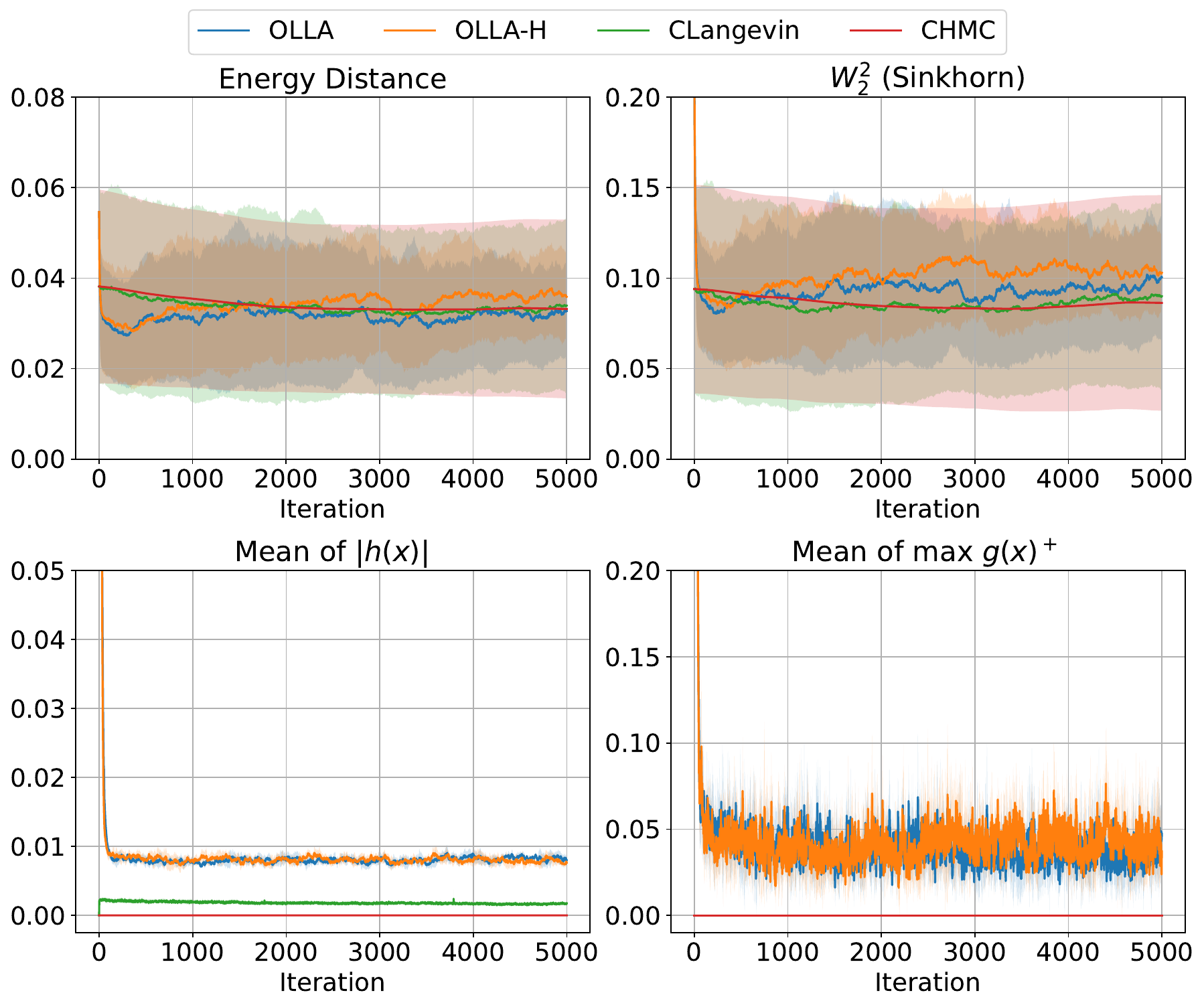}
    \caption{Convergence diagnostics on the Mixture Gaussian example (mixed-case). (1) energy distance to CGHMC samples (top left) and (2) $W_2^2$ distance to CGHMC samples (top right). (3) mean of $|h(x)|$ (bottom left) and (4) mean of $\max g(x)^+$ (bottom right)}
    \label{fig:Convergence diagnostics on the Mixture Gauss example_app}
\end{figure}

\textbf{Supplementary Results - Sampling Accuracy \& Constraint Violation.} \quad In addition to the Mixture Gaussian example (\cref{fig:Convergence diagnostics on the 9 Gaussian mixture on 7 Lobe manifold}), we evaluated OLLA and OLLA-H on three further 2D geometries to verify that the distributional accuracy observed in \cref{fig:Convergence diagnostics on the 9 Gaussian mixture on 7 Lobe manifold} of the main text generalize to other settings. Across all three additional examples, OLLA and OLLA-H closely match the convergence behavior of CLangevin and CHMC in both energy distance and $W_2^2$ metrics, while maintaining constraint violations at low levels without requiring explicit projection steps. These trends are illustrated in \cref{fig:Convergence diagnostics on the Star example_app} (Star), \cref{fig:Convergence diagnostics on the Two Lobes example_app} (Two Lobes), \cref{fig:Convergence diagnostics on the Quadratic poly example_app} (Quadratic Poly), and \cref{fig:Convergence diagnostics on the Mixture Gauss example_app} (Mixture Gaussian under the hyperparameter setup of \cref{tab:hyperparams_2d}). These additional experiments confirm that our landing-based sampler can provide relatively accurate, constraint-respecting samples across diverse manifold geometries.

\textbf{Supplementary Results - Effect of Hyperparameters $\alpha$ and $\epsilon$.} \quad In \cref{tab:effect-alpha_app}, we report how varying the landing rate $\alpha$ (with $\epsilon=1$) affects sampling performance on the Mixture Gaussian example. For both OLLA and OLLA--H, increasing $\alpha$ from $1$ to $500$ leads to consistent reductions in energy distance, $W_2^2$, and the average constraint violation $\bbE[|h(x)|]$. However, setting $\alpha$ too large causes sampling failures and numerical errors (e.g.\ $\alpha=5000$).

Similarly, \cref{tab:effect-epsilon_app} examines the effect of boundary repulsion rate $\epsilon$ (with $\alpha=100$). Across the full range of $\epsilon$ tested, sampling accuracy—as measured by both $W_2^2$ and energy distance—remains essentially invariant, exhibiting smaller variation than when $\alpha$ is changed. In contrast, inequality‐constraint enforcement steadily improves as $\epsilon$ increases: the average violation $\bbE[g(x)^+]$ declines monotonically, reflecting stronger repulsion. Only when $\epsilon$ becomes exceedingly large does one observe a degradation in equality‐constraint enforcement and occasional numerical instabilities, mirroring the breakdown seen at an extreme $\alpha$ value.

\begin{table}[t]
  \centering
  \caption{Effect of $\alpha$ on energy distance, $W_2^2$, $\bbE[|h(x)|]$, and $\bbE[g(x)^+]$ on the Mixture Gaussian example with $\epsilon = 1$ (top: OLLA, bottom: OLLA-H)}
  \label{tab:effect-alpha_app}
  \vspace{5pt}
  \begin{tabular}{ccccc}
    \toprule
    $\alpha$ & energy distance & $W_2^2$ & $\bbE[|h(x)|]$ & $\bbE[g(x)^+]$ \\
    \midrule
    1    & $0.121\text{\scriptsize$\pm$0.025}$ & $0.363\text{\scriptsize$\pm$0.064}$ & $0.682\text{\scriptsize$\pm$0.017}$ & $1.113\text{\scriptsize$\pm$0.351}$ \\
    10   & $0.066\text{\scriptsize$\pm$0.019}$ & $0.200\text{\scriptsize$\pm$0.035}$ & $0.130\text{\scriptsize$\pm$0.001}$ & $0.234\text{\scriptsize$\pm$0.032}$ \\
    100  & $0.053\text{\scriptsize$\pm$0.016}$ & $0.159\text{\scriptsize$\pm$0.032}$ & $0.017\text{\scriptsize$\pm$0.001}$ & $0.045\text{\scriptsize$\pm$0.007}$ \\
    200  & $0.040\text{\scriptsize$\pm$0.012}$ & $0.121\text{\scriptsize$\pm$0.019}$ & $0.008\text{\scriptsize$\pm$0.001}$ & $0.054\text{\scriptsize$\pm$0.009}$ \\
    500  & $0.033\text{\scriptsize$\pm$0.011}$ & $0.104\text{\scriptsize$\pm$0.020}$ & $0.004\text{\scriptsize$\pm$0.000}$ & $0.021\text{\scriptsize$\pm$0.018}$ \\
    700  & $0.044\text{\scriptsize$\pm$0.012}$ & $0.132\text{\scriptsize$\pm$0.019}$ & $0.003\text{\scriptsize$\pm$0.000}$ & $0.016\text{\scriptsize$\pm$0.011}$ \\
    5000 & \multicolumn{4}{c}{NaN (results unavailable)} \\

    \midrule
    1    & $0.104\text{\scriptsize$\pm$0.019}$ & $0.333\text{\scriptsize$\pm$0.077}$ & $0.643\text{\scriptsize$\pm$0.032}$ & $1.102\text{\scriptsize$\pm$0.337}$ \\
    10   & $0.059\text{\scriptsize$\pm$0.018}$ & $0.181\text{\scriptsize$\pm$0.038}$ & $0.129\text{\scriptsize$\pm$0.011}$ & $0.189\text{\scriptsize$\pm$0.023}$ \\
    100  & $0.052\text{\scriptsize$\pm$0.017}$ & $0.156\text{\scriptsize$\pm$0.026}$ & $0.015\text{\scriptsize$\pm$0.001}$ & $0.055\text{\scriptsize$\pm$0.024}$ \\
    200  & $0.052\text{\scriptsize$\pm$0.018}$ & $0.153\text{\scriptsize$\pm$0.037}$ & $0.009\text{\scriptsize$\pm$0.000}$ & $0.039\text{\scriptsize$\pm$0.015}$ \\
    500  & $0.041\text{\scriptsize$\pm$0.013}$ & $0.124\text{\scriptsize$\pm$0.027}$ & $0.004\text{\scriptsize$\pm$0.000}$ & $0.013\text{\scriptsize$\pm$0.009}$ \\
    700  & $0.037\text{\scriptsize$\pm$0.011}$ & $0.110\text{\scriptsize$\pm$0.019}$ & $0.002\text{\scriptsize$\pm$0.000}$ & $0.007\text{\scriptsize$\pm$0.006}$ \\
    5000 & \multicolumn{4}{c}{NaN (results unavailable)} \\
    \bottomrule
  \end{tabular}
\end{table}

\begin{table}[t]
  \centering
  \caption{Effect of $\epsilon$ on energy distance, $W_2^2$, $\bbE[|h(x)|]$, and $\bbE[g(x)^+]$ on the Mixture Gaussian example with $\alpha = 100$ (top: OLLA, bottom: OLLA-H)}
  \vspace{5pt}
  \label{tab:effect-epsilon_app}
  \begin{tabular}{ccccc}
    \toprule
    $\epsilon$ 
      & energy distance 
      & $W_2^2$ 
      & $\bbE[|h(x)|]$ 
      & $\bbE[g(x)^+]$ \\
    \midrule
    0.1   & $0.048\text{\scriptsize$\pm$0.014}$ & $0.151\text{\scriptsize$\pm$0.026}$ & $0.014\text{\scriptsize$\pm$0.001}$ & $0.082\text{\scriptsize$\pm$0.017}$ \\
    1   & $0.033\text{\scriptsize$\pm$0.004}$ & $0.108\text{\scriptsize$\pm$0.011}$ & $0.017\text{\scriptsize$\pm$0.002}$ & $0.067\text{\scriptsize$\pm$0.027}$ \\
    5   & $0.040\text{\scriptsize$\pm$0.006}$ & $0.123\text{\scriptsize$\pm$0.018}$ & $0.016\text{\scriptsize$\pm$0.001}$ & $0.040\text{\scriptsize$\pm$0.015}$ \\
    10  & $0.036\text{\scriptsize$\pm$0.016}$ & $0.112\text{\scriptsize$\pm$0.034}$ & $0.017\text{\scriptsize$\pm$0.001}$ & $0.019\text{\scriptsize$\pm$0.006}$ \\
    50  & $0.038\text{\scriptsize$\pm$0.014}$ & $0.112\text{\scriptsize$\pm$0.029}$ & $0.018\text{\scriptsize$\pm$0.001}$ & $0.003\text{\scriptsize$\pm$0.004}$ \\
    200 & $0.041\text{\scriptsize$\pm$0.020}$ & $0.119\text{\scriptsize$\pm$0.057}$ & $0.018\text{\scriptsize$\pm$0.002}$ & $0.006\text{\scriptsize$\pm$0.012}$ \\
    10000 & $0.066\text{\scriptsize$\pm$0.022}$ & $0.137\text{\scriptsize$\pm$0.124}$ & $0.033\text{\scriptsize$\pm$0.020}$ & $0.000\text{\scriptsize$\pm$0.000}$ \\

    \midrule
    0.1   & $0.039\text{\scriptsize$\pm$0.012}$ & $0.126\text{\scriptsize$\pm$0.014}$ & $0.013\text{\scriptsize$\pm$0.001}$ & $0.073\text{\scriptsize$\pm$0.026}$ \\
    1   & $0.048\text{\scriptsize$\pm$0.018}$ & $0.142\text{\scriptsize$\pm$0.029}$ & $0.016\text{\scriptsize$\pm$0.001}$ & $0.048\text{\scriptsize$\pm$0.014}$ \\
    5   & $0.035\text{\scriptsize$\pm$0.011}$ & $0.111\text{\scriptsize$\pm$0.031}$ & $0.018\text{\scriptsize$\pm$0.001}$ & $0.027\text{\scriptsize$\pm$0.007}$ \\
    10  & $0.044\text{\scriptsize$\pm$0.014}$ & $0.127\text{\scriptsize$\pm$0.026}$ & $0.018\text{\scriptsize$\pm$0.001}$ & $0.027\text{\scriptsize$\pm$0.007}$ \\
    50  & $0.047\text{\scriptsize$\pm$0.013}$ & $0.134\text{\scriptsize$\pm$0.019}$ & $0.018\text{\scriptsize$\pm$0.002}$ & $0.010\text{\scriptsize$\pm$0.016}$ \\
    200 & $0.040\text{\scriptsize$\pm$0.007}$ & $0.117\text{\scriptsize$\pm$0.016}$ & $0.018\text{\scriptsize$\pm$0.001}$ & $0.001\text{\scriptsize$\pm$0.001}$ \\
    10000& $0.073\text{\scriptsize$\pm$0.027}$ & $1.111\text{\scriptsize$\pm$1.830}$ & $0.083\text{\scriptsize$\pm$0.117}$ & $0.000\text{\scriptsize$\pm$0.000}$ \\

    \bottomrule
  \end{tabular}
\end{table}

\clearpage

\subsection{Experiment Settings and Supplementary Results for High-dimensional Manifold with Large Number of Constraints}
\label{app:Experiment setup for high-dimensional and large number of constraint}

\textbf{Experiment Settings.} \quad  In this high-dimensional experiment, we construct a synthetic ``stress-test'' manifold in $\bbR^d$ by imposing $m-1$ linear equality and $l$ spherical inequality constraints inside a bounding sphere. Concretely, we generate $m-1$ random hyperplanes $h_i(x) = a_i^Tx - b_i$ (with $a_i \sim \calN(0, I_d), b_i \sim \calN(0, 0.1^2 )$) for $i \in [m-1]$, together with the sphere constraint $h_m(x) = \norm{x}_2^2 - R^2$ (with $R=5$), and $l$ spherical obstacles $g_j(x) = r^2 - \norm{x-c_j}_2^2$ (with $r=1$ and obstacle centers $c_j \sim \calN(0, \sqrt{R/2}I_d)$) for $j \in [l]$. All randomness is fixed via a seed across experiments.

\begin{wrapfigure}{r}{0.5\textwidth}
\vspace{-12pt}
\captionof{table}{Hyperparameter settings for high-dimensional manifold example ($\Delta t = 1 \times 10^{-2}$)}
\label{tab:hyperparams_high_dim}
\centering
\begin{tabular}{l l}
\toprule
\textbf{Method} & \textbf{Hyperparameters} \\
\midrule
OLLA      & $\alpha = 200,\;\epsilon = 1$ \\
OLLA–H    & $\alpha = 200,\;\epsilon = 1,\;N = 5$ \\
CLangevin & $L = 5,\;\tau = 10^{-4}, \lambda = 0.1$ \\
CHMC      & $\gamma = 1,\;L = 5,\;\tau = 10^{-4}, \lambda = 0$ \\
CGHMC     & $\gamma = 1,\;L = 5,\;\tau = 10^{-4}, \lambda = 0$ \\
\bottomrule
\end{tabular}
\vspace{-12pt}
\end{wrapfigure}

For each choice of ambient dimension $d$, the number of equality constraints $m$, and the number of obstacles $l$, we run one single chain of each algorithm for $1000$ iterations. We discard the first $200$ iterations as burn-in and then retain every $5$th iterate, yielding $160$ post-burn-in samples. Similar to 2D synthetic experiments, the baseline samplers (CLangevin, CHMC, and CGHMC) are initialized exactly on $\Sigma$ so that $X_0 \in \Sigma$, where as OLLA and OLLA-H start from noisy initialization $Y_0 = X_0 + \calN(0,I_d), X_0 \in \Sigma$.

We measure performance along two complementary criteria. First, we report the computational cost per effective sample size (ESS), defined as
\begin{equation*}
  \text{CPU time / ESS} := \frac{\text{total CPU runtime (s)}}{\min_{1\le i\le d}\text{ESS}_i}\,,
\end{equation*}
where \(\text{ESS}_i\) is the uni-variate ESS in coordinate \(i\). Second, we assess the estimation accuracy via the sample means of representative test functions—e.g.\ \(P(x_1>0)\) and some complicated test functions. To isolate the effect of each problem parameter, we vary one element of \(\{d,m,l\}\) at a time while holding the others fixed.

\begin{remark} \quad 
    The results shown in \cref{fig:Sampling performance and accuracy as the dimension d increases} and \cref{fig:Sampling performance and accuracy as the inequalities l increases} were obtained under a different setup as described above: a larger step size $\Delta t$ was used and the regularization parameter $\lambda$ in CLangevin, CHMC, and CGHMC was increased for improved numerical stability under large step size.
\end{remark}

\textbf{Supplementary Results - Scaling under Dimension and the Number of Inequality Constraints} \quad  In \cref{fig:Scaling under ambient dimension_app}, we plot the performance of samplers as the ambient dimension $d$ increases from $50$ to $700$ (with $m=l=5$). Across all $d$, OLLA and OLLA-H produce virtually identical estimates of our benchmark test functions similar to the baselines. In contrast, CPU time per ESS of OLLA-H stays essentially flat (on the order of $\approx 0.05$s/sample), whereas OLLA grows roughly linearly (reaching $\approx 1.1$s/sample at $d=700$), Also, the CPU runtime plot exhibits OLLA-H achieves the lowest wall-clock times across all dimensions.

Similarly, \cref{fig:Scaling under inequalities_app} shows results with dimension fixed at $d=100$ and the number of inequality constraints $l$ varying form $10$ to $60$ (with $m= 5$). Again, OLLA and OLLA-H approximately match baseline methods in estimating the mean of test functions, and OLLA-H demonstrates dramatically lower CPU time per ESS and CPU runtime than both OLLA and the other baselines.
\begin{figure}
    \centering
    \includegraphics[width=1\linewidth]{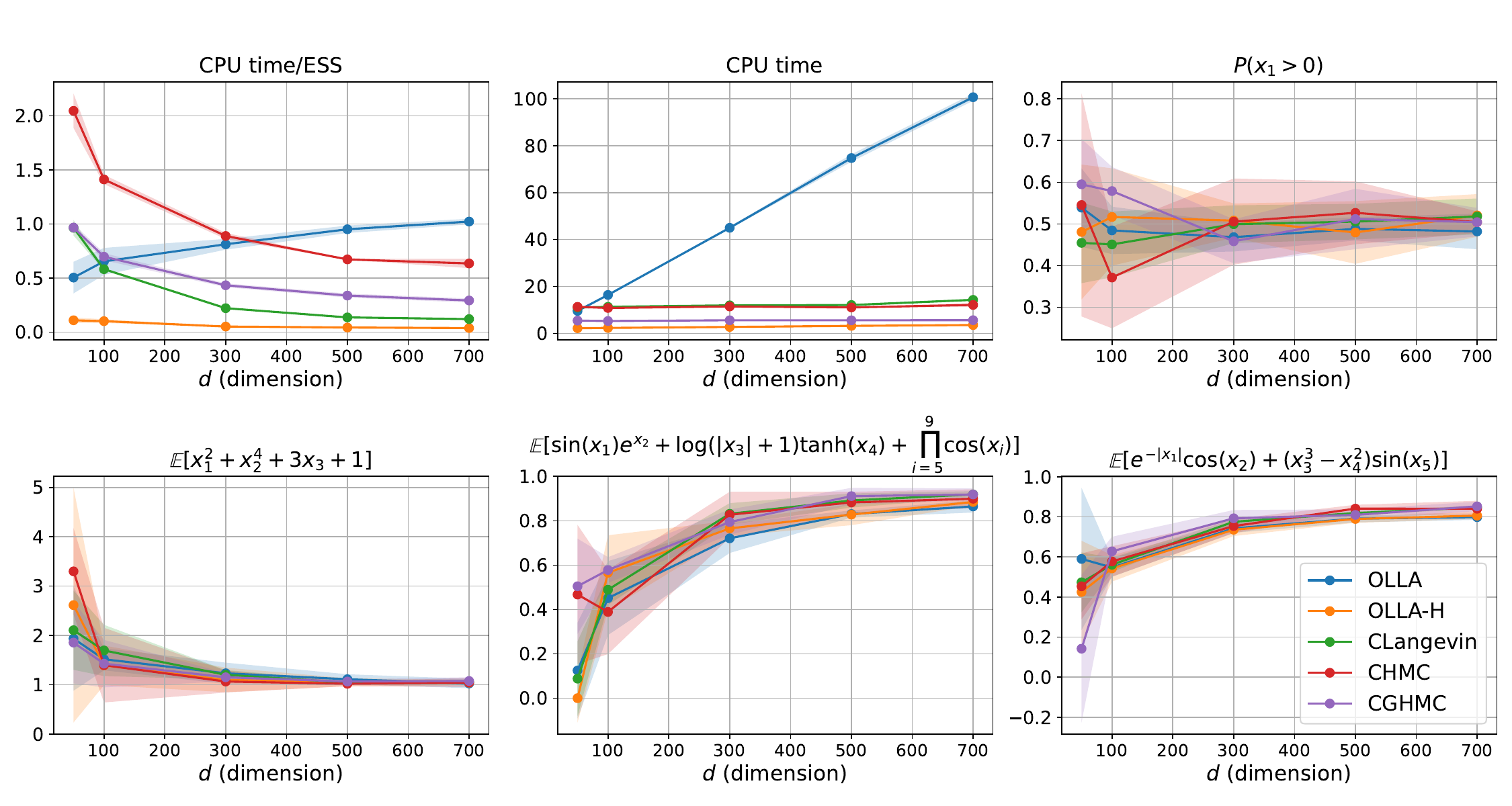}
    \caption{Scaling under the change of ambient dimension $d$. (1) CPU time per ESS (top left), (2) total CPU runtime (top center), (3) estimates of test functions (others) as the dimension $d$ increases from 50 to 700 (with $m=l=5$).}
    \label{fig:Scaling under ambient dimension_app}
\end{figure}

\begin{figure}
    \centering
    \includegraphics[width=1\linewidth]{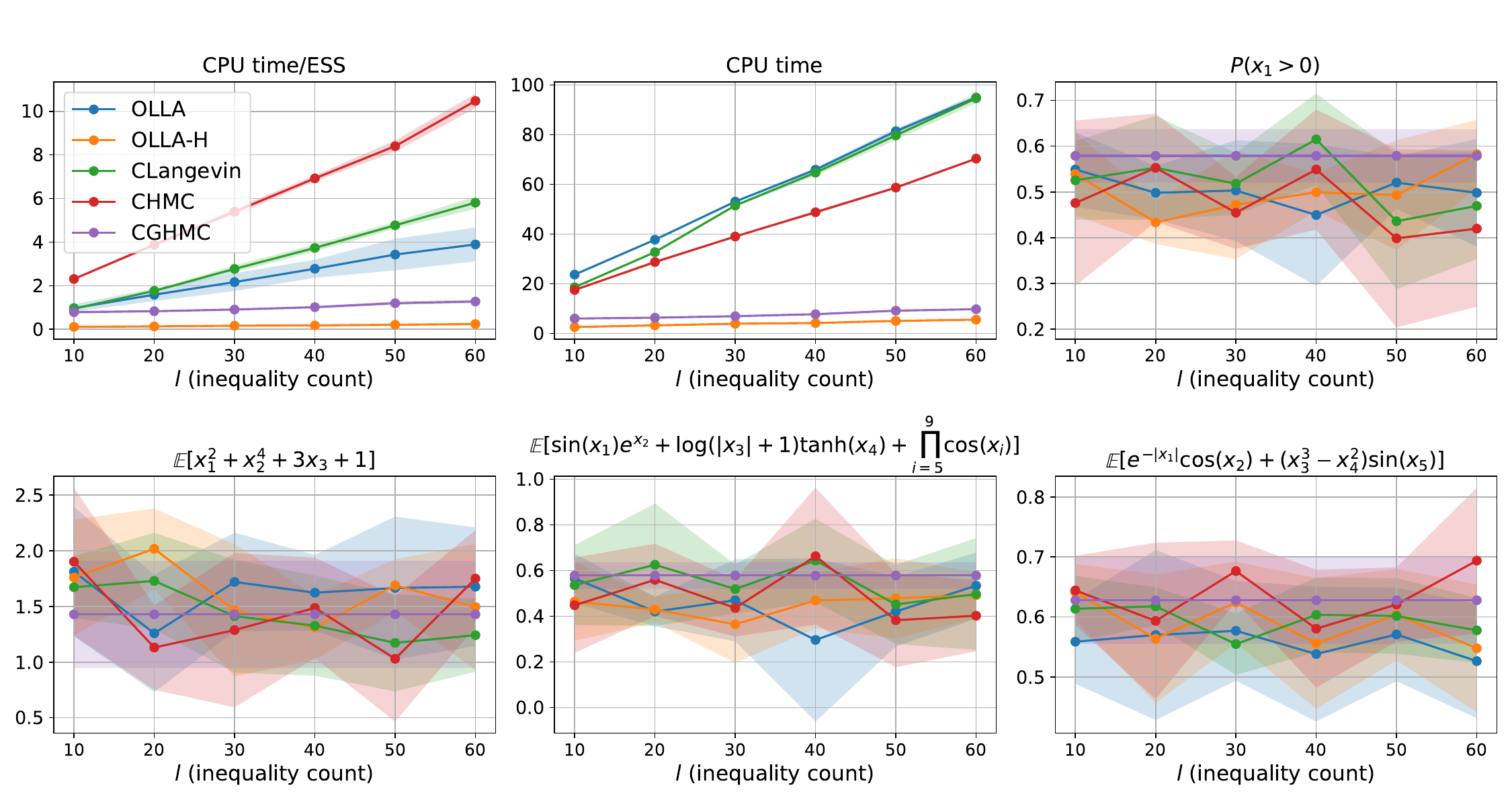}
    \caption{Scaling under the change of inequality count $l$. (1) CPU time per ESS (top left), (2) total CPU runtime (top center) (3) estimates of test functions (others) as the number of inequalities $l$ increases from 10 to 60 (with $d=100, m =5$). The average inequality violation is maintained at $\bbE[g(x)^+] = 0.000 \pm 0.000$ for all samplers.} 
    \label{fig:Scaling under inequalities_app}
\end{figure}

\textbf{Supplementary Results - Scaling under the Number of Equality Constraints} \quad In \cref{fig:Scaling under the change of equality count_app}, we show how sampling performance changes as the number of equality constraints $m$ grows from $10$ to $60$ (with $d= 100, l=5$, fixed $\alpha =200$). Although OLLA-H continues to show the lowest CPU time/ESS (and total CPU time) among the methods, its sampling accuracy  gradually degrades as $m$ increases, drifting away from the baseline values. Especially, we observe that the equality constraint violation worsens, and one must compensate by increasing $\alpha$, reducing $\Delta t$ with a longer chain to suppress the equality constraint violation. Each of these solutions may lead to increase of computational cost, and $\alpha$ cannot be driven arbitrarily high due to its induced discretization instabilities (\cref{tab:effect-alpha_app}). Equality-only baselines therefore achieve more accurate estimates-albeit at higher cost-indicating that large $m$ regimes are particularly challenging to OLLA-H compared to the high-dimensional or many inequality constraints settings.

\begin{figure}
    \centering
    \includegraphics[width=1\linewidth]{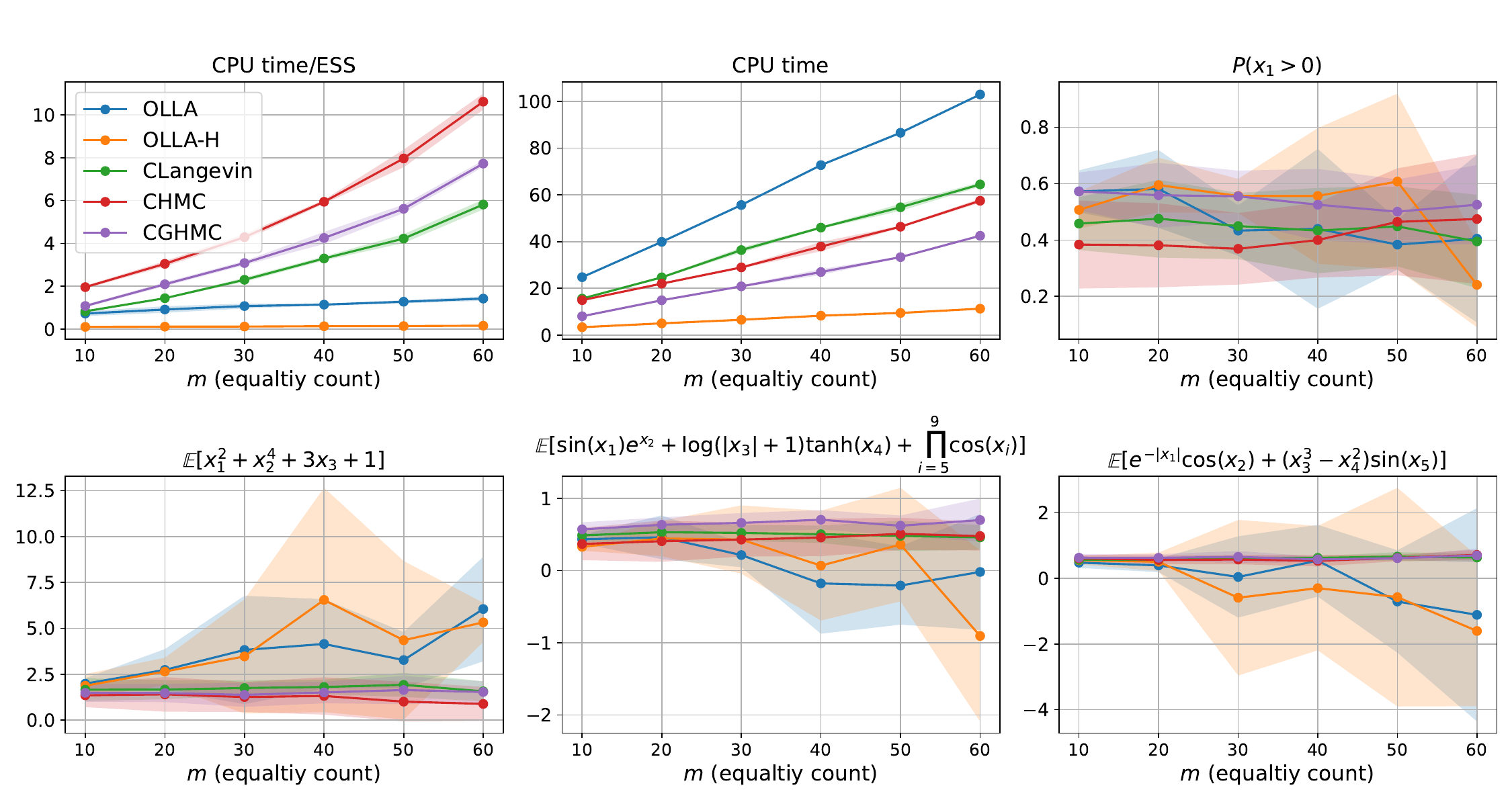}
    \caption{Scaling under the change of equality count $m$. (1) CPU time per ESS (top left), (2) total CPU runtime (top center), (3) estimates of test functions (others) as the number of inequalities $m$ increases from 10 to 60 (with $d=100, l =5$).} 
    \label{fig:Scaling under the change of equality count_app}
\end{figure}

\subsection{Experiment Settings for Molecular system and Bayesian logistic regression task}
\label{app:Experiment setup for Molecular system and Bayesian logistic regression task}

\textbf{Settings.} \quad These two experiments were executed on a Linux machine equipped with Intel Xeon Gold 6226 CPU (24 cores) and 192 GB RAM. 

\textbf{Experiment Settings (Molecular System).} \quad This experiment models a polymer chain of $N_{\text{atoms}}$ atoms in $\bbR^3$, so the state space dimension is $d=3N_{\text{atoms}}$. Let $p \in \bbR^d$ be the flattened vector of atom positions, and let $P_k \in \bbR^3$ denote the position of the $k$-th atom (where $p$ is reshaped into an $N_{\text{atoms}} \times 3$ matrix).

The \textbf{equality constraints} $h(p) =0 $ enforce fixed bond lengths ($l_b$) between adjacent atoms and fixed angles ($\theta_a$) between consecutive bonds:
\begin{enumerate}[leftmargin = 20pt]
    \item Bond length constraints ($k=0, \dots, N_{\text{atoms}}-2$):
    \begin{equation*}
        h_{\text{bond}, k}(p) = \norm{P_k - P_{k+1}}_2^2 - l_b^2 = 0
    \end{equation*}
    \item Bond angle constraints ($k=1, \dots, N_{\text{atoms}}-2$): Let $v_1^k = P_{k-1} - P_k$ and $v_2^k = P_{k+1} - P_k$.
    \begin{equation*}
        h_{\text{angle}, k}(p) = \frac{v_1^k \cdot v_2^k}{\norm{v_1^k}_2 \norm{v_2^k}_2} - \cos(\theta_a) = 0
    \end{equation*}
    with $l_b=1.0$ and $\theta_a = 109.5^{\circ}$.
\end{enumerate}
The \textbf{inequality constraints} ($g(p) \leq 0$) enforce steric hindrance, ensuring a minimum distance ($r_{\text{min}}$) between non-adjacent atoms ($|i-j| \ge 2$):
\begin{equation*}
    g_{ij}(p) = r_{\text{min}}^2 - \norm{P_i - P_j}_2^2 \le 0 \quad \text{for } j \ge i+2
\end{equation*}
with $r_{\text{min}}=1.0$. 

The \textbf{potential function} $f(p)$ models the energy of the polymer configuration and includes terms for torsion angles and non-bonded interactions based on the Weeks-Chandler-Andersen (WCA) potential. It is calculated as $f(p) = \beta (U_{\text{tor}}(p) + U_{\text{nb}}(p))$, where $\beta$ is an inverse temperature parameter (default: $\beta=1.0$).

The torsion potential, $U_{\text{tor}}$, depends on the dihedral angles $\phi_k$ (for $k=1, \dots, N_{\text{atoms}}-3$) formed by consecutive bonds, where $\phi_k$ denotes the dihedral angle between atoms $P_{k-1}, P_k, P_{k+1}, P_{k+2}$.

The potential is a sum over modes $m \in M$ (default: $M = \mbra{1,3}$) with corresponding force constants $k_m$ (default: $k_1 = 0.5, k_3 = 0.2$) and phase shifts $\delta_m$ (default: $\delta_1 =0.0, \delta_3 = 0.0$):
\begin{equation*}
    U_{\text{tor}}(p) = \sum_{k=1}^{N_{\text{atoms}}-3} \sum_{m \in M} k_m (1 + \cos(m \phi_k - \delta_m)).
\end{equation*}

The non-bonded WCA potential, $U_{\text{nb}}$, models repulsion between atoms $i$ and $j$ that are not directly bonded and are separated by at least two bonds ($|i-j| \ge 3$). Let $R_{ij} = \|P_i - P_j\|_2$ be the distance between atoms $i$ and $j$. With the steric minimum distance $r_{\text{min}}$, its associated length scale $\sigma = r_{\text{min}} / 2^{1/6}$, and the potential cutoff is $r_c = 2^{1/6} \sigma$. The interaction energy is given by:
\begin{equation*}
    U_{\text{LJ}}(R_{ij}) = 4 \epsilon_{\text{WCA}} \lbra{ \sbra{\frac{\sigma}{R_{ij}}}^{12} - \sbra{\frac{\sigma}{R_{ij}}}^{6}} + \epsilon_{\text{WCA}}
\end{equation*}
where $\epsilon_{\text{WCA}}$ is the energy scale (default: 1.0). The total non-bonded potential is the sum over eligible pairs ($j \ge i+3$) where the distance is less than the cutoff:
\begin{equation*}
    U_{\text{nb}}(p) = \sum_{i=0}^{N_{\text{atoms}}-4} \sum_{j=i+3}^{N_{\text{atoms}}-1} U_{\text{LJ}}(R_{ij}) \cdot \bbI(R_{ij} < r_c)
\end{equation*}
where $\bbI(\cdot)$ is the indicator function.

We vary $N_{\text{atoms}} \in \{5, 10, 15, 20, 30\}$, corresponding to $d \in \{15, 30, 45, 60, 90\}$. We run a single chain for $K=5000$ steps, discard the first 1000 as burn-in, and thin by 5. To measure the accuracy of sampling, we use the radius of gyration squared $(R_g^2)$ as a test function, which is defined as:
\begin{equation*}
    R_g^2(p) = \frac{1}{N_{\text{atoms}}} \sum_{k=0}^{N_{\text{atoms}}-1} \norm{P_k - P_{cm}}_2^2, \quad \text{where } P_{cm} = \frac{1}{N_{\text{atoms}}} \sum_{k=0}^{N_{\text{atoms}}-1} P_k
\end{equation*}

\vspace{-15pt}
\begin{table}[h]
\centering
\caption{Hyperparameter settings for the Molecular System example ($\Delta t = 1 \times 10^{-5}$)}
\vspace{5pt}
\label{tab:hyperparams_polymer}
\begin{tabular}{l l}
\toprule
\textbf{Method} & \textbf{Hyperparameters} \\
\midrule
OLLA-H    & $\alpha = 500, \epsilon = 1.0, N \in \mbra{0, 5}$ \\
CLangevin & $L = 30, \tau = 10^{-4}, \lambda = 0.5$ \\
CHMC      & $\gamma = 1.0, L = 30, \tau = 10^{-4}, \lambda = 0.0$ \\
CGHMC     & $\gamma = 1.0, L = 30, \tau = 10^{-4}, \lambda = 0.0$ \\
\bottomrule
\end{tabular}
\end{table}

\textbf{Experiment Settings (Bayesian logistic regression).} \quad This experiment involves sampling the posterior distribution of weights $\theta \in \bbR^d$ for a two-layer Bayesian neural network applied to the German Credit dataset \cite{statlog_(german_credit_data)_144}. Let $\sigma(\cdot)$ denote the sigmoid function and $\hat{p}_{\text{logit}}(\theta, x, a)$ be the network's output probability for input features $x$ and sensitive attribute $a$. 

The neural network consists of an input layer, two hidden layers with ReLU activation (sizes $H_1=32$, $H_2=16$), and a final linear output layer combined with a bias term $b_0$ and a term $\alpha \cdot a$ dependent on the sensitive attribute:
\begin{align*}
    h_1 &= \text{ReLU}(W_1 x + b_1) \in \bbR^{H_1} \\
    h_2 &= \text{ReLU}(W_2 h_1 + b_2) \in \bbR^{H_2} \\
    \hat{p}_{\text{logit}}(\theta, x, a) &= w_3^T h_2 + \alpha a + b_0 \in \bbR
\end{align*}
where the parameters constituting $\theta$ have dimensions: $W_1 \in \bbR^{H_1 \times \text{input\_dim}}$, $b_1 \in \bbR^{H_1}$, $W_2 \in \bbR^{H_2 \times H_1}$, $b_2 \in \bbR^{H_2}$, $w_3 \in \bbR^{H_2}$, $\alpha \in \bbR$, and $b_0 \in \bbR$. So, the total dimension $d = H_1 \cdot \text{input\_dim} + H_1 + H_2 \cdot H_1 + H_2 + H_2 + 2$ varies based on the input dimension, which itself depends on the feature hashing dimension used for categorical features. In particular, the parameter size of $\theta$ can be changed by adjusting hashing dimension.

The \textbf{potential function} is the negative log-posterior $f(v) = -\left( \log P(\calD|\theta) + \log P(\theta) \right)$, where $\calD$ is the training data, $P(\calD|\theta)$ is the log-likelihood using the sigmoid of the logits, and $\log P(\theta)$ is the log-prior based on an isotropic Gaussian distribution with precision $10^{-3}$.

The \textbf{equality constraints} ($h(\theta)=0$) enforce fairness via demographic parity on True Positive Rate (TPR) and False Positive Rate (FPR) between groups defined by the sensitive attribute (default: gender) $a \in \{0, 1\}$:
\begin{align*}
    h_{\text{TPR}}(v) &= \bbE_{x, y | a=1, y=1}[\sigma(\hat{p}_{\text{logit}}(v, x, 1))] - \bbE_{x, y | a=0, y=1}[\sigma(\hat{p}_{\text{logit}}(v, x, 0))] = 0 \\
    h_{\text{FPR}}(v) &= \bbE_{x, y | a=1, y=0}[\sigma(\hat{p}_{\text{logit}}(v, x, 1))] - \bbE_{x, y | a=0, y=0}[\sigma(\hat{p}_{\text{logit}}(v, x, 0))] = 0.
\end{align*}
These expectations are estimated using averages over the training data subsets corresponding to each sensitive attribute $a$ and true label $y$.

The \textbf{inequality constraints} ($g(\theta) \le 0$) enforce monotonicity on selected features (default: duration, credit amount, existing credit, age) by requiring the gradient of the logit with respect to these features to have a specific sign (or be close to zero, within a margin $\delta = 1.0$) at a subset of anchor data points $\calD_{\text{anchor}}$. Let $S^+ = \mbra{\text{duration, credit amount, existing credits}}$ and $S^- = \mbra{\text{age}}$. The constraints are formulated as:
\begin{align*}
    g_{\text{mono}}(v) = \max \sbra{\max_{j \in S^+, x_i \in \calD_{\text{anchor}}} \mbra{ -\frac{\partial \hat{p}_{\text{logit}}(v, x_i, a_i)}{\partial x_{ij}} }, \max_{k \in S^-, x_i \in \calD_{\text{anchor}} } \mbra{ \frac{\partial \hat{p}_{\text{logit}}(v, x_i, a_i)}{\partial x_{ik}}} } - \delta.
\end{align*}
The dimension $d$ varies based on feature hashing as $d \in \{706, 1986, 4994, 9986, 49986, 100002\}$. We run a single chain for $K=200$ steps, discard the first 40 as burn-in, and thin by 2. Also, we use the test Negative Log-Likelihood (NLL) as the metric.
\vspace{-15pt}
\begin{table}[h]
\centering
\caption{Hyperparameter settings for the Bayesian logistic regression task}
\vspace{5pt}
\label{tab:hyperparams_german}
\begin{tabular}{l l}
\toprule
\textbf{Method} & \textbf{Hyperparameters} \\
\midrule
OLLA-H    & $\alpha = 100, \epsilon = 1.0, N \in \mbra{0,5}, \Delta t=5\times 10^{-4}$ \\
CLangevin & $L = 10, \tau = 1.0, \lambda = 0.5, \Delta t=5\times 10^{-4}$ \\
CHMC      & $\gamma = 1.0, L = 10, \tau = 1.0, \lambda = 0.5, \Delta t=5 \times 10^{-3}$ \\
CGHMC     & $\gamma = 1.0, L = 10, \tau = 1.0, \lambda = 0.5, \Delta t=5 \times 10^{-3}$ \\
\bottomrule
\end{tabular}
\end{table}